\documentclass[twoside,11pt]{article}
\usepackage{rststyle, theapa, rawfonts}
\usepackage{amsthm}
\usepackage{enumitem}
\usepackage{mnotes}

\usepackage{cutwin}
\usepackage[calc]{adjustbox}
\usepackage{amsmath,amsthm,amssymb}
\usepackage{amsfonts}
\usepackage{graphicx}
\usepackage{mathrsfs}

\usepackage{graphicx}
\usepackage{textcomp}
\usepackage{comment}
\usepackage{bussproofs}
\usepackage{algorithm}
\usepackage[noend]{algpseudocode}
\usepackage{alltt}
\usepackage{transparent}
\usepackage{colortbl}
\usepackage{tikz}
\usepackage{tikz-cd}

\usepackage{wrapfig}

\usepackage{csvsimple}
\definecolor{darkgray}{rgb}{0.3, 0.3, 0.3}
\definecolor{darkgreen}{rgb}{0.0, 0.6, 0.3}
\definecolor{darkblue}{rgb}{0.0, 0.3, 0.8}
\definecolor{darkred}{rgb}{0.8, 0.1, 0.1}
\definecolor{darkpurple}{rgb}{0.4, 0.0, 0.6}
\definecolor{darkgold}{rgb}{0.8, 0.6, 0}

\usepackage{longtable}
\usepackage[scr=boondox]{mathalfa}
\usepackage{bold-extra}

\usepackage{mnotes}
\newcommenter{\anote}{Aaron}{green!20}
\newcommenter{\dnote}{Daniel}{blue!20}
\newcommenter{\mnote}{Mateja}{orange!20}
\newcommenter{\gnote}{Gem}{purple!20}
\newcommenter{\gnotex}{Gem*}{purple!20}
\newcommenter{\remark}{Remark}{white}

\newcommand{\superSpace}{\mathcal{S}}
\newcommand{\multiSpace}{\mathcal{M}}
\newcommand{\goalSpace}{\mathcal{G}}

\newcommand{\contextGraph}{\kappa}
\newcommand{\sourceGraph}{\contextGraph_1}
\newcommand{\targetGraph}{\contextGraph_2}
\newcommand{\assump}{\alpha}
\newcommand{\goa}{\gamma}

\newcommand{\schemaContext}{\contextGraph}%{\schema_{\contextGraph}}
\newcommand{\antecedent}{\assump}%{\schema_{\assump}}
\newcommand{\consequent}{\goa}%{\schema_{\goa}}

\newcommand{\theDelta}{\delta}

\newcommand{\sequent}[2]{\langle #1 \Vdash\! #2 \rangle}
\newcommand{\infSequent}[3]{\langle#1,#2 \Vdash\! #3\rangle}

\newcommand{\antecedentN}[1]{#1_{\assump}}
\newcommand{\consequentN}[1]{#1_{\goa}}

\usepackage{changepage}

\usetikzlibrary{positioning,calc,matrix,arrows.meta,arrows,automata,patterns,calc,shadings,backgrounds,fit,tikzmark,decorations.pathreplacing,decorations,decorations.pathmorphing,shapes,shapes.misc}

\usepackage{xifthen}
\usepackage{xargs}

\tikzset{anchor/.append code=\let\tikz@auto@anchor\relax}

\tikzstyle{termrep} = [draw=black!30, rectangle, rounded corners = 3.7, minimum height = 0.27cm, minimum width = 0.27cm, inner sep=2pt]
\tikzstyle{termIrep} = [termrep, double={black!10}]

\tikzset{termpos/.style n args={3}{termrep, label = {[label distance=#3, anchor=center]#2:{\scriptsize#1\vphantom{(Mg)}}}}}
\tikzset{termIpos/.style n args={3}{termIrep, label = {[label distance=#3, anchor=center]#2:{\scriptsize#1\vphantom{(Mg)}}}}}

\tikzstyle{term} = [termrep, label = {[label distance=-5pt]-35:{\scriptsize#1\vphantom{(Mg)}}}]
\tikzstyle{termI} = [termrep, double={black!10}, label = {[label distance=-5pt]-35:{\scriptsize#1\vphantom{(Mg)}}}]

\tikzstyle{typerep} = [termrep, fill=black!10]
\tikzstyle{typeIrep} = [typerep, double={black!0}]

\tikzset{typepos/.style n args={3}{typerep, label = {[label distance=#3, anchor=center]#2:{\scriptsize#1\vphantom{(Mg)}}}}}
\tikzset{typeIpos/.style n args={3}{typeIrep, label = {[label distance=#3, anchor=center]#2:{\scriptsize#1\vphantom{(Mg)}}}}}

\tikzstyle{type} = [typerep, label = {[label distance=-5pt]-35:{\scriptsize#1\vphantom{(Mg)}}}]
\tikzstyle{typeI} = [typeIrep, label = {[label distance=-5pt]-35:{\scriptsize#1\vphantom{(Mg)}}}]

\tikzstyle{constructorrep} = [draw, rectangle, rounded corners = 4.2, minimum height = 0.32cm, minimum width = 0.32cm, inner sep=2.6pt]
\tikzstyle{constructorGrep} = [constructorrep]
\tikzstyle{constructorErep} = [constructorrep]
\tikzstyle{constructorIrep} = [constructorrep]

\tikzset{constructorpos/.style n args={3}{constructorrep, label = {[label distance=#3, anchor=center]#2:{\scriptsize#1\vphantom{(Mg)}}}}}
\tikzset{constructorGpos/.style n args={3}{constructorGrep, label = {[label distance=#3, anchor=center]#2:{\scriptsize#1\vphantom{(Mg)}}}}}
\tikzset{constructorEpos/.style n args={3}{constructorErep, label = {[label distance=#3, anchor=center]#2:{\scriptsize#1\vphantom{(Mg)}}}}}
\tikzset{constructorIpos/.style n args={3}{constructorIrep, label = {[label distance=#3, anchor=center]#2:{\scriptsize#1\vphantom{(Mg)}}}}}

\tikzstyle{constructor} = [constructorrep, label = {[label distance=-5pt]145:{\scriptsize\vphantom{(Mg)}#1}}]
\tikzstyle{constructorG} = [constructorGrep, label = {[label distance=-5pt]145:{\scriptsize\vphantom{(Mg)}#1}}]
\tikzstyle{constructorE} = [constructorErep, label = {[label distance=-5pt]145:{\scriptsize\vphantom{(Mg)}#1}}]
\tikzstyle{constructorI} = [constructorIrep, label = {[label distance=-5pt]145:{\scriptsize\vphantom{(Mg)}#1}}]

    \tikzstyle{termSE} = [termrep, label = {[label distance=-5pt]-35:{\scriptsize#1\vphantom{(Mg)}}}]
    \tikzstyle{termNE} = [termrep, label = {[label distance=-5pt]35:{\scriptsize#1\vphantom{(Mg)}}}]
    \tikzstyle{termSW} = [termrep, label = {[label distance=-5pt]-145:{\scriptsize#1\vphantom{(Mg)}}}]
    \tikzstyle{termNW} = [termrep, label = {[label distance=-5pt]145:{\scriptsize#1\vphantom{(Mg)}}}]
    \tikzstyle{termS} = [termrep, label = {[label distance=0.5pt,outer sep = 0pt,inner sep = 0pt]-90:{\scriptsize#1\vphantom{(Mg)}}}]
    \tikzstyle{termN} = [termrep, label = {[label distance=-3pt]90:{\scriptsize#1\vphantom{(Mg)}}}]
    \tikzstyle{termE} = [termrep, label = {[label distance=1pt,outer sep = 0pt,inner sep = 0pt]0:{\scriptsize#1\vphantom{(Mg)}}}]
    \tikzstyle{termW} = [termrep, label = {[label distance=1pt,outer sep = 0pt,inner sep = 0pt]180:{\scriptsize#1\vphantom{(Mg)}}}]

    \tikzstyle{termISE} = [termIrep, label = {[label distance=-5pt]-35:{\scriptsize#1\vphantom{(Mg)}}}]
    \tikzstyle{termINE} = [termIrep, label = {[label distance=-5pt]35:{\scriptsize#1\vphantom{(Mg)}}}]
    \tikzstyle{termISW} = [termIrep, label = {[label distance=-5pt]-145:{\scriptsize#1\vphantom{(Mg)}}}]
    \tikzstyle{termINW} = [termIrep, label = {[label distance=-5pt]145:{\scriptsize#1\vphantom{(Mg)}}}]
    \tikzstyle{termIS} = [termIrep, label = {[label distance=0.5pt,outer sep = 0pt,inner sep = 0pt]-90:{\scriptsize#1\vphantom{(Mg)}}}]
    \tikzstyle{termIN} = [termIrep, label = {[label distance=-3pt]90:{\scriptsize#1\vphantom{(Mg)}}}]
    \tikzstyle{termIE} = [termIrep, label = {[label distance=1pt,outer sep = 0pt,inner sep = 0pt]0:{\scriptsize#1\vphantom{(Mg)}}}]
    \tikzstyle{termIW} = [termIrep, label = {[label distance=1pt,outer sep = 0pt,inner sep = 0pt]180:{\scriptsize#1\vphantom{(Mg)}}}]

    \tikzstyle{typeSE} = [typerep, label = {[label distance=-5pt]-35:{\scriptsize#1\vphantom{(Mg)}}}]
    \tikzstyle{typeNE} = [typerep, label = {[label distance=-5pt]35:{\scriptsize#1\vphantom{(Mg)}}}]
    \tikzstyle{typeSW} = [typerep, label = {[label distance=-5pt]-145:{\scriptsize#1\vphantom{(Mg)}}}]
    \tikzstyle{typeNW} = [typerep, label = {[label distance=-5pt]145:{\scriptsize#1\vphantom{(Mg)}}}]
    \tikzstyle{typeS} = [typerep, label = {[label distance=0.5pt,outer sep = 0pt,inner sep = 0pt]-90:{\scriptsize#1\vphantom{(Mg)}}}]
    \tikzstyle{typeN} = [typerep, label = {[label distance=-3pt]90:{\scriptsize#1\vphantom{(Mg)}}}]
    \tikzstyle{typeE} = [typerep, label = {[label distance=1pt,outer sep = 0pt,inner sep = 0pt]0:{\scriptsize#1\vphantom{(Mg)}}}]
    \tikzstyle{typeW} = [typerep, label = {[label distance=1pt,outer sep = 0pt,inner sep = 0pt]180:{\scriptsize#1\vphantom{(Mg)}}}]

    \tikzstyle{typeISE} = [typeIrep, label = {[label distance=-5pt]-35:{\scriptsize#1\vphantom{(Mg)}}}]
    \tikzstyle{typeINE} = [typeIrep, label = {[label distance=-5pt]35:{\scriptsize#1\vphantom{(Mg)}}}]
    \tikzstyle{typeISW} = [typeIrep, label = {[label distance=-5pt]-145:{\scriptsize#1\vphantom{(Mg)}}}]
    \tikzstyle{typeINW} = [typeIrep, label = {[label distance=-5pt]145:{\scriptsize#1\vphantom{(Mg)}}}]
    \tikzstyle{typeIS} = [typeIrep, label = {[label distance=0.5pt,outer sep = 0pt,inner sep = 0pt]-90:{\scriptsize#1\vphantom{(Mg)}}}]
    \tikzstyle{typeIN} = [typeIrep, label = {[label distance=-3pt]90:{\scriptsize#1\vphantom{(Mg)}}}]
    \tikzstyle{typeIE} = [typeIrep, label = {[label distance=1pt,outer sep = 0pt,inner sep = 0pt]0:{\scriptsize#1\vphantom{(Mg)}}}]
    \tikzstyle{typeIW} = [typeIrep, label = {[label distance=1pt,outer sep = 0pt,inner sep = 0pt]180:{\scriptsize#1\vphantom{(Mg)}}}]

    \tikzstyle{constructorNW} = [constructorrep, label = {[label distance=-5pt]145:{\scriptsize\vphantom{(Mg)}#1}}]
    \tikzstyle{constructorNE} = [constructorrep, label = {[label distance=-5pt]35:{\scriptsize\vphantom{(Mg)}#1}}]
    \tikzstyle{constructorSE} = [constructorrep, label = {[label distance=-5pt]-35:{\scriptsize\vphantom{(Mg)}#1}}]
    \tikzstyle{constructorSW} = [constructorrep, label = {[label distance=-5pt]-145:{\scriptsize\vphantom{(Mg)}#1}}]
    \tikzstyle{constructorN} = [constructorrep, label = {[label distance=-3.5pt]90:{\scriptsize\vphantom{(Mg)}#1}}]
    \tikzstyle{constructorGE} = [constructorrep, label = {[label distance=-1pt]0:{\scriptsize\vphantom{(Mg)}#1}}]
    \tikzstyle{constructorS} = [constructorrep, label = {[label distance=1pt,outer sep = 0pt,inner sep = 0pt]-90:{\scriptsize\vphantom{(Mg)}#1}}]
    \tikzstyle{constructorW} = [constructorrep, label = {[label distance=-1pt]180:{\scriptsize\vphantom{(Mg)}#1}}]
    \tikzstyle{constructorE} = [constructorGrep, label = {[label distance=-1pt]0:{\scriptsize\vphantom{(Mg)}#1}}]

    \tikzstyle{constructorGNW} = [constructorGrep, label = {[label distance=-5pt]145:{\scriptsize\vphantom{(Mg)}#1}}]
    \tikzstyle{constructorGNE} = [constructorGrep, label = {[label distance=-5pt]35:{\scriptsize\vphantom{(Mg)}#1}}]
    \tikzstyle{constructorGSE} = [constructorGrep, label = {[label distance=-5pt]-35:{\scriptsize\vphantom{(Mg)}#1}}]
    \tikzstyle{constructorGSW} = [constructorGrep, label = {[label distance=-5pt]-145:{\scriptsize\vphantom{(Mg)}#1}}]
    \tikzstyle{constructorGN} = [constructorGrep, label = {[label distance=-3pt]90:{\scriptsize\vphantom{(Mg)}#1}}]
    \tikzstyle{constructorGE} = [constructorGrep, label = {[label distance=-1pt]0:{\scriptsize\vphantom{(Mg)}#1}}]
    \tikzstyle{constructorGS} = [constructorGrep, label = {[label distance=1pt,outer sep = 0pt,inner sep = 0pt]-90:{\scriptsize\vphantom{(Mg)}#1}}]
    \tikzstyle{constructorGW} = [constructorGrep, label = {[label distance=-1pt]180:{\scriptsize\vphantom{(Mg)}#1}}]

    \tikzstyle{constructorENW} = [constructorErep, label = {[label distance=-5pt]145:{\scriptsize\vphantom{(Mg)}#1}}]
    \tikzstyle{constructorENE} = [constructorErep, label = {[label distance=-5pt]35:{\scriptsize\vphantom{(Mg)}#1}}]
    \tikzstyle{constructorESE} = [constructorErep, label = {[label distance=-5pt]-35:{\scriptsize\vphantom{(Mg)}#1}}]
    \tikzstyle{constructorESW} = [constructorErep, label = {[label distance=-5pt]-145:{\scriptsize\vphantom{(Mg)}#1}}]
    \tikzstyle{constructorEN} = [constructorErep, label = {[label distance=-3pt]90:{\scriptsize\vphantom{(Mg)}#1}}]
    \tikzstyle{constructorEE} = [constructorErep, label = {[label distance=-1pt]0:{\scriptsize\vphantom{(Mg)}#1}}]
    \tikzstyle{constructorES} = [constructorErep, label = {[label distance=1pt,outer sep = 0pt,inner sep = 0pt]-90:{\scriptsize\vphantom{(Mg)}#1}}]
    \tikzstyle{constructorEW} = [constructorErep, label = {[label distance=-1pt]180:{\scriptsize\vphantom{(Mg)}#1}}]

    \tikzstyle{constructorINW} = [constructorIrep, label = {[label distance=-5pt]145:{\scriptsize\vphantom{(Mg)}#1}}]
    \tikzstyle{constructorINE} = [constructorIrep, label = {[label distance=-5pt]35:{\scriptsize\vphantom{(Mg)}#1}}]
    \tikzstyle{constructorISE} = [constructorIrep, label = {[label distance=-5pt]-35:{\scriptsize\vphantom{(Mg)}#1}}]
    \tikzstyle{constructorISW} = [constructorIrep, label = {[label distance=-5pt]-145:{\scriptsize\vphantom{(Mg)}#1}}]
    \tikzstyle{constructorIN} = [constructorIrep, label = {[label distance=-3pt]90:{\scriptsize\vphantom{(Mg)}#1}}]
    \tikzstyle{constructorIE} = [constructorIrep, label = {[label distance=-1pt]0:{\scriptsize\vphantom{(Mg)}#1}}]
    \tikzstyle{constructorIS} = [constructorIrep, label = {[label distance=1pt,outer sep = 0pt,inner sep = 0pt]-90:{\scriptsize\vphantom{(Mg)}#1}}]
    \tikzstyle{constructorIW} = [constructorIrep, label = {[label distance=-1pt]180:{\scriptsize\vphantom{(Mg)}#1}}]

\tikzstyle{constructorO} = [draw, dashed, trapezium, trapezium angle = 60, rounded corners = 10,  minimum width = 1.4cm, minimum height = 1.1cm, outer sep=-4pt]

\tikzstyle{construction} = [node distance = 0.7cm and 1cm, inner sep = 3pt, >={Latex[width=3.5pt,length=4pt]}]
\tikzstyle{decomposition} = [node distance = 1cm and 1cm, inner sep = 3pt, >={Latex[width=6pt,length=6pt]},thick]
\tikzstyle{dInline} = [baseline=-5pt, every node/.style={scale=0.93, inner sep = 2pt}]
\tikzstyle{vInline} = [minimum height = 1.25em]
\tikzstyle{omitting} = [node distance = 1.5em and 2em]
\tikzstyle{index label} = [fill=white, pos=0.4, inner sep=0.1pt, font=\tiny, circle]
\tikzstyle{decomp arrow label} = [fill=white, pos=0.4, inner sep=0.2pt, font=\footnotesize, circle]

\newlength{\illength}
\newcommand{\arrow}[1][]{\hspace{2pt}%
	\ifthenelse{\isempty{#1}}{%
		\tikz[construction,baseline=-3.1pt]{\draw[->] (0,0) to (0.5cm,0)}%
	}{%
		\settowidth{\illength}{\tiny#1}%
		\tikz[construction,baseline=-3.1pt]{\draw[->] (0pt,0) to node[index label, pos=0, xshift=7pt, inner sep=-0.25pt, anchor=west, scale=1]{$#1$} (1.2\illength+18pt,0pt)}%
	}%
	\hspace{2pt}%
}
\newcommandx{\cnode}[7][1,7]{%
	\ifthenelse{\isempty{#7}}{%
		\node[constructor#1 = {#2}, below = 0.5cm of #5](#4){#3};
	}{%
		\node[constructor#1 = {#2}, #7](#4){#3};
	}
	\draw[->] (#4) edge (#5);
	\foreach \x/\l[count = \i] in {#6} {\draw[->] (\x) edge node[index label]{\ifthenelse{\equal{\l}{\x}}{\i}{\l}} (#4);}
}
\newcommandx{\tnode}[5][1=term,5]{%
	\ifthenelse{\isempty{#5}}{%
		\node[#1 = {#2}](#4){#3};
	}{%
		\node[#1 = {#2}, #5](#4){#3};
	}
}

\newcommand{\incVert}{\mathit{iv}} % function that returns the vertices incident with an arrow

\newcommand{\inA}[1]{\mathit{in}_{\hspace{-0.084em}A}(#1)} % set of arrows going in to a vertex
\newcommand{\outA}[1]{\mathit{out}_{\hspace{-0.084em}A}(#1)} % set of arrows going out from a vertex
\newcommand{\pb}{\mathit{Cr}}

\newcommand{\consl}{\mathit{co}}

\newcommand{\tokens}{\mathit{To}} % tokens

\newcommand{\types}{\mathit{Ty}} % set of types
\newcommand{\type}{\mathit{type}}% typing function for terms

\newcommand{\tsystemn}{T} % type system name

\newcommand{\constructors}{\mathit{Co}} % set of constructors
\newcommand{\sig}{\mathit{sig}} % signature of constructor
\newcommand{\cspecificationn}{C} % construction specification name

\newcommand{\specialise}{s}
\newcommand{\reify}{f}
\newcommand{\refine}{r}
\newcommand{\loosen}{\ell}

\newcommand{\Acirc}{\begin{tikzpicture}[scale=0.8,anchor = center]\footnotesize
	\draw (1.6,1.4) circle (0.8) node[xshift = -0.7cm, yshift = 0.55cm] {$A$} ;
	\end{tikzpicture}}
\newcommand{\Bcirc}{\begin{tikzpicture}[scale=0.8,anchor = center]\footnotesize
	\draw (1.4,1.5) circle (0.8) node[xshift = -0.7cm, yshift = 0.55cm] {$B$} ;
	\end{tikzpicture}}
\newcommand{\Ccirc}{\begin{tikzpicture}[scale=0.8,anchor = center]\footnotesize
	\draw (3.6,1.5) circle (0.8) node[xshift = 0.7cm, yshift = 0.55cm] {$C$};
	\end{tikzpicture}}
\newcommand{\AsubB}{\begin{tikzpicture}[scale=0.8,anchor = center]\footnotesize
	\draw (1.6,1.4) circle (0.5) node[xshift = -0.42cm, yshift = 0.45cm] {$A$} ;
	\draw (1.4,1.5) circle (1) node[xshift = -0.75cm, yshift = 0.7cm] {$B$} ;
	\end{tikzpicture}}
\newcommand{\BdisjC}{\begin{tikzpicture}[scale=0.8,anchor = center]\footnotesize
	\draw (1.4,1.5) circle (1) node[xshift = -0.75cm, yshift = 0.7cm] {$B$} ;
	\draw (3.6,1.5) circle (0.8) node[xshift = 0.65cm, yshift = 0.6cm] {$C$};
	\end{tikzpicture}}

\newcommand{\AdisjCsmall}{%
\begin{tikzpicture}[scale=0.5,anchor = center]\footnotesize
\draw (1.6,1.4) circle (0.5) node[xshift = -0.3cm, yshift = 0.32cm] {$A$} ;
\draw (3.6,1.5) circle (0.8) node[xshift = 0.43cm, yshift = 0.4cm] {$C$};
\end{tikzpicture}}

\newcommand{\AdisjCselectAsmall}{%
\begin{tikzpicture}[scale=0.5,anchor = center]\footnotesize
\fill[black, rounded corners = 2, opacity = 0.2] (0.5,0.57) rectangle (4.9,2.7);
\draw[fill= white] (1.6,1.4) circle (0.5) node[xshift = -0.3cm, yshift = 0.32cm] {$A$} ;
\draw (3.6,1.5) circle (0.8) node[xshift = 0.43cm, yshift = 0.4cm] {$C$};
\end{tikzpicture}}

\newcommand{\AdisjCselectCsmall}{%
\begin{tikzpicture}[scale=0.5,anchor = center]\footnotesize
	\fill[black, rounded corners = 2, opacity = 0.2] (0.5,0.57) rectangle (4.9,2.7);
\draw (1.6,1.4) circle (0.5) node[xshift = -0.3cm, yshift = 0.32cm] {$A$} ;
\draw[fill= white] (3.6,1.5) circle (0.8) node[xshift = 0.43cm, yshift = 0.4cm] {$C$};
\end{tikzpicture}}
\newcommand{\Acircsmall}{\begin{tikzpicture}[scale=0.5,anchor = center]\footnotesize
	\draw (1.6,1.4) circle (0.5) node[xshift = -0.35cm, yshift = 0.25cm] {$A$} ;
	\end{tikzpicture}}
\newcommand{\AcircsmallSelectNothing}{\begin{tikzpicture}[scale=0.5,anchor = center]\footnotesize
	\fill[black, rounded corners = 2, opacity = 0.2] (0.4,0.7) rectangle (2.3,2.4);
	\draw (1.6,1.4) circle (0.5) node[xshift = -0.35cm, yshift = 0.3cm] {$A$} ;
\end{tikzpicture}}
\newcommand{\AcircsmallSelect}{\begin{tikzpicture}[scale=0.5,anchor = center]\footnotesize
\fill[black, rounded corners = 2, opacity = 0.2] (0.4,0.7) rectangle (2.3,2.4);
\draw[fill = white] (1.6,1.4) circle (0.5) node[xshift = -0.35cm, yshift = 0.3cm] {$A$} ;
\end{tikzpicture}}
\newcommand{\Bcircsmall}{\begin{tikzpicture}[scale=0.5,anchor = center]\footnotesize
	\draw (1.4,1.5) circle (1) node[xshift = -0.54cm, yshift = 0.39cm] {$B$} ;
	\end{tikzpicture}}
\newcommand{\Ccircsmall}{\begin{tikzpicture}[scale=0.5,anchor = center]\footnotesize
	\draw (3.6,1.5) circle (0.8) node[xshift = 0.45cm, yshift = 0.35cm] {$C$};
	\end{tikzpicture}}
\newcommand{\AsubBsmall}{\begin{tikzpicture}[scale=0.5,anchor = center]\footnotesize
\draw (1.6,1.4) circle (0.5) node[xshift = -0.36cm, yshift = 0.23cm] {$A$} ;
\draw (1.4,1.5) circle (1) node[xshift = -0.54cm, yshift = 0.39cm] {$B$} ;
\end{tikzpicture}}
\newcommand{\AsubBsmallSelectNothing}{\begin{tikzpicture}[scale=0.5,anchor = center]\footnotesize
		\fill[black, rounded corners = 2, opacity = 0.2] (-0.1,0.3) rectangle (2.6,2.7);
		\draw (1.6,1.4) circle (0.5) node[xshift = -0.36cm, yshift = 0.23cm] {$A$} ;
		\draw (1.4,1.5) circle (1) node[xshift = -0.54cm, yshift = 0.39cm] {$B$} ;
\end{tikzpicture}}
\newcommand{\AsubBsmallSelect}{\begin{tikzpicture}[scale=0.5,anchor = center]\footnotesize
\fill[black, rounded corners = 2, opacity = 0.2] (-0.1,0.3) rectangle (2.6,2.7);
\draw[fill= white] (1.4,1.5) circle (1) node[xshift = -0.54cm, yshift = 0.39cm] {$B$} ;
\draw (1.6,1.4) circle (0.5) node[xshift = -0.36cm, yshift = 0.23cm] {$A$} ;
\end{tikzpicture}}
\newcommand{\BdisjCsmall}{\begin{tikzpicture}[scale=0.5,anchor = center]\footnotesize
\draw (1.4,1.5) circle (1) node[xshift = -0.54cm, yshift = 0.39cm] {$B$} ;
\draw (3.6,1.5) circle (0.8) node[xshift = 0.47cm, yshift = 0.35cm] {$C$};
\end{tikzpicture}}
\newcommand{\AsubBdisjCsmall}{\begin{tikzpicture}[scale=0.5,anchor = center]\footnotesize
\draw (1.6,1.4) circle (0.5) node[xshift = -0.36cm, yshift = 0.23cm] {$A$} ;
\draw (1.4,1.5) circle (1) node[xshift = -0.54cm, yshift = 0.39cm] {$B$} ;
\draw (3.6,1.5) circle (0.8) node[xshift = 0.47cm, yshift = 0.35cm] {$C$};
\end{tikzpicture}}

\newcommand{\AsubBdisjCinline}{%
	\adjustbox{scale=0.42,valign=c,raise=0.02cm}{%
		\begin{tikzpicture}[scale=0.5,anchor = center]
		\draw[thick] (1.7,1.4) circle (0.4) node[xshift = -0.32cm, yshift = 0.23cm] {$A$} ;
		\draw[thick] (1.4,1.5) circle (1) node[xshift = -0.58cm, yshift = 0.35cm] {$B$} ;
		\draw[thick] (3.6,1.5) circle (0.8) node[xshift = 0.48cm, yshift = 0.33cm] {$C$};
		\end{tikzpicture}}}

\newtheorem{theorem}{Theorem}[section]
\newtheorem{lemma}{Lemma}[section]
\theoremstyle{definition}
\newtheorem{definition}{Definition}[section]

\newtheorem{example}{Example}[section]

\newcommand{\myscale}{0.8}

\begin{document}

\title{Structure Transfer: an Inference-Based Calculus for the Transformation of Representations}

\author{\name Daniel Raggi \email daniel.raggi@cl.cam.ac.uk \\
       \name Gem Stapleton \email ges55@cam.ac.uk \\
       \name Mateja Jamnik \email mateja.jamnik@cl.cam.ac.uk \\
       \addr University of Cambridge, Cambridge, UK \\
       \AND
       \name Aaron Stockdill \email a.a.stockdill@sussex.ac.uk \\
       \name Grecia Garcia Garcia \email g.garcia-garcia@sussex.ac.uk \\
       \name Peter C-H. Cheng\email p.c.h.cheng@sussex.ac.uk \\
       \addr University of Sussex, Brighton, UK}

% For research notes, remove the comment character in the line below.
% \researchnote

\maketitle

\begin{abstract}
Representation choice is of fundamental importance to our ability to communicate and reason effectively. A major unsolved problem, addressed in this paper, is how to devise \textit{representational-system (RS) agnostic} techniques that drive representation transformation and choice. We present a novel calculus, called \textit{structure transfer}, that enables representation transformation across diverse RSs. Specifically, given a \textit{source} representation drawn from a source RS, the rules of structure transfer allow us to generate a \textit{target} representation for a target RS.
The generality of structure transfer comes in part from its ability to ensure that the source representation and the generated target representation satisfy \textit{any} specified relation (such as semantic equivalence). This is done by exploiting \textit{schemas}, which encode knowledge about RSs. Specifically, schemas can express \textit{preservation of information} across relations between any pair of RSs, and this knowledge is used by structure transfer to derive a structure for the target representation which ensures that the desired relation holds. We formalise this using Representational Systems Theory~\cite{raggi2022rst}, building on the key concept of a \textit{construction space}. The abstract nature of construction spaces grants them the generality to model RSs of diverse kinds, including formal languages, geometric figures and diagrams, as well as informal notations. Consequently, structure transfer is a system-agnostic calculus that can be used to identify alternative representations in a wide range of practical settings.
\end{abstract}

\section{Introduction}\label{sec:introduction}
Symbols mediate how we store, use, and transmit knowledge. We invent, carve, utter and manipulate symbols. We let our symbols make predictions and solve problems for us, and we create machines which we command through symbols we made just for them. Symbols are organised into \textit{systems}, comprising a set of symbols along with their natural and prescribed relations to each other~\cite{Palmer1978fundamental,barwise2019visual,shimojima1999graphic}. Symbolic systems change, and new notations are often conceived. New programming languages are created, each time with a claim of supremacy on \textit{some} aspect. Educators and communicators look for new and more effective ways of presenting information. Given their utility and complexity, we place enormous value on the development of symbolic systems, and in the knowledge of how to use them effectively and what to use them for. But with respect to both logic and cognition, there is no one system of symbols that works best for \textit{every} task, and often a change of representation can determine whether we understand our task, and ultimately whether we can perform it effectively~\cite{polya:htsi,newell1972human,Cheng2001}. Thus, given any complex task for which the use of symbols is required, we face a large variety of symbolic systems -- each with its own advantages -- and it is up to us to decide which one to use, how to represent our task in it, and how to transfer potentially useful knowledge across systems. To be able to do all of this effectively we need to understand the connections between systems.
%%%%

The major contribution of this paper is a calculus for transforming any given representation from a \textit{source} system into another representation in a \textit{target} system. This is done using only the structure of the given representation and knowledge about invariants -- that is, homomorphisms -- between the source and target systems. We formalise our approach using Representational Systems Theory (RST)~\cite{raggi2022rst} and, within this context, a representation is taken to be a syntactic entity whose structure can be encoded by RST. The specific research questions we address are:
\begin{enumerate}[topsep = 4pt, itemsep = 0pt, leftmargin = 2.75em]
	\item[RQ1:] Is there a general calculus for transforming a given representation into a new representation in such a way that any specified relation is guaranteed to hold between them? We show this is possible using the theory of \textit{schemas} developed in this paper.
	\item[RQ2:] Is it possible to do this rigorously and under limited or uncertain knowledge? Our approach explicitly includes derivation of new facts, extending the knowledge base. Moreover, even when insufficient facts are available, partial transformations are supported. The inference framework allows fuzzy or multi-valued logics for reasoning under uncertainty -- and still produce desirable transformations.
\end{enumerate}
As a result of the calculus presented in this paper we can transform between the representations of any two representational systems across any relation, given some knowledge of the invariants across systems, that is, knowledge about \textit{information being preserved through relations}. This has wide ranging applications, demonstrating the potential significance and impact of our new theory for scientific development, communication, and education. Such applications include automatically generating diagrams and figures from formal languages, or improving human-computer interaction in scientific software, theorem provers or computer algebra systems, and enabling creative problem-solving in machines -- which often requires the solver to consider different representations of the problem. We will also address the connections and applications of our calculus to some problems of related areas, from formal methods to analogy% (see Section~\ref{sec:properties})
, demonstrating the generality of our approach. While our methods are general, for concreteness we will focus primarily on one illustrative application: that of \textit{depiction and observation}, where a given statement is depicted graphically, and then consequences of the original statement are observed from the depiction. This is of particular interest to us because it demonstrates the notion of the \textit{observational advantages}~\cite{stapleton:wmaeroiafaoo} of some representations over others. 

\subsection{Example: the depict-and-observe process}\label{depict-and-observe}\vspace*{-1pt}
Consider the statement: \textit{$A$ is contained in $B$, and $B$ is disjoint from $C$}, otherwise written formally as $A \subseteq B \,\wedge\, B \cap C = \emptyset$.
There are infinitely-many conclusions that we could derive\pagebreak[3] from that statement. Of course, the reader may be guessing where we are going: $A \cap C = \emptyset$. This is a simple and obvious semantic consequence which is not syntactically explicit.

We could try to work out how the premise, $A \subseteq B \,\wedge\, B \cap C = \emptyset$, can be manipulated using inference rules, to conclude $A \cap C = \emptyset$. However, it is probably fair to assume that
\begin{wrapfigure}[5]{r}{0.26\linewidth}\vspace*{-2ex}
	%\begin{center}
		\hfill%
		{\begin{tikzpicture}[]
			%\draw (0,0.2) rectangle (4.8,2.9);
			\draw (1.6,1.35) circle (0.4) node[xshift = -0.5cm, yshift = 0.37cm] {$A$} ;
			\draw (1.4,1.5) circle (0.8) node[xshift = -0.75cm, yshift = 0.65cm] {$B$} ;
			\draw (3.1,1.5) circle (0.67) node[xshift = 0.69cm, yshift = 0.58cm] {$C$};
		\end{tikzpicture}}
	%\end{center}
\end{wrapfigure}
most readers will \textit{interpret} the statement and imagine a depiction of the premises, similar to the figure shown here, from which the conclusion, $A \cap C = \emptyset$, can be observed without any additional syntactic manipulations.

Those with a formal mathematical education will easily produce some form of sentential argument wherein we introduce and eliminate symbols such as $\in$ and $\wedge$ (or the words `in' and `and'.). The sentential style of formal reasoning has been investigated exhaustively to the point where many methods and implementations exist for producing such arguments. However, to investigate and produce rigorous methods that formalise the process of depicting a picture and observing a conclusion from it, we need different tools. Specifically, we need to be able to encode diagrams over the same foundations in which we encode typical formal languages, and we need to be able to specify rigorously the relations between diagrammatic systems and sentential/linguistic systems. This, we contend, is provided by RST.

%	Now, going back to the premise $A \subseteq B \wedge B \cap C = \emptyset$, we can see that the first and second items in the conjunction can be represented as
%	\begin{center}
%	\adjustbox{valign = c}{%
%		\begin{tikzpicture}[scale=0.8]\small
%		%\draw (0,0.2) rectangle (2.8,2.9);
%		\draw (1.6,1.4) circle (0.5) node[xshift = -0.42cm, yshift = 0.45cm] {$A$} ;
%		\draw (1.4,1.5) circle (1) node[xshift = -0.75cm, yshift = 0.7cm] {$B$} ;
%		\end{tikzpicture}}
%		\hspace{1cm} and \hspace{1cm}
%		\adjustbox{valign = c}{%
%		\begin{tikzpicture}[scale=0.8]\small
%		%\draw (0,0.2) rectangle (4.8,2.9);
%		\draw (1.4,1.5) circle (1) node[xshift = -0.75cm, yshift = 0.7cm] {$B$} ;
%		\draw (3.6,1.5) circle (0.8) node[xshift = 0.65cm, yshift = 0.6cm] {$C$};
%		\end{tikzpicture}}
%\end{center}
%	respectively. Now we can merge them, identifying the circles labelled $B$, to produce
%	\begin{center}
%	\begin{tikzpicture}[scale=0.8]\small
%		%\draw (0,0.2) rectangle (4.8,2.9);
%		\draw (1.6,1.4) circle (0.5) node[xshift = -0.42cm, yshift = 0.45cm] {$A$} ;
%		\draw (1.4,1.5) circle (1) node[xshift = -0.75cm, yshift = 0.7cm] {$B$} ;
%		\draw (3.6,1.5) circle (0.8) node[xshift = 0.65cm, yshift = 0.6cm] {$C$};
%	\end{tikzpicture}.
%	\end{center}

In this paper we will use the depict-and-observe process as a working example. This will demonstrate how RST can be used to understand, formally, the whole process from depicting the premises to observing the conclusion, using transformations between encodings of Set Algebra and Euler diagrams. Furthermore, we will illustrate the generality of our methods with some examples of a different nature.

%%%%
\subsection{How are different systems connected?}\vspace*{-2pt}
The connections between symbolic systems are sometimes built-in by design (e.g., a high-level programming language and an assembly language), but sometimes they need to be discovered (e.g., the algebraic vs geometric representations of complex numbers). Some connections may underlie cognitive processes without being explicitly encoded into our external symbolic systems, though these are often manifested in our complex and consistent uses of metaphor and analogy~\cite{lakoff2008metaphors}, even in mathematics~\cite{lakoff2000mathematics}. Understanding these relations is fruitful, as it enables the transference of knowledge and expertise from one system to another. For example, in linear algebra, the discovery that linear maps can be represented as matrices (and, more importantly, that operations and properties of linear maps translate to simple operations and properties of matrices), allows us to apply computational methods to solve linear algebra problems. But many of the most commonly used and fruitful transformations exist outside of the purview of formal knowledge. For example, the use of diagrams is ubiquitous in reasoning and learning~\cite{larkin:wadiswttw,barwise:viavr,ainsworth1999functions}, and the benefits of diagrams are well-known~\cite{Cheng2001,cheng2002electrifying,cheng2011probably} and have been studied in terms of \textit{free rides} and \textit{observational advantages}~\cite{stapleton:wmaeroiafaoo,blake:efraoaisv}, yet there is no general theory that satisfactorily reveals how to systematically and effectively transfer information between \textit{any} pair of systems regardless of whether they are sentential or diagrammatic -- and whether they are considered formal or \textit{informal}.

The concepts and methods presented in this paper, built on the foundations of Representational Systems Theory~\cite{raggi2022rst}, exploit the capacity of RST to capture structure from diverse representational systems and extend it with the means for transforming representations. As we will see, the methods presented here are very general, not only because of the expressive generality inherited from RST, but because the transformation calculus is based on a simple principle: the \textit{transfer schema}, a formalism for capturing \textit{invariants across systems through arbitrary relations}, where transfer schema \textit{applications} allow us to derive the structure of the transformed representation. Transfer schemas can capture analogies between concepts in different representational systems, but more specifically they express concisely that \textit{some} information is preserved across relations.

%This kind of fruitful relation between symbolic systems is characterised by the preservation of structure. The relation that maps every complex number $a+bi$ to the vector $(a,b)$ is fruitful because the addition of two complex numbers corresponds with vector addition and their multiplication corresponds to the multiplication of their magnitude and the sum of their angles.
In principle, any relation between the objects of two systems can be used for \textit{encoding} some objects of one system into the other. But the encoding is only fruitful if it preserves information \textit{desirably}. What this means is a question of pragmatics.
If we are interested in the composition and application of linear maps and we have computational power at our disposal, then the canonical transformation that relates each linear map to a matrix is \textit{desirable} because composition and application of linear maps is preserved -- as matrix multiplication -- across the transformation from linear maps into a field of matrices. Moreover, this structure is preserved \textit{desirably} because matrix multiplication is a simple operation given the availability of computational power.

%%%%%

Invariants between algebraic structures are generally captured by the notion of \textit{homomorphism}. These capture preservation of algebraic information (e.g., group operations,
\begin{wrapfigure}[5]{r}{0.39\textwidth}
	\vspace*{-2ex}
	\hfill\adjustbox{scale = 0.9}{%
		\begin{tikzpicture}[node distance=2.5cm, auto]\small
			%\clip (-0.94,-2.11) rectangle (2.87,0.38);
			\node (AA) {$A \times \cdots \times A$};
			\node (A) [above = 0.8cm of AA] {$A$};
			\node (BB) [right = 2.1cm of AA] {$B \times \cdots \times B$};
			\node (B) [above = 0.8cm of BB] {$B$};
			\draw[->] (AA) edge node [above,yshift = -0.06cm] {$(f,\ldots,f)$} (BB);
%			\draw[->, dashed] ([xshift = 0.75cm]AA.south) edge[out = -90, in = -90] node [below] {$f$} ([xshift = 0.75cm]BB.south);
			\draw[->] (A) to node[above, yshift = -0.05cm] {$f$} (B);
			\draw[->] (AA) to node {$\mu_A$} (A);
			\draw[->] (BB) to node [swap] {$\mu_B$} (B);
	\end{tikzpicture}}
\end{wrapfigure}
vector space structure, and so forth.). A homomorphism from a set $A$ to a set $B$ is a function, $f \colon A \to B$, for which an operation $\mu_A$ of $A$ is \textit{preserved} as some operation $\mu_B$ of $B$. This is captured by the equation $f(\mu_A(a_1,\ldots,a_n)) = \mu_B(f(a_1),\ldots,f(a_n))$, and illustrated here (right) by a `commutative' diagram. 
Of course, the map $f$ and the operations $\mu_A$ and $\mu_B$ need not be functions and the diagram will still capture the core principle that
\begin{wrapfigure}[5]{r}{0.45\textwidth}
	\vspace*{-2ex}
	\hfill\adjustbox{scale = 0.9}{%
}%
\end{wrapfigure}

\vspace{-1pt}
\noindent
information is preserved \textit{through the arrows}. Thus we can replace $f$ with arbitrary relations, $R_1,\ldots,R_n,R$, with arbitrary domains and codomains, as illustrated here (right).
  %
%%%%%
%	\centerline{%
%		%
}
 This version readily captures the well-known invariant between linear maps and matrices. Consider the function $\mathit{Mat} \colon  T \to M$ that encodes every
 \begin{wrapfigure}[5]{r}{0.34\textwidth}
 	\vspace*{-2ex}
 	\hfill\adjustbox{scale = 0.9}{%
 		%
%	}\vspace*{-1ex}
%\end{wrapfigure}

\vspace{-1pt}
\noindent
linear map as a matrix, where $T$ is a set of linear transformations and $M$ is a set of matrices. 
Moreover, consider the function $\mathit{Enc}\colon V \to M$ that encodes any vector $(x,y)$ as a matrix\ $\big[$\adjustbox{scale = 0.84, raise = 1pt}{$\footnotesize%
$}$\big]$. Then, the well-known invariant across $\mathit{Mat}$ and $\mathit{Enc}$ is illustrated by the commutative diagram here (right).

As we will see, the concept of \textit{transfer schema} captures invariants across systems at this level of generality\footnote{From a practical perspective, transfer schemas are even more general than this. For instance, the arities of source and target (e.g. $T \times V$ and $M \times M$ in the linear map/matrix example) are not required to be the same (2 in this case).}, and it can be used for re-representation.

%\begin{center}
%	%
%\end{center}

The calculus that we present in this paper, called \textit{structure transfer}, uses transfer schemas %as units of knowledge 
to produce a structure in the target system whose existence guarantees that a desired relation holds between the source and target representations. Our approach has the following properties, which address research questions RQ1 and RQ2 stated \begin{samepage}above:
\begin{enumerate}[itemsep=0pt,topsep=6pt, leftmargin = 2.5em]
	\item \textbf{Representational generality}: structure transfer is applicable to any system that can be captured by RST, thus inheriting its scope in terms of the variety of representational systems that we can apply it to. [RQ1]
	\item \textbf{Relational generality}: any set of relations between the representations across a pair of systems can be used to perform a transformation, as transfer schemas can capture invariants across systems through arbitrary relations. [RQ1]
	\item \textbf{Validity}: given a trusted knowledge-base, a transformation ensures that the desired relation holds between the source and target representations. [RQ1]
	\item \textbf{Partiality and extendability}: schemas are used as units of knowledge. New facts are inferred in the process of structure transfer, which means that the map (knowledge base) only needs to be partially specified a-priori. Existing knowledge will be used in an attempt to produce the transformation, and even when unsuccessful in producing a verified transformation, it may produce a partial, or unverified-yet-correct, transformation. [RQ2]
	\item \textbf{Logic-agnosticism}: we assume very little about the logic in which the schemas are encoded\footnote{In fact, any logic that can be encoded within the RST framework may be used. This means that the logic used to derive a transformation can itself be diagrammatic, or atypical in other ways.}, which opens the door for the use of fuzzy or multi-valued logics to encode uncertain knowledge and still produce transformations under uncertainty. [RQ2]
\end{enumerate}
\end{samepage}
%
%Understanding the connections between systems may be so fruitful that foundational abstractions of, say, mathematics, may be the result of discovering and exploiting analogies (or metaphors) between different domains of knowledge (CITE Lakoff and N\'u\~nez).

Others have built systems for exploiting heterogeneity in ontologies and mathematical knowledge management~\cite{kutz2010carnap,rabe2013scalable,mossakowski2007heterogeneous}, and some have implemented tools for knowledge transference within theorem proving environments~\cite{huffman2013lifting,raggi2016automating}. Some heterogeneous systems have been built with a hard-coded transformation between two kinds of representations~\cite{barwise:hlrwd}. Our use of RST makes our approach more general, as it is not committed to any \textit{kind} of representation. A typical formal definition of \textit{language} is as a set of \textit{words}, defined as strings over a set of symbols, and a typical approach to defining the \textit{grammar} of a language involves rules for composing and decomposing words, yielding some structure (usually a rooted tree). RST weakens these assumptions by abstracting the explicit nature of the \textit{words}, which instead we call \textit{tokens}, and focusing on structure and classification -- that is, how tokens relate to one-another and what their \textit{types} are. This treatment of representational systems opens the door for modelling representations typically consider informal, such as geometric figures, plots, and other kinds of diagrams. But more importantly, by taking these representations seriously and within the same meta-theory as more typical systems, it allows us to reason rigorously about the relations between representations within a system {and} the relations between representations across systems.

In this paper we will focus on \textit{transformation} and the necessary theory for realising them, which includes concepts for effecting \textit{inference} in RST. This is part of a wider project, \textit{rep2rep}, which has a main goal to understand and produce tools that enable \textit{effective choice of representations}. \citeA{raggi2022rst} already introduced the foundations of Representational Systems Theory. Some of our previous work dealt with the problem of representation selection more informally~\cite{raggi2020re} and some dealt with the cognitive aspect~\cite{cheng2021cognitive} of representation processing.

%\dnote{update overview according to final presentation}
\paragraph{Overview of this paper} First we will present preliminary knowledge and notations (Section~\ref{sec:preliminaries}) and an overview of RST (Section~\ref{sec:theory}).
%In Section~\ref{sec:preliminaries} we present some preliminary knowledge and conventions needed for this paper. In Section~\ref{sec:theory} we give an overview of Representational Systems Theory. %, whose foundations are presented in detail in~\citeA{raggi2022rst}. 
The main contributions of this paper come in the sections that follow. In Section~\ref{sec:sequentsAndSchemas} we introduce the concept of a \textit{schema}, for expressing invariants within and across representational systems, leading to the crucial concept of a \textit{transfer schema application}. In Section~\ref{sec:StructureTransfer} we present \textit{structure transfer} with an algorithmic approach. In Section~\ref{sec:properties} we compare structure transfer to well-known methods, such as term rewriting, and we present some applications, including the use and discovery of analogies. We conclude in Section~\ref{sec:conclusion}. 

\section{Preliminary Concepts and Conventions}\label{sec:preliminaries}
Here we present a brief overview of standard concepts needed throughout the paper.

\paragraph{Functions} %We take a set-theoretic view: a function, $f\colon X \to Y$, is a set of ordered pairs such that if $f(a)=b$ (i.e., $(a,b)\in f$) and $f(a)=b'$ then $b=b'$. Given $f_1\colon X_1 \to Y_1$ and $f_2\colon X_2 \to Y_2$, we can form their union and intersection, which may not be functions, where $f_1 \cup f_2 \subseteq (X_1 \cup X_2) \times (Y_1 \cup Y_2)$ and $f_1 \cap f_2 \subseteq (X_1 \cap X_2) \times (Y_1 \cap Y_2)$. 
Given a function, $f\colon X \to Y$, and a set $Z \subseteq X$, the \textbf{restriction} of $f$ to $Z$ is denoted by $f|_Z\colon Z \to Y$. The \textbf{image} of $f$ with its domain restricted to $Z$ is denoted $f[Z]$. %Moreover, given a set $W\subseteq Y$, if $f[X] \subseteq W$

%\gnote{we are over-using ell, here and also for loosening later. suggest keep looesening and change this one here.}
\paragraph{Graph Notation}
We extensively use directed, bipartite graphs, where the vertices and arrows (i.e., directed edges) are labelled. A \textbf{directed labelled bipartite graph}, or simply \textbf{graph}, is a tuple, $g=(V_1,V_2,A,\incVert,l_A,l_1,l_2)$, where: $V_1$ and $V_2$ are disjoint sets of \textbf{vertices}; $A$ is a set of \textbf{arrows} that is disjoint from $V_1\cup V_2$;  $\incVert\colon  A \to (V_1 \times V_2) \cup (V_2 \times V_1)$ is an \textbf{incident vertices function}; and $l_A\colon A\to L_A$, $l_1\colon V_1\to L_1$, and $l_2\colon V_2\to L_2$ are \textbf{labelling functions}, where $L_A$, $L_1$ and $L_2$ are sets of labels. %The three labelling functions are assumed to be total, but sometimes we require $l_A$ to be partial. 
We write $V_1(g)$, $V_2(g)$, and $A(g)$ for the sets of vertices and arrows in $g$ and set $V=V(g)=V_1\cup V_2$. %We define $l = l_g = l_A\cup l_1\cup l_2$. 
The \textbf{neighbourhood} of a vertex, $v$, in $g$ is the smallest subgraph of $g$ containing $v$ and all the arrows incident with $v$. %The sets  of \textbf{incoming} and \textbf{outgoing} arrows of $v$ in $g$ are, resp., $\inA{v}=\{a \in A \colon \exists v' \!\in\! V\; \mathit{iv}(a) \!=\! (v',v)\}$ and $\outA{v}=\{a \in A \colon \exists v' \!\in\! V\; \mathit{iv}(a) \!=\! (v,v')\}$. 
We also use the standard concept of \textbf{subgraph}.

%\gnote{I think we only use the incoming and outgoing sets in one place, so we could omit the defn and just directly write what we mean later: defn of a structure graph}

\paragraph{Graph Operations} %The operation of \textit{removing} one graph, $g$, from another, $g'$, is not straightforward. Suppose we have $v_1 \longrightarrow v_2 \longrightarrow v_3$, and we want to remove $v_2 \longrightarrow v_3$. Should the arrow from $v_1$ to $v_2$ also be removed, to respect the removal of $v_2$, or should $v_2$ be kept, to respect the non-removal of the arrow from $v_1$ to $v_2$? As a pragmatic choice, we take the second approach, where we remove the arrows that appear in both graphs, whilst keeping all the vertices that are necessary to ensure the result is a graph. The result of \textbf{removing}  $g'$ from $g$, denoted $g\setminus g'$, is the minimal subgraph of $g$ that contains $V_1(g) \setminus V_1(g')$, $V_2(g) \setminus V_2(g')$ and $A(g) \setminus A(g')$ as vertices and arrows. Thus, if $g$ is  $v_1 \longrightarrow v_2 \longrightarrow v_3$ and $g'$ is $v_2 \longrightarrow v_3$, then $g \setminus g'$ is $v_1 \longrightarrow v_2$. 
To perform operations on graphs, we need the notion of \textit{compatibility}: graphs $g = (V_1,V_2,A,\incVert,l_A,l_1,l_2)$ and $g' = (V_1',V_2',A',\incVert',l_A',l_1',l_2')$ are \textbf{compatible} provided their argument-wise union, $(V_1 \cup V_1', V_2 \cup V_2', A \cup A', \incVert \cup \incVert', l_A \cup l_A', l_1 \cup l_1', l_2 \cup l_2')$, is a graph. Whenever $g$ and $g'$ are compatible, their \textbf{union}, $g\cup g'=(V_1\cup V_1',...,l_2 \cup l_2')$ and \textbf{intersection}, $g\cap g'=(V_1\cap V_1',...,l_2\cap l_2')$ are graphs.

\paragraph{Graph Morphisms} To define various morphisms from $g = (V_1,V_2,A,\incVert,l_A,l_1,l_2)$ to $g' = (V_1',V_2',A',\incVert',l_A',l_1',l_2')$,  we exploit functions of the form  $f \colon V_1 \cup V_2 \cup A \to V_1' \cup V_2' \cup A'$, where $f[V_1]\subseteq V_1'$, $f[V_2]\subseteq V_2'$ and $f[A]\subseteq A'$. %that maps vertices in $V_1$ to vertices in $V_1'$ and so forth.
We write $f\colon g \to g'$ to mean such a function. A \textbf{homomorphism} is a function, $f \colon g \to g'$, such that for any $a\in A$ if $\incVert(a)= (v_i,v_j)$, then $\incVert'(f(a))= (f(v_i),f(v_j))$. The image, $f[g]$, of $f$ is taken to be the subgraph of $g'$ mapped to by $g$ under $f$. A homomorphism, $f\colon  g \to g'$, is \textbf{label-preserving} provided for any $x \in V_1\cup  V_2\cup  A$ it is the case that $l(x) = l'(f(x))$. We say $f$ is \textbf{label-preserving up to} a set $W \subseteq V_1 \cup V_2 \cup A$ provided it preserves labels for $(V_1 \cup V_2 \cup A) \setminus W$, but not necessarily for $W$. A bijective homomorphism is an \textbf{isomorphism}. Graphs $g$ and $g'$ are \textbf{isomorphic} whenever there exists an isomorphism $f\colon g \to g'$. A homomorphism, $f\colon g \to g'$ is a \textbf{monomorphism} provided $f$ is an isomorphism from $g$ to its image $f[g]$. Lastly, we need homomorphisms that map from, say, $g_1 \cup g_2$ to $h_1 \cup h_2$ that guarantees $g_1$ maps to $h_1$ and $g_2$ to $h_2$: more generally, given tuples, $(g_1,\ldots,g_n)$ and $(h_1,\ldots,h_n)$, we denote by $f\colon (g_1,\ldots,g_n) \to (h_1,\ldots,h_n)$, any homomorphism $f\colon g_1 \cup \ldots \cup g_n \to h_1 \cup \ldots \cup h_n$ such that the image of every $g_i$ under $f$ is contained in $h_i$, that is, $f[g_i] \subseteq h_i$.

\section{Representational System Theory}\label{sec:theory}
Representational System Theory was developed to understand the structure of representations and to facilitate transformations between them~\cite{raggi2022rst}. Most definitions in this section are from \citeA{raggi2022rst}; those without citations are new. In RST, a \textit{representational system} comprises three  spaces: grammatical, entailment and identification spaces. A \textit{grammatical space} encodes the relationship between \textit{tokens}, capturing which tokens build other tokens; for example, $a+1$ is built from $a$, $+$ and $1$. An \textit{entailment space} encodes inferential relations between  tokens, given some notion of entailment; for example, $a+0$ can be rewritten to $a$ given standard arithmetic entailment. An \textit{identification space} encodes properties of the tokens; for example, $3$ uses fewer pixels than $1+2$. To formalise these spaces, RST uses one fundamental abstraction: a \textit{construction space}.

\subsection{From Construction Spaces to Multi-Space Systems}

We present an overview of construction spaces and associated concepts that are necessary for the original contributions of Sections~\ref{sec:sequentsAndSchemas} and~\ref{sec:StructureTransfer}. Building on those concepts, we introduce the novel definition of %a \textit{construction system} and
a \textit{multi-space system}. Firstly, a construction space encodes information about how \textit{tokens} (i.e., representations) relate to one another. To define construction spaces we need the foundational concepts of a \textit{type system} and a \textit{constructor specification}.
\begin{definition}
	A \textbf{type system} is a pair, $\tsystemn = (\types,\leq)$, where $\types$ is a set whose elements are called \textbf{types}, and $\leq$ is a partial order over $\types$.\hfill \cite{raggi2022rst}
\end{definition}

Types encode classes of tokens which are, at a certain level of abstraction, indistinguishable. The partial order, $\leq$, captures hierarchies of abstraction relative to what is considered \textit{relevant} within a representational system.  For example, we can define the type \texttt{two} as an abstraction of instances of the numeral $2$, including two occurrences of it in $2+2$. Whilst types $\texttt{two}$ and $\texttt{three}$ distinguish $2$ and $3$, a \textit{supertype}, $\texttt{numeral}$, may capture both, along with the rest of the numerals. To encode this, we may define $\texttt{two} \leq \texttt{numeral}$ and $\texttt{three} \leq \texttt{numeral}$.

\begin{definition}
A \textbf{constructor specification} over a type system, $\tsystemn = (\types,\leq)$, is a pair, $\cspecificationn = (\constructors,\sig)$, where
	\begin{enumerate}[itemsep=1pt,topsep=6pt]
		\item $\constructors$ is a set, disjoint from $\types$, whose elements are called \textbf{constructors},
		\item $\sig$ % : \constructors \to (\seq(\types), \types)$
		is a function over $\constructors$ such that, for any $c\in \constructors$, the \textbf{signature} of $c$, $\sig(c) = ([\tau_1,\ldots,\tau_n],\tau)$, where $[\tau_1,\ldots,\tau_n]$ is a sequence of types, $\tau$ is a type, and $n > 0$.
	\end{enumerate}
	Given $\sig(c) = ([\tau_1,\ldots,\tau_n],\tau)$, we call $[\tau_1,\ldots,\tau_n]$ the \textbf{input types} of $c$, and $\tau$ the \textbf{output type} of $c$.\hfill \cite{raggi2022rst}
\end{definition}

Constructors encode basic relations between representations, such as the fact that $a$, $+$ and $1$ produce $a+1$ when configured in a certain way. But a constructor may also encode the fact that $P$ and $\neg P \vee Q$ can be used to produce $Q$ through an inference rule.

\begin{example}[\textsc{Set Algebra}]
	The type system, $(\types,\leq)$, of \textsc{Set Algebra} is such that $\types$ includes types $\texttt{var}$, $\texttt{const}$, $\texttt{setExp}$, $\texttt{unaryOp}$, $\texttt{binOp}$, $\texttt{binRel}$, $\texttt{formula}$. This means that the tokens of \textsc{Set Algebra} will include set expressions (including variables and some constants), operator and relation symbols, and formulas. This is relatively typical in the specification (often called the signature) of a formal language. However, in \textsc{Set Algebra} we will also include many types in $\types$, which would typically be considered as terms. For example, for each $A$, $A \cup B$ and $\emptyset \subseteq A$, we will introduce a type: $\texttt{A}$, $\texttt{A\_union\_B}$ or $\texttt{empty\_subset\_A}$. This is a fairly atypical use of types in type theory, but it respects the token/type dichotomy of semiotics \cite{wetzel:tat} and Information Flow theory~\cite{barwise1997information}, while still allowing us to use types for their more typical type-theoretic purposes using their subtype order, $\leq$. For example, $\texttt{A} \leq \texttt{var} \leq \texttt{setExp}$ means that any instance of type $\texttt{A}$ will also fall under types $\texttt{var}$ and $\texttt{setExp}$. Similarly, we have $\texttt{union} \leq \texttt{binOp}$, $\texttt{empty\_subset\_A} \leq \texttt{formula}$, and so forth.
	
The constructors of \textsc{Set Algebra} include \texttt{infixRel}, which infixes a binary relation to produce a formula, with signature $\sig(\texttt{infixRel})=([\texttt{setExp},\texttt{binRel},\texttt{setExp}],\texttt{formula}).$ Given the subtype order, $\texttt{infixRel}$ will take as input things of more \textit{specialised} types, such as $[\texttt{A},\texttt{subset},\texttt{empty\_union\_B}]$, and may yield output $\texttt{A\_subset\_empty\_union\_B}$, where $\texttt{A\_subset\_empty\_union\_B} \leq \texttt{formula}.$
\end{example}

Of course, we have not yet formalised what it means for a constructor to `take inputs' or `yield outputs'. In order to capture this relation between constructors and their inputs and outputs, we use the fundamental concept of a \textit{structure graph}\footnote{\citeA{raggi2022rst} defines structure graphs from the concept of a \textit{configuration}, which is a structure graph with exactly one vertex labelled by a constructor; we do not need this level of conceptual refinement. The definitions are trivially equivalent.}, whose vertices are labelled with constructors and types. The conditions in the definition of structure graph ensure that for every vertex labelled by a constructor, $c$, its input/output arrows are connected to tokens whose types are subtypes of those prescribed by the signature of $c$. Structure graphs generalise the notion of a syntax tree.

\begin{definition} Let $C = (\constructors,\sig)$ be a constructor specification over a type system $(\types,\leq)$. A \textbf{structure graph} for $C$ is a graph, $(\tokens,\pb,A,\incVert,\mathit{index},\type,\consl)$, where
	\begin{enumerate}[itemsep=1pt,topsep=6pt]
		\item $\mathit{iv} \colon A \to (\tokens \times \mathit{Cr}) \cup (\mathit{Cr} \times \tokens)$ is such that, for every $u \in \pb$:
		\begin{enumerate}[itemsep=0pt,topsep=0pt]
            \item $u$ has exactly one outgoing arrow\footnote{The sets of incoming and outgoing arrows of $v$ in $g$ are, respectively, $\inA{v}=\{a \in A \colon \exists v' \!\in\! V\; \mathit{iv}(a) \!=\! (v',v)\}$ and $\outA{v}=\{a \in A \colon \exists v' \!\in\! V\; \mathit{iv}(a) \!=\! (v,v')\}$.}: $|\outA{u}| = 1$,
			\item $u$ has at least one incoming arrow: $|\inA{u}| > 0$,
		\end{enumerate}
        \item $\mathit{index} \colon A \to \mathbb{N}$ is a partial function,
		\item $\type \colon \tokens \to \types$ labels the elements of $\tokens$, called \textbf{tokens}, with types, and
		\item $\consl \colon \mathit{Cr} \to \constructors$ labels the elements of $\mathit{Cr}$, %called \textbf{configurators}, 
		with constructors
	\end{enumerate}
such that for every $a \in A$, and $t\in \tokens$, $u \in \mathit{Cr}$, with $\sig(\consl(u)) = ([\tau_1,\ldots,\tau_n],\tau)$:
		\begin{enumerate}[itemsep=0pt,topsep=4pt]%,topsep=0pt]
            \item[\textbullet] $\mathit{index}|_{\inA{u}}$ is a bijection to $\{1,\ldots,n\}$ and $\mathit{index}|_{\outA{u}}$ is undefined,
			\item[\textbullet] if $\mathit{iv}(a) = (t,u)$ then $\type(t) \leq \tau_{\mathit{index}(a)}$,
			\item[\textbullet] if $\mathit{iv}(a) = (u,t)$ then $\type(t) \leq \tau$.\hfill (Adapted from \citeA{raggi2022rst}.)
		\end{enumerate}
\end{definition}

In summary, a structure graph contains tokens which are assigned types, and they connect through arrows to vertices labelled by constructors. Moreover, for any vertex, $u$, labelled by a constructor, $c$, its inputs and outputs respect the signature of $c$ so that any token is assigned a type which is a subtype of the corresponding one in the signature of $c$. %arrows are assigned numbers by the labelling function such that any token that targets $u$ via the $i^{\text{th}}$ arrow must be assigned a subtype of $\tau_i$

It is important to note that in a structure graph each token, $t$, its assigned type, $\type(t)$, is taken to be its minimum type. %Moreover, the type system does not distinguish any pair of tokens, $t$ and $t'$, which are labelled with the same type; this does not mean that $t$ and $t'$ are actually identical.  For example, the two tokens of the form $\emptyset$ in expression $\emptyset = \emptyset$ are assigned a type $\texttt{empty}$, which is a subtype of $\texttt{const}$ which, in turn, is a subtype of $\texttt{setExp}$. Thus both tokens are an \textit{instance} of $\texttt{empty}$, $\texttt{const}$ and $\texttt{setExp}$. 
In general, we say that $t$ is an instance of all supertypes of $\type(t)$.

\begin{example}[\textsc{Euler Diagrams}]
	This example is based on the design of Euler diagrams described by~\citeA{stapleton2010inductively}. The tokens that represent Euler diagrams are comprised of a set of \textit{simple closed curves} (e.g., circles) superimposed on a canvas. Each of which is assigned a distinct label. 
	For this example we require that every curve gives rise to a pair of complementary non-empty regions: its interior and its exterior.
	This means that each diagram is partitioned into a set of minimal regions, called \textit{zones}, and every zone can be characterised by a set of curves. Specifically, if $L$ is the set of labels of a diagram, then a zone $z$ is characterised by a set $Z \subseteq L$, such that $z$ is interior to all the curves labelled by $Z$ and exterior to the rest. For simplicity, zones will be written as words, for example, we write $\mathtt{AB}$ or $\mathtt{BA}$ to denote a zone which is interior to curves labelled by $\mathtt{A}$ and $\mathtt{B}$ and exterior to any other curves in the diagram. 
	Thus, given a set of labels, $L$, a diagram can be described by a \textit{set of zones} which, as we established, we write as a set of words from alphabet $L$. 
	
	This characterisation of diagrams is the basis for our type system, where we have one type \texttt{label}, and any element in $L$ is a minimal type, subtype of \texttt{label}, and any set of words for $L$ will be a minimal type and a subtype of type \texttt{diagram}. We illustrate that here (note that the tags for zones, in red, are there only for elucidation: they are not part of the diagrams):
	\begin{center}
		\adjustbox{minipage = 6cm}{%
			\centerline{\small \textit{Two tokens of type $\mathtt{\{AB,A,B,\emptyset\}}$:}}\vskip1ex
			\centerline{
}}\hspace{2cm}
		\adjustbox{minipage = 5.5cm}{%
			\centerline{\small \textit{A token of type $\mathtt{\{AB,B,C,\emptyset\}}$:}}\vskip0.5ex
			\centerline{%
}}
	\end{center}
	Moreover, for convenience, we see the regions of a diagram as tokens themselves, also assigned sets of zones as types. Thus, we have type \texttt{region}. Regions are abstractions for a given diagram, but here we will illustrate them by shading their complement:
	\begin{center}
		\adjustbox{minipage = 7.0cm}{%
			{\small \textit{In the context of diagram of type $\mathtt{\{AB,B,\emptyset\}}$, we illus\-trate regions $\mathtt{\{AB,B\}}$ (left) and $\mathtt{\{B\}}$ (right):}}\vskip1ex
			\centerline{%
}}\hfill
		\adjustbox{minipage = 7cm}{%
			{\small \textit{In the context of a diagram of type $\mathtt{\{A,C,\emptyset\}}$, we illustrate regions $\mathtt{\{A\}}$ (left) and $\mathtt{\{A,\emptyset\}}$ (right):}}\vskip1ex
			\centerline{%
}\vspace{0.5ex}}
	\end{center}
	Given our type system, we can introduce constructor \texttt{addCurve}, with signature $([\mathtt{diagram},\linebreak[2]\mathtt{label},\mathtt{region},\mathtt{region}],\mathtt{diagram})$, which takes four inputs and outputs a diagram. Intuitively, \texttt{addCurve} contains instructions for how to add a curve with a new label to a given diagram. The first and second inputs are the diagram and label for the new curve. The third and fourth inputs are the regions that should be, respectively, interior and exterior to the new curve being added. This information is sufficient to determine the type of the output diagram~\cite{stapleton2010inductively}\footnote{The regions that are neither in the interior or exterior are the ones which are cut through by the new curve.}. Below we show two examples of structure graphs that use the constructor \texttt{addCurve}. On the left, we take a diagram \adjustbox{scale = 0.4,raise = -2.5pt}{\AdisjCsmall} and draw a new curve labelled by $B$ with the provision that the region \adjustbox{scale = 0.4,raise = -2.5pt}{\AdisjCselectAsmall} must be interior to $B$, and \adjustbox{scale = 0.4,raise = -2.5pt}{\AdisjCselectCsmall} must be exterior to $B$. This results in a construction of diagram \adjustbox{scale = 0.4,raise = -2.5pt}{\AsubBdisjCsmall}. On the right, the same diagram is drawn in a different order, in addition to displaying an extra step.
	\begin{center}
	\adjustbox{valign = t, scale = 0.95}{
}
	\end{center}
	Structure graphs for Euler diagrams may also include other constructors, such as \texttt{merge}, \texttt{mergeSub} and \texttt{mergeDisj} which represent some more complex operations of Euler diagrams: merging two diagrams that may have overlapping labels, merging with the constraint that some region must be contained in another one, or merging with the constraint that two regions must be disjoint.
\end{example}

Below are some examples of structure graphs for \textsc{Set Algebra} (left) and \textsc{Euler Diagrams} (right) that illustrate possible uses of multiple tokens of the same type vs multiple uses of the same token. In the left-most graph there are two distinct tokens of type \texttt{A} as inputs to a constructor vertex. In the second graph there are two tokens of type \texttt{B}. In the third graph one token is used as input for two different constructor vertices. In the fourth graph one token is used twice as input, and also happens to be the output; making a cycle.
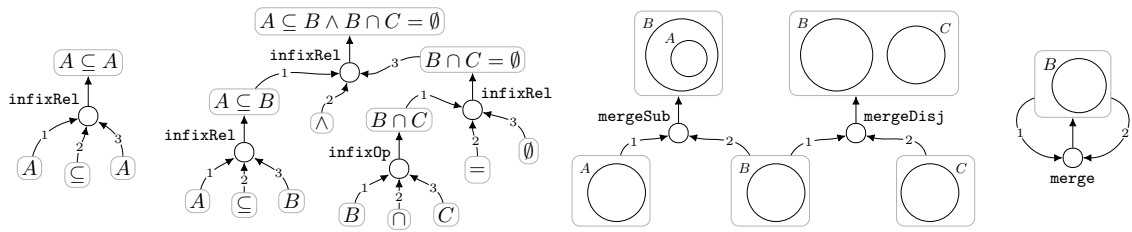
\begin{center}
\adjustbox{valign = c, scale = 0.8}{
}
\end{center}

\pagebreak[4]
By definition, structure graphs may be infinite and have arbitrary complexity. In particular, Definition~\ref{defn:constructionSpace}, of a construction space, requires a potentially infinite structure graph called its \textit{realm}, that essentially represents the universe of all representations encompassed by the construction space.

\begin{definition}\label{defn:constructionSpace}
	A \textbf{construction space} is a triple, $\mathcal{C} = (\tsystemn,\cspecificationn,G)$, where $\tsystemn$ is a type system, $\cspecificationn$ is a constructor specification over $\tsystemn$, and $G$  is a structure graph for $\cspecificationn$ \cite{raggi2022rst}.
	We call $G$ the \textbf{realm} of $\mathcal{C}$. A structure graph, $g$, is called a \textbf{construction graph for $\mathcal{C}$} if it is a subgraph of $G$. In such case we say that $g$ \textbf{belongs} to $\mathcal{C}$.
\end{definition}

The realm of $\mathcal{C}$ determines whether any specific structure graph (such as the ones displayed above) actually \textit{belong} to $\mathcal{C}$. Some graphs which satisfy the rules determined by the constructor specification (e.g., the input and output types are correct) may not live 
in the realm. For example, given reasonable type assignments, the following graphs \textit{are} structure graphs for the \textit{constructor specifications} of \textsc{Set Algebra} and \textsc{Euler Diagrams}, but they \textit{do not} live in their respective realms: one (left) does not distinguish between two distinct token occurrences, and the other (right) distinguishes a pair of tokens which actually only appears once in the token which is being constructed\footnote{Such distinctions are modelling choices; the construction $A \subseteq A$ may belong to a construction space for \textsc{Set Algebra} if we do not care to distinguish between different instances of the same type. If there is only one token of type $A$, can we model the property \textit{is written to the left of}? No!}.
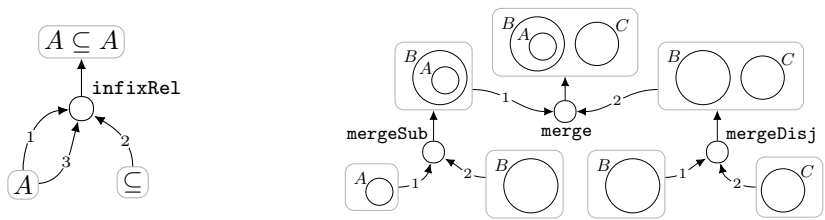
\begin{center}%{r}{0.42\linewidth}\vspace*{-0ex}
%\hfill\adjustbox{minipage = 6.3cm}{%
		\adjustbox{valign = c, scale = 1}{%
}
%}\vspace*{-2ex}
\end{center}
 
It is easy to see that we can specify construction spaces for \textsc{Set Algebra} and \textsc{Euler Diagrams}, and we have demonstrated how the type systems and constructor specifications can be defined. The examples of construction spaces that we have shown for these systems are what we earlier called \textit{grammatical spaces}. In this paper, we do not need to preoccupy ourselves with the three spaces that form a representational system: the methods and results presented here concern construction spaces, with no regard to whether they are the grammatical, entailment, or identification spaces. %Moreover, construction spaces need not be related at all to any representational system.
In order to demonstrate what can be modelled with construction spaces, we give an example below and further examples -- on geometric constructions and proofs --  can be found in Appendix~\ref{secApp:RST}. % (which may or may not be seen as part of any representational system).

\begin{example}[\textsc{Matrix Algebra}]\label{ex:matrices} The tokens of one construction space are matrix expressions where the constructors \texttt{transpose} and \texttt{multiply} take one and two inputs, respectively, and output the resulting expression\footnote{See here that the token $\textsc{T}$ is not an input to the \texttt{transpose} vertex. In relation to~\citeA{raggi2022rst}, the illustrated construction space would \textit{not} typically be considered as being in the realm of a grammatical space for a `matrix algebra' representational system: any such grammatical space would, intuitively, indicate that the token \adjustbox{scale = 0.9}{$\left[\begin{matrix} 3&\hspace*{-3pt} 1 \end{matrix}\right]^{\textsc{T}}$} was built from \adjustbox{scale = 0.9}{$\left[\begin{matrix} 3&\hspace*{-3pt} 1 \end{matrix}\right]$} and \adjustbox{scale = 0.9,raise=-0.5pt}{${\textsc{T}}$}. Thus, this example illustrates the generality of the theory presented in this paper: the structure graph of any construction space is a graph encoding of any (and all) relations of interest, between the tokens, as captured by the constructors.}; see the structure graph on the left. We can also define another construction space where, instead of `putting tokens together' to form composite matrix expressions, the constructors represent calculations in question. For example, below right is a graph\footnote{This structure graph \textit{could} be considered as part of the realm of a matrix algebra entailment space.}, isomorphic to that on the left, but the `output' tokens are the result of evaluating the `input' expressions.
	\begin{center}
		\adjustbox{scale = 0.9, valign = c}{
}
	\end{center}
\end{example}

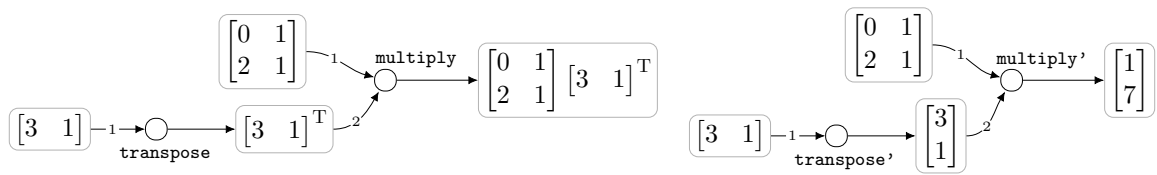
\newcommand{\geometricA}{%
	\begin{tikzpicture}[anchor = center, scale = 1.5]
	\draw (0.6,0.2) circle (.2cm);
	\draw (0.6*1.3,0.2*1.3) circle (0.2*1.3cm);
	\draw (0,0) -- (0.6*1.8,0);
	\draw (0,0) -- (0.6*1.8,0.2*1.8);
	\end{tikzpicture}}
\newcommand{\geometricB}{%
	\begin{tikzpicture}[anchor = center, scale = 1.5]
	\draw (0.6,0.2) circle (.2cm);
%	\draw (0.6*1.3,0.2*1.3) circle (0.2*1.3cm);
	\draw (0,0) -- (0.6*1.8,0);
	\draw (0,0) -- (0.6*1.8,0.2*1.8);
	\end{tikzpicture}}
\newcommand{\geometricC}{%
	\begin{tikzpicture}[anchor = center, scale = 1.5]
	\begin{scope}[shift={(3,5)},rotate=135]
	\draw (0.6,0.2) circle (.2cm);
	\draw (0.6*1.3,0.2*1.3) circle (0.2*1.3cm);
	\draw (0,0) -- (0.6*1.8,0);
	\draw (0,0) -- (0.6*1.8,0.2*1.8);
	\end{scope}
	\end{tikzpicture}}
\newcommand{\geometricD}{%
	\begin{tikzpicture}[anchor = center, scale = 1.5]
	\draw (0.6,0.2) circle (.2cm);
	\draw (0.6*1.3,0.2*1.3) circle (0.2*1.3cm);
	\draw (0,0) -- (0.6*1.8,0);
	\draw (0,0) -- (0.6*1.8,0.2*1.8);
	\end{tikzpicture}}
\newcommand{\geometricE}{%
	\begin{tikzpicture}[anchor = center, scale = 1.5]
	\draw (0.6,0.2) circle (.2cm);
	\draw (0.6*1.3,0.2*1.3) circle (0.2*1.3cm);
	\draw (0,0) -- (0.6*1.8,0);
	\draw (0,0) -- (0.6*1.8,0.2*1.8);
	\end{tikzpicture}}

%% CONSTRUCTION DEFINED HERE %%
%An important structural attribute of any graph is whether its tokens can be constructed in more than one way. We say that a graph, $g$, in a construction space is \textit{unistructured} if every token has at most one incoming arrow in $g$ \cite{raggi2022rst}. One important concept in RST is that of a \textit{construction}, which is a finite unistructured graph where one token is highlighted as \textit{that which is being constructed} \cite{raggi2022rst}. For the purposes of this paper, we work with structure graphs in general, so we do not need a formal definition of a construction. However, most of our examples of structure graphs are constructions. See~\citeA{raggi2022rst} for details.
%
%% END OF CONSTRUCTION DEFN PARA %%

Recall that our major aim is to transform one token into another, such that some \textit{specified relation} holds between them. We need a way to capture \textit{relations} \textit{across} spaces -- not only within a space -- defined using \textit{properties} of tokens.  A \textit{multi-space system}%\footnote{\citeA{raggi2022rst} define a representational system as comprising three construction spaces: grammatical -- $\mathcal{G}$, entailment -- $\mathcal{E}$, and identification -- $\idSpace$ \cite{raggi2022rst}. Both $(\mathcal{G},\idSpace)$  and $(\mathcal{E},\idSpace)$  are multi-space systems, as is $(\mathcal{G}\cup \mathcal{E},\idSpace)$. This implies that the results in this paper readily apply to representational systems.} 
augments a tuple of construction spaces, $\mathcal{C}_1,\ldots,\mathcal{C}_n$, with another construction space, $\goalSpace$, which captures properties of, and relations between, tokens. Here we use the concept of a \textit{meta-space}, which generalises the concept of an identification space from~\citeA{raggi2022rst}.

We now define a \textit{multi-space system} which will express relations both \textit{across the tokens} of $n$ construction spaces and \textit{within} each of the individual construction spaces using a \textit{meta-space}, $\goalSpace$.

\begin{definition}\label{defn:multiSpace}
A tuple of construction spaces, $\multiSpace=(\mathcal{C}_1,\ldots,\mathcal{C}_n,\goalSpace)$, is a \textbf{a multi-space system} provided
\begin{enumerate}[itemsep = 1pt, topsep = 6pt]
\item for each $\mathcal{C}_i$, where $i\leq n$, the type system, $T_{\goalSpace}=(\types_{\goalSpace}, \leq_\goalSpace)$, of $\goalSpace$ extends the type system, $T_{\mathcal{C}_i}=(\types_{\mathcal{C}_i}, \leq_{\mathcal{C}_i})$, of $\mathcal{C}_i$, that is: $\leq_{\mathcal{C}_i}\, \subseteq\, \leq_{\goalSpace}$,
% $(\mathcal{C}_i,\goalSpace)$ is a construction system, and
%
\item $\mathcal{C}_1\cup \cdots\cup \mathcal{C}_n\cup \goalSpace$ is a construction space.
\end{enumerate}
The spaces $\goalSpace$ and, resp., $\mathcal{S}=\mathcal{C}_1\cup \cdots\cup \mathcal{C}_n\cup \goalSpace$ are called the \textbf{meta-space} and, resp., the \textbf{superspace} of $\multiSpace$. %We write $\multiSpace_n$ to mean that the multi-space system $\multiSpace=(\mathcal{C}_1,\ldots,\mathcal{C}_n,\goalSpace)$ comprises $n$ construction spaces as well as $\goalSpace$. 
A tuple of structure graphs, $(g_1,\ldots,g_n,g)$, where each $g_i$ belongs to $\mathcal{C}_i$ and $g$ belongs to $\goalSpace$, is called a \textbf{construction graph}\footnote{Strictly, of course, $(g_1,\ldots,g_n,g)$ is not a graph. We call it a graph for simplicity, noting that the union of the graphs, namely $g_1\cup \cdots\cup g_n\cup g$, is a construction graph for the superspace $\mathcal{S}$.} that \textbf{belongs} to $\multiSpace$.
\end{definition}

%Strictly, of course, $(g_1,\ldots,g_n,g)$ is not a graph. However, the union of the graphs, namely $g_1\cup \cdots\cup g_n\cup g$, is a construction graph for the superspace $\mathcal{S}$. We have blurred this distinction by naming $(g_1,\ldots,g_n,g)$ a construction graph.  The condition of definition~\ref{defn:multiSpace}, that $\leq_{\mathcal{C}_i}\, \subseteq\, \leq_{\goalSpace}$ ensures that all types of $\mathcal{C}_i$ are also types of $\goalSpace$, but $\goalSpace$ may have additional types and a more refined subtype relation. This condition can be thought of as specifying that $\goalSpace$ extends $\mathcal{C}_i$ at the type level. 
The next example shows that meta-spaces can be used for the purpose of encoding the semantics of a construction space. Appendix~\ref{secApp:RST} includes a further example, focusing on low-level properties of tokens drawn from set algebra.

\begin{example}[Semantics for \textsc{Matrix Algebra}]\label{ex:matrices-semantics}
	Take a construction space, $\mathcal{C}$, for matrix algebra as in example~\ref{ex:matrices}, and  meta-space, $\goalSpace$, for $\mathcal{C}$ that contains, for example, a token for each linear function. Then $\goalSpace$ represents some relations using constructors \texttt{mapRep} and \texttt{vectorRep}, take matrix expressions as input and output corresponding vectors. Thus, we can capture the typical linear-algebra semantics of matrices. If the realms of $\mathcal{C}$ and $\goalSpace$ model matrix algebra and its linear algebra semantics correctly, we expect to see that matrix multiplication and linear map application behave similarly. That is, we expect that \textit{whenever} the inputs of \texttt{multiply} and \texttt{apply} are related by  \texttt{mapRep} and \texttt{vectorRep} then the\pagebreak[4] outputs \textit{must} be related by the \texttt{vectorRep} relation. Thus, we will find structure graphs, such as the following, in the realm of $\mathcal{C} \cup \goalSpace$.
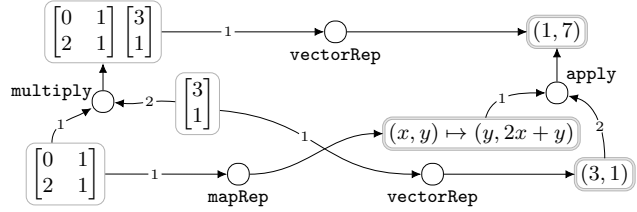
	\begin{center}
		\scalebox{0.9}{
}
	\end{center}
\end{example}

Now, every construction space, $\mathcal{C}$, is a special case of a multi-space system: we have equivalence with the multi-space system $\multiSpace=(\goalSpace)$, where $\goalSpace=\mathcal{C}$. Thus, we will allow ourselves to omit brackets, writing `let $\mathcal{C}$ be a multi-space system' to mean `let $\multiSpace=(\mathcal{C})$ be a multi-space system'. Similarly, we will write `Let $\mathcal{C}$ be a construction space' then use it as if it were a multi-space system. In general, multi-space systems are crucial to our approach to structure transfer, where we seek to identify a token (or tokens) in some construction spaces that are in defined relationships with tokens in other construction spaces.

\subsection{Pattern Graphs}
Given a construction space, $\mathcal{C} = (T,C,G)$, a \textit{pattern graph} for it is a structure graph that satisfies the constraints laid down by the type system, $T$, and constructor specification, $C$, but which need not be part of the realm, $G$. In other words, the vertices of a pattern graph for $\mathcal{C} = (T,C,G)$ are labelled by either types of $T$ or by constructors of $C$, and for any vertex labelled by a constructor, $c$, its inputs and outputs must be in agreement with the signature, $\sig(c)$. A \textit{pattern space}, $\mathcal{P}$, for $\mathcal{C}$, is a space where pattern graphs for $\mathcal{C}$ are taken from, with the condition that the realm, $G$, of $\mathcal{C}$ instantiates the realm, $\Gamma$, of $\mathcal{P}$.

%A \textit{pattern space}, $\mathcal{P}$, for $\mathcal{C}$ will be defined as a construction space over the same type system and constructor specification of $\mathcal{C}$ subject to some conditions that ensure we can map the realm, $G$, of $\mathcal{C}$ into the realm, $\Gamma$, of $\mathcal{P}$. This mapping will ensure that $G$ is label-isomorphic, up to the tokens, to a subgraph of $\Gamma$. Importantly, for each token, $t$, in $G$, it will map to a token, $t'$, of $\Gamma$ where the type of $t$ is a subtype of that of $t'$. In this sense, $G$ \textit{instantiates} a subgraph of $\Gamma$. Pattern spaces will give us freedom to exploit \textit{pattern graphs}, which are graphs that belong to the realm of $\mathcal{P}$. The next example shows how a pattern graph can be used to characterise the structure of classes of graphs in the realm of $\mathcal{C}$ without committing to specific tokens.

\begin{example} Given the construction space of \textsc{Euler Diagrams}, the construction graph on the left is an \textit{instantiation} of the \textit{pattern graph} on the right. The construction graph on the left lives in the realm of \textsc{Euler Diagrams} and it \textit{instantiates} the pattern graph on the right. Many other construction graphs in the realm of \textsc{Euler Diagrams} also instantiate this given pattern graph. For instance, the vertex with type $\mathtt{label}$ could be instantiated with, say, $D$, and the vertices with type $\mathtt{region}$ could be instantiated with other regions. Some instantiations would result in a construction of `{equivalent}' Euler diagrams (to \adjustbox{raise = -1pt}{$\AsubBdisjCinline$}) but some may be genuinely different. Hence, this pattern graph characterises many construction graphs in the realm of \textsc{Euler Diagrams}.
	\begin{center}
\adjustbox{valign = c}{
}
\end{center}
\end{example}
As shown in the previous example, our convention moving forward will be to draw tokens within vertices, while types will be written outside -- labelling the vertices. Thus, when we draw pattern graphs, where the specific tokens are irrelevant, we will draw empty vertices with their types.

\begin{definition}\label{defn:Specialise}
Let $\mathcal{C}=(T,C,G)$ and $\mathcal{P}=(T,C,\Gamma)$ be construction spaces and let $g$ and $\contextGraph$ be construction graphs for $\mathcal{C}$ and $\mathcal{P}$, respectively. We say that $g$ is a \textbf{specialisation} of $\contextGraph$ provided there exists an isomorphism, $\specialise\colon \contextGraph\to g$, that preserves constructor labels, and for each token, $t$, of $\contextGraph$, the type of $\specialise(t)$ in $\mathcal{C}$ is a subtype of the type of $t$ in $\mathcal{P}$ (adapted from \cite{raggi2022rst}). 

If $g$ is a specialisation of $\contextGraph$ then $\contextGraph$ is a \textbf{generalisation} of $g$, and the functions $\specialise$ and $\specialise^{-1}$ are called \textbf{specialisation} and \textbf{generalisation functions}, respectively. We also say that $g$ is an \textbf{instantiation of \boldmath$\contextGraph$ in \boldmath$\mathcal{C}$} and, thus, that $\contextGraph$ is \textbf{instantiatable} in $\mathcal{C}$.
\end{definition}

%\begin{definition}\label{defn:instantiate}
%Let $\mathcal{C}=(T,C,G)$ and $\mathcal{P}=(T,C,\Gamma)$ be construction spaces and let $g$ and $\contextGraph$ be construction graphs that belong to the realms of $\mathcal{C}$ and, resp., $\mathcal{P}$. Then $g$ is an \textbf{instantiation of $\contextGraph$ in $\mathcal{C}$} provided there exists an isomorphism, $\specialise\colon \contextGraph\to g$, that preserves constructor labels, and for each token, $t$, of $\contextGraph$, the type of $t$ in $\contextGraph$ is a supertype of $\specialise(t)$ in $g$ (adapted from \cite{raggi2022rst}). Such a function, $\specialise$, is called is an \textbf{instantiation function}. We also call instantiations \textbf{specialisations} and, when $g$ instantiates $\contextGraph$ we say that $g$ \textbf{specialises} $\contextGraph$.
%\end{definition}

\begin{definition} A \textbf{pattern space} for construction space $\mathcal{C} = (\tsystemn,\cspecificationn,G)$ is a construction space, $\mathcal{P} = (\tsystemn,\cspecificationn,\Gamma)$, where $G$ instantiates some subgraph of $\Gamma$ (adapted from \cite{raggi2022rst}).
A \textbf{pattern graph} for $\mathcal{C}$ is any structure graph that belongs to the realm of some pattern space for $\mathcal{C}$.
\end{definition}

Importantly, we can easily determine whether a graph, $\contextGraph$, is a pattern graph for $\mathcal{C}$: if $\contextGraph$ uses only constructors and types in $\mathcal{C}$ and ensures that configurations of the constructors agree with the constructor's signature as defined in $\mathcal{C}$, then $\contextGraph$ is a pattern graph for $\mathcal{C}$.  We note that Definition~\ref{defn:Specialise}, whilst using our convention of denoting construction space $\mathcal{C}$'s pattern space by $\mathcal{P}$, is not restricted to a construction space, $\mathcal{C}$, and its pattern space, $\mathcal{P}$. Indeed, we often wish to talk of specialisations, $\specialise\colon \contextGraph\to \contextGraph'$, between pattern graphs in a pattern space $\mathcal{P}$. This motivates our use of dual terminology: when working across spaces, we typically talk of instantiations whereas when working within a space we tend to talk of specialisations. The graphs in a pattern space{%\footnote{The definition of a pattern space that we present is more restricted than that in~\cite{raggi2022rst}: here, we require that the entire graph $G$ instantiates some subgraph of $\Gamma$. However, in~\cite{raggi2022rst}, we only require that certain subgraphs of $G$ instantiate subgraphs of $\Gamma$. The additional restriction that the entire graph $G$ instantiates some subgraph of $\Gamma$ arises since our transformation process focuses on arbitrary structure graphs that belong to $\mathcal{C}$. For the purposes of clarity, all pattern spaces for $\mathcal{C}$ in this paper are pattern spaces for $\mathcal{C}$ as defined in~\cite{raggi2022rst}, but not vice-versa.}
}, $\mathcal{P}$, for $\mathcal{C}$ must satisfy the rules given by the type system and constructor specification for $\mathcal{C}$, yet they need not live in the realm of $\mathcal{C}$.  We require that the realm, $G$, of $\mathcal{C}$ is an \textit{instantiation} of some part of the realm, $\Gamma$, of $\mathcal{P}$. Thus, a pattern space, $\mathcal{P}$, provides a supply of pattern graphs which can be used to describe construction graphs that belong to $\mathcal{C}$. Given a construction space, $\mathcal{C}$, we will denote its construction graphs using the Latin alphabet, while its pattern graphs will be denoted using the Greek alphabet. We  depict vertices in pattern spaces with a shaded background, and their labels will represent their assigned type, for example, a token, $t$, with type $\tau$ will be drawn as: \tikz[construction,baseline=-.15cm]{\node[typepos = {$\tau$}{-13}{0.14cm}](){$t$};}\!\!, or simply \tikz[construction,baseline=-.15cm]{\node[typepos = {$\tau$}{-13}{0.14cm}](){};}\!\! when there is no reason to name the token. We now generalise the concept of a pattern space to a \textit{pattern system} for a multi-space system.

\begin{definition}
Let $\multiSpace=(\mathcal{C}_1,\ldots,\mathcal{C}_n,\goalSpace)$ be a multi-space system with superspace $\superSpace$. A \textbf{pattern system} for $\multiSpace$ is a multi-space system, $(\mathcal{P}_1,\ldots,\mathcal{P}_n,\mathcal{P})$, such that
\begin{enumerate}[itemsep = 1pt, topsep = 6pt]
\item each $\mathcal{P}_i$ is a pattern space for $\mathcal{C}_i$,
\item $\mathcal{P}$ is a pattern space for $\goalSpace$, and
\item $\mathcal{P}_1\cup \cdots \cup \mathcal{P}_n\cup  \mathcal{P}$ is a pattern space for $\superSpace$.
\end{enumerate}
 A tuple of pattern graphs, $(\contextGraph_1,\ldots,\contextGraph_n,\contextGraph)$, where each $\contextGraph_i$ is a pattern graph for $\mathcal{C}_i$ and $\contextGraph$ is a pattern graph for $\goalSpace$ is a \textbf{pattern graph} for $\multiSpace$.
\end{definition}

In what follows, whenever we exploit pattern graphs for any $\mathcal{C}_i$ or $\goalSpace$, we assume that they are drawn from a fixed pattern system for $\multiSpace$; we do not make the pattern system explicit. This simply ensures the compatibility of any pair of pattern graph and does not lose generality. In particular, if we have a pattern graph, $\contextGraph_i$, for each $\mathcal{C}_i$ and a pattern graph, $\contextGraph$, for $\goalSpace$ then it is taken that $\contextGraph_1\cup \cdots\cup \contextGraph_n\cup \contextGraph$ is a pattern graph for $\superSpace$.

%Whenever we have a construction space, $\mathcal{C}$, and pattern graphs $\gamma_1,...,\gamma_n$ we assume that  $\gamma_1$ to $\gamma_n$ are drawn from some fixed pattern space, $\mathcal{P}$, for $\mathcal{C}$. We do not make the space $\mathcal{P}$ explicit. Thus, we just assume the compatibility of $\gamma_i$ and $\gamma_j$. In turn, we assume the existence of one function, $\type$, that assigns types to vertices in $\gamma_i$ and $\gamma_j$. Given $\gamma_i$, we define the \textit{instantiations} of $\gamma_i$ in $\mathcal{C}$ as the specialisations of $\gamma_i$ that belong to the realm of $\mathcal{C}$.

\subsection{Special Kinds of Morphisms}

A very important kind of specialisation, $\specialise\colon \contextGraph \to \contextGraph'$, is where $\contextGraph'$ can be instantiated in $\mathcal{C}$. We will often refer to these as \textit{instantiatable} specialisations.

\begin{definition}
Let $\mathcal{C}$ be a construction system with pattern graphs $\contextGraph$ and $\contextGraph'$. We say that $\contextGraph'$ is an \textbf{instantiatable specialisation} of $\contextGraph$ if there exists a specialisation function, $\specialise \colon \contextGraph \to \contextGraph'$ and $\contextGraph'$ can be instantiated. We call $\specialise$ an \textbf{instantiatable specialisation function}.
\end{definition}

 Moving forward, we simplify statements such as `let $\contextGraph'$ be a specialisation of $\contextGraph$ with specialisation function $\specialise\colon \contextGraph \to \contextGraph'$' to simply `let $\specialise\colon \contextGraph \to \contextGraph'$ be a specialisation function', introducing the function and codomain together. When we say \textit{instantiatable} specialisation function, $\specialise\colon \contextGraph \to \contextGraph'$, we mean $\contextGraph'$ can be instantiated in a specified construction space, $\mathcal{C}$.

Now we define a \textit{reification} of a pattern graph, $\contextGraph$: to reify $\contextGraph$ means to enlarge it, by adding vertices and arrows, and to specialise its types. Enlarging $\contextGraph$ adds \textit{structural context}, and replacing types with subtypes makes it \textit{more concrete} because subtypes are `closer' to tokens in $\mathcal{C}$. %In later sections, we will see that each alteration provides additional information that will help us to determine whether $\contextGraph$ has an instantiation in $\mathcal{C}$.

\begin{definition}\label{defn:reification}
Let $\contextGraph$ and $\contextGraph'$ be graphs for construction space $\mathcal{P}$. Then $\contextGraph'$ is a \textbf{reification} of $\contextGraph$ provided there exists a monomorphism, $\reify \colon \contextGraph \to \contextGraph'$, called a \textbf{reification function}, where $\reify[\contextGraph]$  specialises $\contextGraph$.
\end{definition}

Reifications will be used abundantly in this paper. As well as reifications, which specialise types and enlarge a pattern graph, we also require the notion of a \textit{loosening}. A loosening allows us to shrink a pattern graph whilst generalising its types (i.e., replacing types with supertypes). Essentially, a loosening \textit{removes structural context} and makes it \textit{less concrete} because supertypes are `further' from tokens in $\mathcal{C}$. %Each alteration makes it more likely we can instantiate the pattern graph in $\mathcal{C}$.

\begin{definition}\label{defn:looesening}
Let $\contextGraph$ and $\contextGraph'$ be graphs for construction space $\mathcal{P}$. Then $\contextGraph$ is a \textbf{loosening} of $\contextGraph'$  provided there exists a partial surjective function, $\loosen \colon \contextGraph' \to \contextGraph$, called a \textbf{loosening map}, whose inverse is a reification function. Given a set, $W$, of tokens, in $\contextGraph'$, we say that $\loosen$ is a loosening \textbf{up to \boldmath$W$} provided $\loosen$ generalises the types of all tokens in its pre-image that are not in $W$.
\end{definition}

%Loosenings will also be used abundantly in this paper. Both reifications and loosenings build on the key concept of a specialisation.
%\gnote{fix this bit}

Similar language conventions are adopted for \textit{instantiatable} reifications and loosenings, which require the instantiatability of the codomain. 

\begin{example} In the patterns, $\alpha$ and $\alpha'$ depicted below, for \textsc{Set Algebra}, the function, $\reify : \alpha \to \alpha'$, where $\reify(t) = t'$, $\reify(t_1) = t_1'$, $\reify(t_2) = t_2'$ and $\reify(t_3) = t_3'$, is a reification, and it is also an instantiatable reification, as there exists an instantiation, $g$, of $\alpha'$. The partial function $\reify^{-1} : \alpha' \to \alpha$ is a loosening.
	\begin{center}

	\end{center}
\end{example}

The next two lemmas follow trivially from the definitions above.
\begin{lemma}
For any instantiatable reification function, $\reify\colon \contextGraph \to \contextGraph'$, $\contextGraph$ is instantiatable.
\end{lemma}

\begin{lemma}
For any loosening map, $\loosen\colon \contextGraph \to \contextGraph'$, if $\contextGraph$ is instantiatable then so is $\contextGraph'$.
\end{lemma}

We now generalise the concepts of specialisation and reification to multi-space systems.

\begin{definition}
Let $\multiSpace=(\mathcal{C}_1,\ldots \mathcal{C}_n,\goalSpace)$ be a multi-space system. Let $(\contextGraph_1,\ldots,\contextGraph_n,\contextGraph)$ and $(\contextGraph_1',\ldots,\contextGraph_n',\contextGraph')$ be pattern graphs for $\multiSpace$. A homomorphism, $f\colon (\contextGraph_1,\ldots,\contextGraph_n,\contextGraph) \to (\contextGraph_1',\ldots,\contextGraph_n',\contextGraph')$, is a \textbf{specialisation function} (resp., a \textbf{reification function}) in $\multiSpace$ provided $f|_{\contextGraph}\colon \contextGraph \to \contextGraph'$ and all $f|_{\contextGraph_i}\colon \contextGraph_i \to \contextGraph_i'$ for $i \leq n$, are specialisations (resp. reifications).

We say that $(\contextGraph_1',\ldots,\contextGraph_n',\contextGraph')$ \textbf{specialises} (resp. \textbf{reifies}) $(\contextGraph_1,\ldots,\contextGraph_n,\contextGraph)$ in $\multiSpace$.
\end{definition}

To conclude this section, we extend the notion of instantiation to multi-space systems. %and establish that graceful sub-systems behave in the expected way with respect to specialisation and instantiation.

\begin{definition}
Let $\multiSpace=(\mathcal{C}_1,\ldots \mathcal{C}_n,\goalSpace)$ be a multi-space system with  pattern graph $(\contextGraph_1,\ldots,\contextGraph_n,\contextGraph)$. Then $(\contextGraph_1,\ldots,\contextGraph_n,\contextGraph)$ is \textbf{instantiatable} in $\multiSpace$ provided there exists a construction graph,  $(g_1,\ldots,g_n, g)$, that belongs to $\multiSpace$ and specialises $(\contextGraph_1,\ldots,\contextGraph_n,\contextGraph)$. An \textbf{instantiatable specialisation function} is a specialisation function, $\specialise \colon (\contextGraph_1,\ldots,\contextGraph_n,\contextGraph) \to (\contextGraph_1',\ldots,\contextGraph_n',\contextGraph')$, such that $(\contextGraph_1',\ldots,\contextGraph_n',\contextGraph')$ can be instantiated in $\multiSpace$. We say that $(\contextGraph_1',\ldots,\linebreak[2]\contextGraph_n',\linebreak[2]\contextGraph')$ is an instantiatable specialisation of $(\contextGraph_1,\ldots,\contextGraph_n,\contextGraph)$ in $\mathcal{M}$.
\end{definition}

\section{Sequents and Schemas for Multi-Space Systems}\label{sec:sequentsAndSchemas}
Here we will introduce some notions that amount to a meta-logic for deriving knowledge about construction spaces. As we will see, this meta-logic can be used for solving various problems. These problems can be grouped in the following three classes:
\begin{enumerate}[itemsep = 1pt, topsep = 6pt]
\item can we \textit{instantiate} a pattern graph in a construction space? (Section~\ref{sec:instSchema}),
\item can we \textit{infer} that a particular relation holds for instantiations of a pattern graph in a multi-space system? (Section~\ref{sec:inferenceSchema}), and
\item can we \textit{transform} a construction graph in a multi-space system, $\multiSpace$, so that it is ensured that some specified relation holds between the tokens of the resulting graph and the rest of tokens in $\multiSpace$? (Section~\ref{sec:transfer}).
\end{enumerate}
As we will see in Sections~\ref{sec:instSchema} and~\ref{sec:inferenceSchema}, instantiation and inference are two uses of the same process. However, in Section~\ref{sec:transfer} we will see that transformations exploit the same process but with a twist: a target pattern graph is reified, resulting in a new structure which, if instantiated, gives rise to a desired re-representation. This is the idea behind the notion of \textit{structure transfer}.

First we will look at one example of each of the categories above: instantiation, inference, and transformation.
\paragraph{Instantiation.} Can the pattern graph $\gamma$, drawn below, be instantiated in the construction space for \textsc{Euler Diagrams}? 
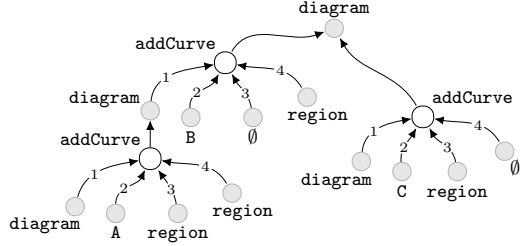
\begin{center}
\adjustbox{scale = 0.9}{
}
\end{center}

\paragraph{Inference.} Can $A \cap C = \emptyset$ be observed from an Euler diagram with type $\mathtt{\{AB,B,C,\emptyset\}}$,\linebreak[2] given the patterns $\contextGraph_1$ (for \textsc{Euler Diagrams}) and $\contextGraph_2$ (for \textsc{Set Algebra})? In other words, can $A \cap C = \emptyset$ can be observed from the diagram $\AsubBdisjCinline$? Note that the \textit{goal} of determining observability is captured by a pattern, $\goa$ (drawn with red dashed arrows), for the meta-space of a multi-space containing \textsc{Euler Diagrams} and \textsc{Set Algebra}.
	\begin{center}
		\adjustbox{scale = 0.9}{
}
	\end{center}

\paragraph{Transformation.} Can we find a pattern graph, $\targetGraph$, which ensures that any instantiation of it contains a token that \textit{depicts} $A \subseteq B \wedge B \cap C = \emptyset$? Similarly to the inference problem, we want to determine whether some pattern graph, $\goa$ (drawn with red dashed arrows), can be instantiated for $\sourceGraph$ and $\targetGraph$, but in this case only $\sourceGraph$ is given, and we wish to find some $\targetGraph$ that ensures the instantiatability of $\goa$. 
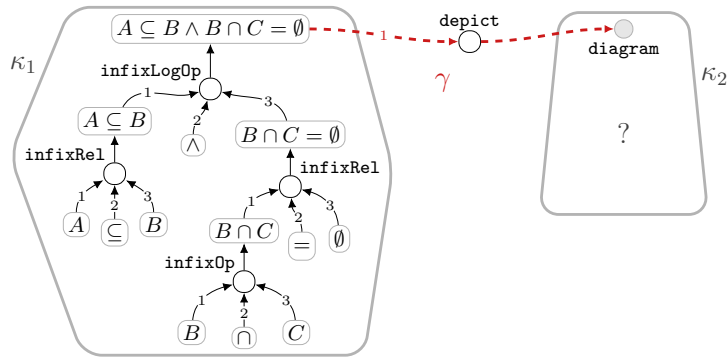
\begin{center}
	\adjustbox{scale = 0.9}{%
}
\end{center}

The problem of transformation, above, is generally motivated by the wish to produce a transformation of one token, $t_1$, into another token, $t_2$, using relations defined in an multi-space system. Later in Section~\ref{sec:transfer} we will see how we might obtain enough information about such a diagram, $t_2$, to actually build it. The process we suggest produces a new pattern graph, $\contextGraph_2$ by reifying it iteratively in such a way that it is ensured that any instantiation of $\contextGraph_2$ results in a $t_2$ that is in the desired relationship with $t_1$. The way that $\contextGraph_2$ is generated is by using \textit{schemas}, which can be understood as units of knowledge that encode invariants within and across different spaces. Informally, some schemas can be understood as capturing analogies across spaces. For example, as we will see, the analogy between the use of the symbol $\subseteq$, in \textsc{Set Algebra} and the actual \textit{containment} of a region inside a curve in \textsc{Euler Diagrams} can be readily captured by a schema. The formalisation of this process, called \textit{structure transfer}, by which schemas are used to produce a target graph from a given graph and some constraints, is not restricted to pairs of tokens, but to any $n$-tuples of structure graphs in a multi-space. Nevertheless, our motivating examples are for $n = 2$. %To illustrate, consider a multi-space system $\multiSpace=(\mathcal{C}_1,\mathcal{C}_2,\goalSpace)$, where $\mathcal{C}_1$ and $\mathcal{C}_2$ are construction spaces for set algebra and, resp., Euler diagrams. We may seek a token, $t_2$, that is of type \texttt{diagram} that depicts the token $t_1$, where $t_1$ is the expression $A \subseteq B \wedge B \cap C = \emptyset$. The transformation relies on having a graph, $\sourceGraph$, that shows how $t_1$ can be \textit{constructed} in $\mathcal{C}_1$.

%
%We are searching for a token, $t_2$, of type \texttt{diagram}, illustrated in the pattern graph $\targetGraph$, that is an output of the \texttt{depict} constructor with input $t_1$, illustrated by the pattern graph, $\goa$, for $\goalSpace$. Crucially, at our starting point we know something about the structure of $A \subseteq B \wedge B \cap C = \emptyset$ -- captured by $\sourceGraph$ -- but we start with no knowledge of the structure of $t_2$, only its type and the relation that must be satisfied. 

We start with the definition of a \textit{sequent}, drawing analogy from sequent calculus: given a collection of pattern graphs representing \textit{assumptions} and another representing a \textit{goal}, drawn from a multi-space system, we can form a \textit{sequent}. Then we define the concept of a \textit{schema}, which is a rule for discerning whether a goal can be instantiated, given an instantiation of the assumptions.
The operation of \textit{applying}  schemas to sequents will capture how a sequent needs to be manipulated in order to derive whether, given an instantiation of the assumptions, there exists a suitable instantiation of the goal. In addition, schemas are foundational for defining \textit{instantiation}, \textit{inference} and \textit{transfer} schemas.

\begin{definition}
%Let $\multiSpace_n$ %$(\mathcal{C}_1,\ldots\mathcal{C}_n)$
%be a multi-space.
%with super-space $\superSpace$.
A \textbf{sequent} for multi-space system $\multiSpace$ is a  pattern graph, $(\contextGraph_1,\ldots,\contextGraph_{n},\assump)$, for $\multiSpace$ and a pattern graph, $\goa$, for the meta-space, $\goalSpace$, denoted $\sequent{\contextGraph_1,\ldots,\contextGraph_n,\assump}{\goa}$. We call $(\contextGraph_1,\ldots,\contextGraph_{n})$ the \textbf{context}, $\assump$ the \textbf{antecedent}, and $\goa$ the \textbf{consequent} of the sequent. If the context and antecedent are empty graphs we will write $\sequent{\ \ }{\goa}$. Similarly, if the consequent is an empty graph we will write $\sequent{\contextGraph_1,\ldots,\contextGraph_n,\assump}{\ \ }$.
\end{definition}
Sequents can be used to capture problems. In the context where $\sequent{\contextGraph_1,\ldots,\contextGraph_n,\assump}{\goa}$ represents a problem we will call $(\contextGraph_1,\ldots,\contextGraph_n,\assump)$ the \textbf{assumptions} and $\goa$ the \textbf{goals} of the problem. The instantiation problem above is captured by a sequent, $\sequent{\ \  }{\goa}$, because the goal is to instantiate $\goa$. The inference and transformation problems are captured by a sequent $\sequent{\contextGraph_1,\contextGraph_2,\assump}{\goa}$ where $\assump$ is empty. To \textit{solve} instantiation and inference problems means, roughly, to determine that the instantiatability of the assumptions implies the instantiatability of the goal. As we will define next, any sequent that satisfies this property is called a \textit{schema}. Thus, solving the instantiation or inference problem $\sequent{\contextGraph_1,\contextGraph_2,\assump}{\goa}$ means proving that it is a schema. The transformation case is different: solving a transformation problem means finding the graph $\contextGraph_2$ for which $\sequent{\contextGraph_1,\contextGraph_2,\assump}{\goa}$ is a schema.

Throughout this paper, we will draw an analogy between the theory of schemas and model theory (more specifically, satisfiability of formulas in the sequent calculus). The property of \textit{being a schema for a construction space}, for sequents, is analogous to the property of \textit{being satisfiable via an interpretation}, for formulas. This is reflected by the language and notation, where we use words like \textit{antecedent} and \textit{consequent}, and the symbol $\Vdash$ in analogy with the turnstile, $\vdash$, of the sequent calculus. The foundation of the analogy is this: to be instantiatable in a construction space is analogous to being satisfiable via some interpretation.

\begin{definition}\label{defn:instantiation-schema}
Let $\sequent{\contextGraph_1,\ldots,\contextGraph_{n},\assump}{\goa}$ be a sequent for multi-space system $\multiSpace$.
Then $\sequent{\contextGraph_1,\ldots,\contextGraph_{n},\assump}{\goa}$ is a \textbf{schema}  for $\multiSpace$ provided for any instantiatable specialisation function, $\specialise \colon (\contextGraph_1,\ldots,\contextGraph_n,\assump) \to (\contextGraph_1',\ldots,\contextGraph_{n}',\assump')$, there exists an instantiatable specialisation,\linebreak[10] $\specialise' \colon (\contextGraph_1,\ldots,\contextGraph_{n},\assump\cup \goa) \to (\contextGraph_1',\ldots,\contextGraph_{n}',\assump'\cup \goa')$ that extends $\specialise$. %We call $(\contextGraph_1,\ldots,\contextGraph_{n})$ the \textbf{context}, $\assump$ the \textbf{antecedent} and $\goa$ the \textbf{consequent} of the schema $\sequent{\contextGraph_1,\ldots,\contextGraph_{n},\assump}{\goa}$.
\end{definition}

Next we will show three examples of schemas in increasing levels of complexity -- first for a case where $\multiSpace = \mathcal{C}$, then for a case where $\multiSpace = (\mathcal{C},\goalSpace)$, and lastly for a case where $\multiSpace = (\mathcal{C},\mathcal{D},\mathcal{G})$.

\begin{example}\label{ex:inst-schema} Let $\mathcal{C}$ be the construction space for \textsc{Set Algebra}. The sequent $\sequent{\assump}{\goa}$, visualised below left, is a schema for $\mathcal{C}$. The pattern graph %$\goa$, contains constructor \texttt{infixOp}, for \textsc{Set Algebra}, and 
	$\assump$ contains exactly the inputs for an $\texttt{infixOp}$ constructor. The pattern graph, $\assump'$ is an example of an instantiatable specialisation of $\assump$. The fact that $\sequent{\assump}{\goa}$ is a schema guarantees that some $\goa'$ must exist for which $\assump' \cup \goa'$ is an instantiatable specialsation of $\assump \cup \goa$.
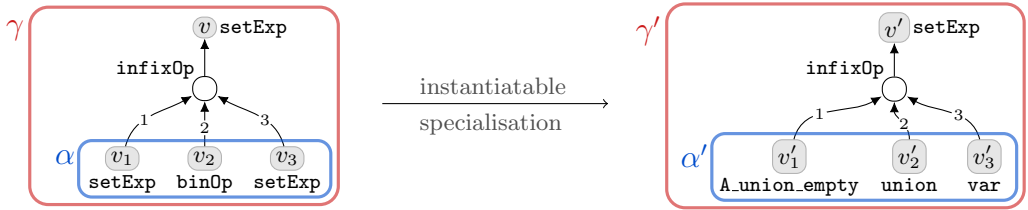
	\begin{center}
		%
}
		\hspace{0.3cm}
		\adjustbox{valign = c}{\tikz{\path[-to] (0,0) edge node[text width = 2cm, text opacity = 0.7, yshift = -0.03cm]{\footnotesize instantiatable specialisation} (3,0);}}
		\hspace{0.0cm}
		\adjustbox{valign = c}{%
	\end{center}
	In other words, $\sequent{\assump}{\goa}$ is a schema because whichever way we restrict the types of $v_1$, $v_2$, and $v_3$, (e.g., to $\texttt{A\_union\_empty}$, $\texttt{union}$, $\texttt{var}$), if we can instantiate $\assump'$ then we can instantiate $\goa'$ respecting the types chosen for the specialisation, $\assump'$, of $\assump$ (e.g, $v'$ can be instantiated to a token of the form $A \cup \emptyset \cup B$). %Note that if we restrict the type of \texttt{binOp} to \texttt{union} and that of \texttt{infixOp} to \texttt{intersection} then the resulting pattern graph, $\goa''$ may not be instantiatable\footnote{Whether the resulting pattern graph is instantiatable depends on the grammatical rules of the \textsc{Set Algebra} system, in particular whether any grammatical conventions have been specified for unambiguously interpreting expressions that omit brackets.}.
\end{example}

\begin{example}\label{ex:fuzzy-schema}
Consider a multi-space system $(\mathcal{C},\goalSpace)$ where $\mathcal{C}$ encodes some basic logic and $\goalSpace$ encodes some relations \textit{fuzzily} -- every number between $0$ and $1$ is a truth value token in $\goalSpace$ and, for all $0\leq \texttt{a} \leq \texttt{b} \leq 1$, the interval $\texttt{[a,b]}$ is a type, where the subtype order, $\preccurlyeq$, satisfies the following: $\texttt{[a,b]} \preccurlyeq \texttt{[c,d]}$ iff $\texttt{c} \leq \texttt{a}$ and $\texttt{b} \leq \texttt{d}$. For any number $0 \leq x \leq 1$, we have that $\type(x) = \texttt{[{x},{x}]}$. Thus, certainty of truth is represented by token $1$ of type $\texttt{[1,1]}$ and certainty of falsity is represented by token $0$ of type $\texttt{[0,0]}$. For all $0\leq \texttt{a} \leq\texttt{b} \leq 1$, the sequent $\sequent{\contextGraph,\assump}{\goa}$, illustrated below, is a schema\footnote{From here on, our drawing convention for any sequent, $\sequent{\contextGraph_1,\ldots,\contextGraph_n,\assump}{\goa}$, is that the context graphs, $\contextGraph_1,\ldots,\contextGraph_n$, will drawn with black arrows, the antecedent, $\assump$, will be drawn with thick blue arrows, and the consequent, $\goa$, will be drawn with dashed red arrows.}.
It captures a generalisation of the fact that the truth value of a formula is one minus the truth value of its negation.
	\begin{center}

	\end{center}
\end{example}

\begin{example}\label{ex:depict-conj} Let $\multiSpace = (\mathcal{C},\mathcal{D},\mathcal{G})$ be a multi-space where $\mathcal{C}$ encodes \textsc{Set Algebra}, $\mathcal{D}$ encodes \textsc{Euler Diagrams} and $\mathcal{G}$ encodes relations across and within the spaces. The sequent $\sequent{\contextGraph_1,\contextGraph_2,\assump}{\goa}$ for $\multiSpace$, depicted below, is as schema. It captures the intuition that a diagram depicts a conjunction if it depicts both conjuncts.
		\begin{center}
			%
		\end{center}
\end{example}

In terms of our three overarching  problems (on instantiation, inference and transformation), we need to answer the following question: if we can instantiate $(\contextGraph_1, \ldots, \contextGraph_n, \assump)$, can we also instantiate $(\contextGraph_1, \ldots, \contextGraph_n, \assump \cup \goa)$? Schemas are our route to an answer, so it is useful to have some results that establish when sequents are also schemas. Lemma~\ref{lem:trivialInstSchema} shows that any sequent where $\goa \subseteq \assump$ is a schema. Secondly, lemma~\ref{lem:InstSchemaSubgraph} asserts that we can shrink the consequent of a known schema to a subraph, $\goa' \subseteq \goa$, and also add some subgraph of the antecedent, $\assump' \subseteq \assump$, to the goal, and the resulting sequent is also schema. Since both lemmas hold trivially, we omit their proofs.

\begin{lemma}\label{lem:trivialInstSchema}
Any sequent $\sequent{\contextGraph_1,\ldots,\contextGraph_n,\assump}{\goa}$ for multi-space system $\multiSpace$ where $\goa \subseteq \assump$ is a schema for $\multiSpace$.
\end{lemma}%

\begin{lemma}\label{lem:InstSchemaSubgraph}
Let  $\sequent{\contextGraph_1,\ldots,\contextGraph_n,\assump}{\goa}$ be a schema for multi-space system $\multiSpace$. Let $\assump'$ and $\goa'$ be pattern graphs for $\multiSpace$'s meta-space such that $\assump'\subseteq \assump$ and $\goa'\subseteq \goa$. Then $\sequent{\contextGraph_1,\ldots,\contextGraph_n,\assump}{\goa'\cup \assump'}$ is a schema for $\multiSpace$.
\end{lemma}

A crucial property of schemas is that they allow us to infer that an instantiation of $(\contextGraph_1, \ldots, \contextGraph_n, \assump \cup \goa)$ exists given the instantiatability of $(\contextGraph_1, \ldots, \contextGraph_n, \assump)$. That is, if an instantiation, $(g_1, \ldots, g_n, g_{\assump})$, of $(\contextGraph_1, \ldots, \contextGraph_n, \assump)$ exists, then an instantiation $(g_1', \ldots, g_n', g_{\assump}'\cup g_{\goa}')$ of $(\contextGraph_1, \ldots, \contextGraph_n, \assump \cup \goa)$ must exist. It is important to note that such $(g_1', \ldots, g_n', g_{\assump}')$ need not have the same tokens as $(g_1, \ldots, g_n, g_{\assump})$, but the types, however minimal, must be the same. This is because any instantiation is itself an instantiatable specialisation, so if we take $\specialise\colon  (\contextGraph_1,\ldots,\contextGraph_n,\assump) \to (g_1,\ldots,g_n,g_{\assump})$ as our instantiatable specialisation function, we can extend this to some instantiatable specialisation function, $\specialise'\colon  (\contextGraph_1,\ldots,\contextGraph_n,\assump) \to (g_1,\ldots,g_n,g_{\assump} \cup \goa')$, noting that $(g_1,\ldots,g_n,g_{\assump} \cup \goa')$ may not be an instantiation itself, but that it is instantiatable. In the illustration below, where $(g_1,\ldots,g_n)$ is empty, notice that while $\specialise$ instantiates $\assump$, $\specialise'$ specialises the type $\tau_4$ to some type $\tau_4'$:
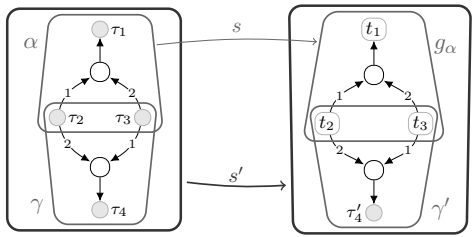
\begin{center}
\adjustbox{scale = \myscale}{%
}
\end{center}
In turn, this guarantees that there exists an instantiation, $g_{\assump}'$, of $g_{\assump}$ whose types are subtypes of the corresponding ones in $g_{\assump}$, that can be extended to an instantiation, $g_{\assump}'\cup g_{\goa}'$, of $g_{\assump}\cup \goa'$ and, therefore, of $\assump\cup \goa$:
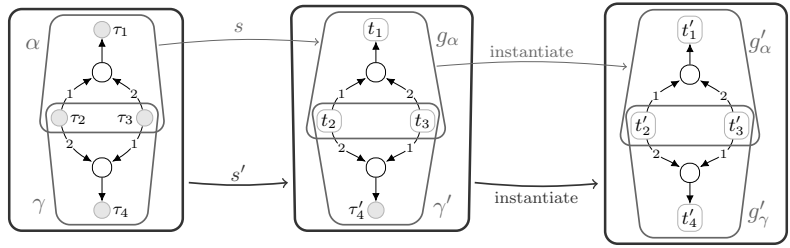
\begin{center}
\adjustbox{scale = \myscale}{%
	%
}
\end{center}
%As is evident in the illustration above, the instantiation of $\antecedent\cup \consequent$ is a reification of $g_{\assump}$.

Moving back to our general context, we are seeking to establish whether a sequent, $\sequent{\contextGraph_1,\ldots,\contextGraph_n,\assump}{\goa}$, is a schema. Our approach is to see whether we can grow $\assump$ and reduce $\goa$ in such a way that $\assump$ envelops $\goa$, enabling us to invoke Lemma~\ref{lem:trivialInstSchema}. Such manipulations, which grow the antecedent or reduce the consequent, are covered by what we call \textit{applications} of a known schema to a sequent (Definition~\ref{defn:instantiation-schema}). If we can apply schemas to a sequent until we reach a schema, then we say that the sequent is \textit{valid} (Definition~\ref{defn:valid-instantiation-sequent}).

In order for the application of schemas to be as intended we need the notion of \textit{monotonicity}. The property of monotonicity is analogous to the property of monotonicity in logic where, given a valid inference, additional facts cannot invalidate it (i.e., the property that if $A \vdash B$ then $A,C \vdash B$ for any $C$). Monotonicity allows us to use schemas to \textit{grow} instantiations. That is, we will know that it is possible to \textit{grow} an instantiation of $\contextGraph_1\cup \cdots\cup \contextGraph_n\cup \assump$ to an instantiation of $\contextGraph_1\cup \cdots\cup \contextGraph_n\cup \assump \cup \goa$.
\begin{definition}
Let  $\sequent{\contextGraph_1,\ldots,\contextGraph_n,\assump}{\goa}$ be a schema and $(\contextGraph_1',\ldots,\contextGraph_n',\assump')$ be a  pattern graph for multi-space system $\multiSpace$. Then  $\sequent{\contextGraph_1,\ldots,\contextGraph_n,\assump}{\goa}$ is \textbf{monotonic} for $(\contextGraph_1',\ldots,\contextGraph_n',\assump')$ in $\multiSpace$ provided $\sequent{\contextGraph_1 \cup \contextGraph_1', \ldots,\contextGraph_n\cup \contextGraph_n',\assump \cup \assump'}{\goa}$ is a schema for $\multiSpace$. Such a pattern graph, $(\contextGraph_1',\ldots,\contextGraph_n',\assump')$, is called a \textbf{monotonic extender} for $\sequent{\contextGraph_1,\ldots,\contextGraph_n,\assump}{\goa}$. 

We say that $\sequent{\contextGraph_1,\ldots,\contextGraph_n,\assump}{\goa}$ is \textbf{monotonic} for $\multiSpace$ provided it is monotonic for every pattern graph, $(\contextGraph_1',\ldots,\contextGraph_n',\assump')$, for $\multiSpace$.
\end{definition}
Many schemas are not monotonic for all pattern graphs but may still have non-trivial monotonic extenders.

\begin{example}\label{ex:non-monotonic} The schema $\sequent{\assump}{\goa}$ for \textsc{Set Algebra} from Example~\ref{ex:inst-schema} is not monotonic for all pattern graphs. The pattern $\assump'$, illustrated below, contains $\assump$ and is instantiatable. Yet, if the vertex labelled with type $\texttt{binOp}$ is specialised with, say, type $\texttt{intersection}$, then $\goa$ cannot be instantiated, as the set expression at the top would need to be instantiated  to $A \cup B$ and also to some expression of the form $x \cap y$. Thus, $\assump'$ is a not a monotonic extender for $\sequent{\assump}{\goa}$.
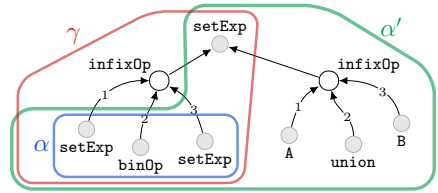
\begin{center}
	\scalebox{\myscale}{
}
\end{center}
\end{example}

At the heart of our approach to defining applications of schemas is the notion of \textit{weakening}. This operation takes a sequent, $\sequent{\contextGraph_1,\ldots,\contextGraph_n,\assump}{\goa}$, and allows us to specialise and augment the assumptions via a reification. In addition, a weakening allows us to generalise the goal and remove parts of it, via a loosening map. %\footnote{Formally, a weakening also allows us to arbitrarily enlarge the goal. Later, Definition~\ref{defn:monotonicweakening} restricts how much we can do this.}. 
As we will see in theorem~\ref{thm:weakening-schema}, a weakening applied to a schema yields also a schema if a monotonicity condition is satisfied given the enlargement of the assumptions and goal, captured by the definition of a \textit{refinement}. This will be crucial for the notion of schema application, as it will allow us to modify a schema to adapt it to the context in which it will be used.

\begin{definition}\label{defn:weakening}
Let $\sequent{\contextGraph_1,\ldots,\contextGraph_n,\assump}{\goa}$ and $\sequent{\contextGraph_1',\ldots,\contextGraph_n',\assump'}{\goa'}$ be sequents for multi-space system $\multiSpace$. 
Then $\sequent{\contextGraph_1',\ldots,\contextGraph_n',\assump'}{\goa'}$ is a \textbf{weakening} of $\sequent{\contextGraph_1,\ldots,\contextGraph_n,\assump}{\goa}$ provided there exists a partial function, $\refine\colon (\contextGraph_1,\ldots,\contextGraph_{n},\assump\cup \goa) \to (\contextGraph_1',\ldots,\contextGraph_{n}',\assump'\cup \goa')$, such that
	\begin{enumerate}[itemsep = 1pt, topsep=6pt]
		\item $\refine|_{(\contextGraph_1,\ldots,\contextGraph_n,\assump)}$ is a reification,
                and
		\item $\refine|_{\goa}$ is a loosening map, up to the tokens in $(\contextGraph_1\cup \cdots\cup \contextGraph_n\cup \assump)\cap \goa$.
              \end{enumerate}
We call $\refine$ a \textbf{weakening map}.
\end{definition}

%\gnote{which version of the lemma is better?}
%
%\begin{lemma}\label{lemma:weakening-image}
%Let $\sequent{\contextGraph_1,\ldots,\contextGraph_n,\assump}{\goa}$ be a sequent for multi-space system $\multiSpace$  with weakening $\sequent{\contextGraph_1',...,\contextGraph_n',\assump'}{\goa'}$ identified by $\refine\colon (\contextGraph_1,...,\contextGraph_n,\assump\cup \goa) \to (\contextGraph_1',...,\contextGraph_n',\assump' \cup\goa')$. If $\sequent{\contextGraph_1,...,\contextGraph_n,\assump}{\goa}$ is a schema for $\multiSpace$ then so is $\sequent{r[\contextGraph_1],...,r[\contextGraph_n],\refine[\assump]}{r[\goa]}$.
%\end{lemma}
To understand weakenings better, let us build more on the analogy to the sequent calculus. Refining a schema is analogous to weakening it by specialising its antecedent and generalising its consequent. For example, in logic if we have a statement, $P(x) \vdash Q(x,3)$, we can weaken it to, $P(2),R \vdash Q(2,z)$, by specialising $x$ to $2$, generalising $z$ to $3$ and adding another assumption, $R$. If we know that $P(x) \vdash Q(x,3)$ is satisfiable then $P(2),R \vdash Q(2,z)$ must also be satisfiable. Now, given that schemas need not be generally monotonic, a weakening may make a schema invalid. However, Definition~\ref{defn:refinement}, of refinement, captures the conditions under which schemas can be weakened to ensure the result is also a schema.

\begin{example}Consider schema $\sequent{\contextGraph,\assump}{\goa}$ from Example~\ref{ex:fuzzy-schema}, for values $a = 0$ and $b = 0.5$, shown below left. This schema asserts that if the truth value of \textit{not P} is a value between $0$ and $0.5$, then the truth value of \texttt{P} is a value between $0.5$ and $1$. The sequent, $\sequent{\contextGraph',\assump'}{\goa'}$, weakens it by augmenting and specialising the context and antecedent, and by generalising the consequent.
	\begin{center}
		\adjustbox{valign = c, scale = 1}{
}
	\hspace{-0.3cm}
	\adjustbox{valign = c, raise = 0.25cm}
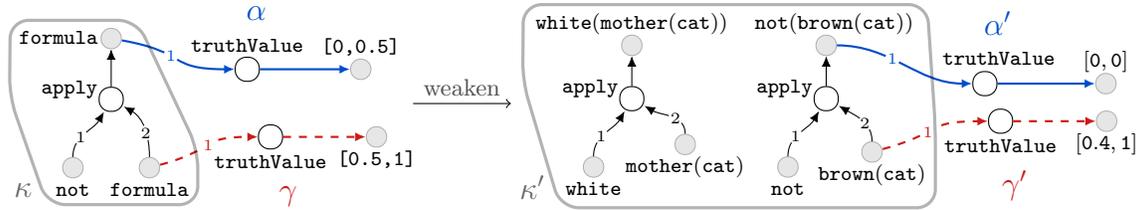
{\tikz{\path[-to] (0,0) edge node[above, text opacity = 0.7, yshift = -0.03cm]{\footnotesize weaken} (1.3,0);}}
	\hspace{-0.25cm}
	\adjustbox{valign = c}{%
}
\end{center}
The weaker version, $\sequent{\contextGraph',\assump'}{\goa'}$ enlarges $\contextGraph$ to $\contextGraph'$, and asserts that if the truth value of $\mathtt{not(brown(cat))}$ is exactly $0$ then the truth value of $\mathtt{brown(cat)}$ is \textit{at least} $0.4$. Of course it is! the value of $\mathtt{brown(cat)}$ is actually $1$, but we have now weakened the schema. 
\end{example}

As this example illustrates, weakening will be useful for contextualising a schema to a problem. Perhaps we need to show that the truth value in the interval $\mathtt{[0.4,1]}$, and we actually have a schema that guarantees a truth value in the interval $\mathtt{[0.5,1]}$. If we can\pagebreak[4] weaken such a schema to fit the goals of the problem, then we will be able to apply it. Of course, this is only useful if weakening actually preserves the property of being a schema, which is not generally true for non-monotonic schemas. In order to show that it is true for some specific kinds of weakenings, we first prove, in the next lemma, that it is true for a restriction of a weakening to its `image'. 

\begin{lemma}\label{lemma:weakening-image}
Let $\sequent{\contextGraph_1,\ldots,\contextGraph_n,\assump}{\goa}$ be a sequent for multi-space system $\multiSpace$ and let $\sequent{\contextGraph_1',\ldots,\linebreak[2]\contextGraph_n',\assump'}{\goa'}$ be a weakening of $\sequent{\contextGraph_1,\ldots,\contextGraph_n,\assump}{\goa}$ with map $\refine$. If $\sequent{\contextGraph_1,\ldots,\contextGraph_n,\assump}{\goa}$ is a schema for $\multiSpace$, then so is $\sequent{r[\contextGraph_1],\ldots,r[\contextGraph_n],\refine[\assump]}{r[\goa]}$.
\end{lemma}

Later, Theorem~\ref{thm:weakening-schema} will build up on Lemma~\ref{lemma:weakening-image} by using the notion of a \textit{refinement}, which both restricts and builds . Theorem~\ref{thm:weakening-schema} is important since it allows us to modify schemas to obtain new schemas. The sufficient conditions are given in Definition~\ref{defn:refinement}. Proofs of both Lemma~\ref{lemma:weakening-image} and Theorem~\ref{thm:weakening-schema} can be found in Appendix~\ref{secApp:IIATS}.

A \textit{refinement} of a schema not only weakens it, but also can augment the consequent $\goa$ to $\goa'$ as long as $\goa'$ does not add anything beyond the assumptions. Importantly, the extra conditions ensure that the refinement of a schema is necessarily a schema.

\begin{definition}\label{defn:refinement}
Let $\sequent{\contextGraph_1,\ldots,\contextGraph_n,\assump}{\goa}$ be a schema and $\sequent{\contextGraph_1',\ldots,\contextGraph_n',\assump'}{\goa'}$ a sequent for multi-space system $\multiSpace$. Then $\sequent{\contextGraph_1',\ldots,\contextGraph_n',\assump'}{\goa'}$  is a \textbf{refinement} of $\sequent{\contextGraph_1,\ldots,\contextGraph_n,\assump}{\goa}$ provided there exists a weakening map, $\refine\colon (\contextGraph_1,\ldots,\contextGraph_{n},\assump\cup \goa) \to (\contextGraph_1',\ldots,\contextGraph_{n}',\assump'\cup \refine[\goa])$, such that
\begin{enumerate}[itemsep = 1pt, topsep = 6pt]
	\item $\sequent{\refine[\contextGraph_1],\ldots,\refine[\contextGraph_n],\refine[\assump]}{r[\goa]}$ is monotonic for $(\contextGraph_1',\ldots,\contextGraph_n',\assump')$ in $\multiSpace$, and
	\item $\goa' \subseteq  \assump'\cup \refine[\goa]$.
\end{enumerate}
We call $\refine$ a \textbf{refinement map}.
%Such a function, $\refine$, \textbf{monotonically identifies} the weakening.
\end{definition}

The weakening of Example~\ref{ex:fuzzy-schema} is a refinement provided the multi-space system is built so that monotonicity holds. Moreover, in this example, $\goa'$ equals $\refine[\goa]$, so the condition $\goa' \subseteq \assump' \cup \refine[\goa]$ is met trivially. In the following example we see $\goa$ being augmented to some $\goa'$ that is larger than $\refine[\goa]$.

\begin{example}\label{ex:weakening} Consider the schema $\sequent{\assump}{\goa}$ from Example~\ref{ex:inst-schema}. The sequent $\sequent{\assump'}{\goa'}$, visualised below, is a refinement of $\sequent{\assump}{\goa}$, with refinement map $\refine \colon (\assump,\goa) \to (\assump',\goa')$, where $\refine(v) = w$, $\refine(v_1) = w_1$, $\refine(v_2) = w_2$, and $\refine(v_3) = w_3$.  % It is clear that $\refine|_{\assump}$ is a reification, and $\refine|_{\goa} \colon \goa \to \refine[\goa]$ is trivially a loosening map, as it 
	\begin{center}
		\adjustbox{valign = t}{
}
	\end{center}
Monotonicity is true by design of the space of \textsc{Set Algebra}: no instantiatiatable specialisation of $\assump'$ \textit{interferes} with our ability to also instantiatably specialise $\goa'$. Moreover, the condition $\goa' \subseteq \assump' \cup \refine[\goa]$ is clearly satisfied because $\goa'=\assump' \cup \refine[\goa]$.
\end{example}

\begin{theorem}\label{thm:weakening-schema}
Let $\sequent{\contextGraph_1,\ldots,\contextGraph_n,\assump}{\goa}$  be a schema for multi-space system $\multiSpace$ with refinement $\sequent{\contextGraph_1',\ldots,\contextGraph_n',\assump'}{\goa'}$. Then  $\sequent{\contextGraph_1',\ldots,\contextGraph_n',\assump'}{\goa'}$ is a schema for $\multiSpace$.
\end{theorem}

%\gnote{need an example for the next defn, showing one forward and one backward application. the example should clearly show the deltas.}
%\dnote{ok, I'll do this. I suppose a simple example where n = 0 will do?}
%\gnote{copied from slack: I think it would be better to have an example where n>1 (ideally n>2). I think this is a good example of where it helps to have examples earlier in the paper that will be used in the later sections. Personally, I am finding it had to think about text around the defns without these yet-to-be included examples. What the paper is really missing is an early section that shows three examples, one for each problem solved (instantiation, inference, transfer), that motivate the theory that is coming up. Then those examples need revisiting in the sections where we introduce the theory that solves the problems. This place is a good spot where we can show how the theory applies to a non-trivial (i.e. proper multispace) example.
%Perhaps this suggests we should decide on what these three examples to motivate the theory should be?
%Instantiation: I suggest this is basically example 5.2, without the details obv.
%Inference: based on example 6.2
%Transfer: I think we need an example with n>2 spaces. What about:
%given $\neg \exists x ((A(x)\vee B(x)\vee C(x)) \wedge \neg D(x))$ and $A\cup B\subseteq D$, find an Euler diagram, $d$, that is semantically equivalent to the former and from which we can observe the latter? We'd need to think about whether we could draw suitable patterns etc to solve this problem.}
Now we are in a position to define the notion of \textit{application} of a schema to a sequent. Applications are done in either \textit{forward} or \textit{backward} manner. Both cases rely on finding a refinement of the schema where the context of the schema is identical to the context of the sequent. For a forward application we try to refine the schema so that the context and antecedent are identical to the sequent's assumptions, and this will induce a consequent that can be added to the assumptions. For a backward application we try to refine the schema so that the consequent is identical to the goal, and this yields an assumption that will replace the goal of the sequent.

\begin{definition}\label{defn:schema-application}
Let $\sequent{\contextGraph_1,\ldots,\contextGraph_n,\assump}{\goa}$  be a schema and $\sequent{\contextGraph_1',\ldots,\contextGraph_n',\assump'}{\goa'}$   be a sequent for multi-space system $\multiSpace$. An \textbf{application} of $\sequent{\contextGraph_1,\ldots,\contextGraph_n,\assump}{\goa}$ to $\sequent{\contextGraph_1',\ldots,\contextGraph_n',\assump'}{\goa'}$ is either
	\begin{enumerate}[itemsep = 1pt, topsep = 6pt]
		\item a sequent of the form $\sequent{\contextGraph_1',\ldots,\contextGraph_n',\assump'\cup \theDelta}{\goa'}$  where $\sequent{\contextGraph_1',\ldots,\contextGraph_n',\assump'}{\theDelta}$  is a refinement of $\sequent{\contextGraph_1,\ldots,\contextGraph_n,\assump}{\goa}$, or
		\item a sequent of the form $\sequent{\contextGraph_1',\ldots,\contextGraph_n',\assump'}{\theDelta}$  where $\sequent{\contextGraph_1',\ldots,\contextGraph_n',\assump'\cup \theDelta}{\goa'}$ is a refinement of $\sequent{\contextGraph_1,\ldots,\contextGraph_n,\assump}{\goa}$.
		%\item a sequent of the form $\sequent{\assump}{\assump'}$ where $\sequent{\assump'}{\goa}$ is a refinement of $\sequent{\antecedent}{\consequent}$ and $\assump$ is a monotonic extender for $\sequent{\assump'}{\goa}$.
	\end{enumerate}
Then $\sequent{\contextGraph_1',\ldots,\contextGraph_n',\assump'\cup \theDelta}{\goa'}$ and $\sequent{\contextGraph_1',\ldots,\contextGraph_n',\assump'}{\theDelta}$ are, respectively, \textbf{\boldmath$\theDelta$-forward} and \textbf{\boldmath$\theDelta$-backward} applications of $\sequent{\contextGraph_1,\ldots,\contextGraph_n,\assump}{\goa}$  to $\sequent{\contextGraph_1',\ldots,\contextGraph_n',\assump'}{\goa'}$. %In addition, $\theDelta$ is called the \textbf{delta} of the application.
\end{definition}

\begin{example}\label{ex:schema-application} Let $\multiSpace = (\mathcal{C},\mathcal{D},\mathcal{G})$ be a multi-space where $\mathcal{C}$ encodes \textsc{Set Algebra}, $\mathcal{D}$ encodes \textsc{Euler Diagrams} and $\mathcal{G}$ encodes relations across and within the spaces. Consider the sequent $\sequent{\contextGraph_1',\contextGraph_2',\assump'}{\goa'}$ where $\assump'$ is the empty graph, as illustrated below. 
	\begin{center}\hspace*{-2cm}

	\end{center}
The schema $\sequent{\contextGraph_1,\contextGraph_2,\assump}{\goa}$ from Example~\ref{ex:depict-conj} tells us that we can show that a diagram depicts a conjunction if it depicts both of the conjuncts. We will apply it backward to $\sequent{\contextGraph_1',\contextGraph_2',\assump'}{\goa'}$. For that purpose we need to refine it to suit our sequent. In other words,\pagebreak[4] we need to find a $\theDelta$ such that $\sequent{\contextGraph_1',\contextGraph_2',\assump'\cup \theDelta}{\goa'}$ is a refinement of $\sequent{\contextGraph_1,\contextGraph_2,\assump}{\goa}$. Such a refinement is illustrated below:
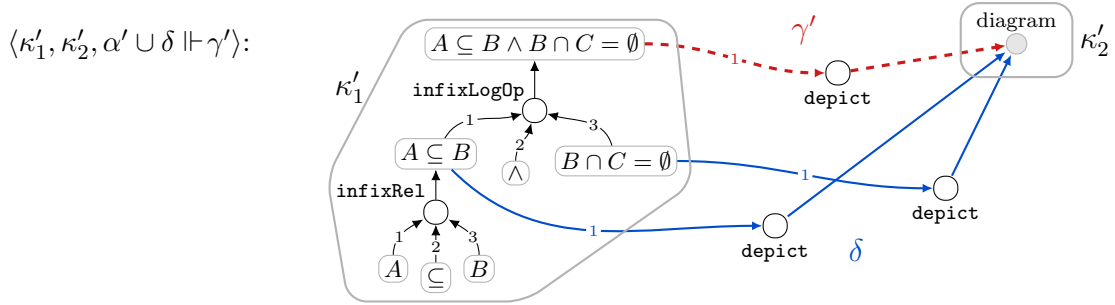
\begin{center}
	\hspace*{-0cm}

\end{center}
	Then, by Definition~\ref{defn:schema-application}, the sequent $\sequent{\contextGraph_1',\contextGraph_2',\assump'}{\theDelta}$, illustrated below, is a $\theDelta$-backward application of $\sequent{\contextGraph_1,\contextGraph_2,\assump}{\goa}$ to $\sequent{\contextGraph_1',\contextGraph_2',\assump'}{\goa'}$:
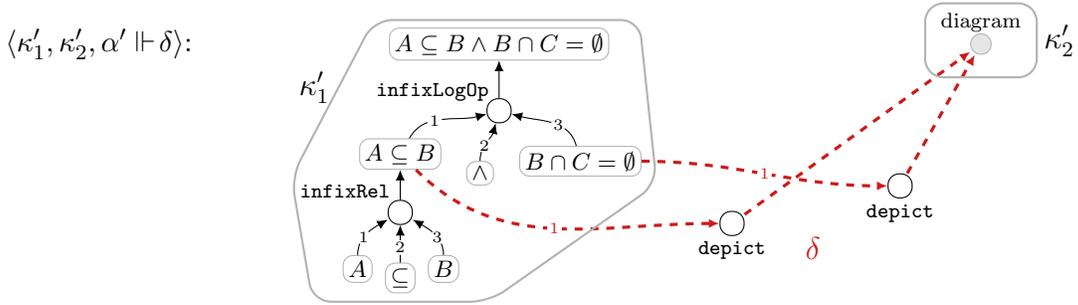
	\begin{center}
		\hspace*{-0.5cm}
		%
	\end{center}
Intuitively, we were able to replace the goal $\goa'$ with two subgoals, captured by $\theDelta$, because the schema could be refined for the task.
\end{example}

The next theorem states that if we apply a schema to a sequent, $\sequent{\contextGraph_1',\ldots,\contextGraph_n',\assump'}{\goa'}$, and we obtain a schema, then $\sequent{\contextGraph_1',\ldots,\contextGraph_n',\assump'}{\goa'}$ must be a schema itself. Its proof can be found in Appendix~\ref{secApp:IIATS}.

\begin{theorem}\label{thm:schema-application-schema}
Let $\sequent{\contextGraph_1,\ldots,\contextGraph_n,\assump}{\goa}$  be a schema and $\sequent{\contextGraph_1',\ldots,\contextGraph_n',\assump'}{\goa'}$   be a sequent for multi-space system $\multiSpace$.  If $\sequent{\contextGraph_1',\ldots,\contextGraph_n',\assump''}{\goa''}$ is an application of $\sequent{\contextGraph_1,\ldots,\contextGraph_n,\assump}{\goa}$ to $\sequent{\contextGraph_1',\ldots,\contextGraph_n',\assump'}{\goa'}$ and a schema for $\multiSpace$ then $\sequent{\contextGraph_1',\ldots,\contextGraph_n',\assump'}{\goa'}$ is also a schema for $\multiSpace$.
\end{theorem}

Ultimately, the purpose of schema applications is to derive whether the goal pattern, $\goa$, drawn from the meta-space, in a sequent can be instantiated, along with the assumptions, $(\contextGraph_1,\ldots,\contextGraph_n,\assump)$. The sequents that result from the iterative application of schemas will be called \textit{valid}. As we will show in Theorem~\ref{thm:valid-sequents}, whose proof is in Appendix~\ref{secApp:IIATS}, valid sequents are themselves schemas. The proof uses an induction approach, with Theorem~\ref{thm:schema-application-schema} covering the inductive step.

\begin{definition}\label{defn:valid-instantiation-sequent}
Let $\mathbb{S}$ be a set of schemas and let $\sequent{\contextGraph_1',\ldots,\contextGraph_n',\assump}{\goa'}$ be a sequent for multi-space system $\multiSpace$.  Then $\sequent{\contextGraph_1',\ldots,\contextGraph_n',\assump}{\goa'}$ is \textbf{valid over \boldmath$\mathbb{S}$} provided either
	\begin{enumerate}[itemsep=1pt,topsep=6pt]
		\item $\sequent{\contextGraph_1',\ldots,\contextGraph_n',\assump'}{\goa'}\in \mathbb{S}$, or
		\item there exists $\sequent{\contextGraph_1,\ldots,\contextGraph_n,\assump}{\goa} \in \mathbb{S}$ and a sequent, $\sequent{\contextGraph_1',\ldots,\contextGraph_n',\assump''}{\goa''}$, for $\multiSpace$ such that $\sequent{\contextGraph_1',\ldots,\contextGraph_n',\assump''}{\goa''}$ is valid over $\mathbb{S}$ and it is an application of $\sequent{\contextGraph_1,\ldots,\contextGraph_n,\assump}{\goa}$ to $\sequent{\contextGraph_1',\ldots,\contextGraph_n',\assump'}{\goa'}$.
	\end{enumerate}
A sequent is \textbf{valid} if it is valid for some set, $\mathbb{S}$, of schemas.
\end{definition}

Theorem~\ref{thm:valid-sequents} tells us that we can use schemas to prove that, given an instantiation of the assumptions, $(\contextGraph_1,\ldots,\contextGraph_{n},\assump)$, we can find an instantiation of the assumptions together with the goal, $(\contextGraph_1,\ldots,\contextGraph_{n},\assump\cup \goa)$.

\begin{theorem}\label{thm:valid-sequents}
Let $\mathbb{S}$ be a set of schemas and let $\sequent{\contextGraph_1',\ldots,\contextGraph_n',\assump'}{\goa'}$ be a sequent for multi-space system $\multiSpace$. If $\sequent{\contextGraph_1',\ldots,\contextGraph_n',\assump'}{\goa'}$ is valid over $\mathbb{S}$ then it is also a  schema for $\multiSpace$.
\end{theorem}

\begin{theorem}\label{cor:instantiation-of-inst-seq}
Let $\sequent{\contextGraph_1,\ldots,\contextGraph_n,\assump}{\goa}$ be a valid sequent for multi-space system $\multiSpace$. If $(\contextGraph_1,\ldots,\contextGraph_n,\assump)$ is instantiatable in $\multiSpace$ then so is $(\contextGraph_1,\ldots,\contextGraph_{n},\assump\cup \goa)$.
\end{theorem}

Example~\ref{appendixSec4:example}, found in Appendix~\ref{secApp:IIATS}, shows a sequence of schema applications in order to derive the validity of a given sequent.

\subsection{Instantiation: determining pattern instantiatability}\label{sec:instSchema}

Here we address the problem of how to determine if a pattern graph in a construction space, $\mathcal{C}=(T,C,G)$, can be instantiated. Recall that the realm, $G$, of $\mathcal{C}$ is a graph that abides by the constraints set out by the type system, $T$, and the constructor specification, $C$. That is, whenever a vertex in the realm, $G$, is labelled by a constructor, $c$, its input and output vertices must be labelled by subtypes of those specified by the signature of $c$. However, not every graph that abides by the constraints lives in $G$. This is simply another way of saying that not every pattern graph, $\contextGraph$, for $\mathcal{C}$ can be instantiated in $\mathcal{C}$. Moreover, though a good modelling principle is that it should be \textit{easy} to decide whether $\contextGraph$ can be instantiated, this is not necessarily the case. Here we state the instantiation problem:

\paragraph{Problem} Given a pattern graph, $\goa$, for space $\mathcal{C}$, is it possible to instantiate $\goa$ in $\mathcal{C}$?

\paragraph{Solution} If we determine that $\sequent{\ \ }{\goa}$ is a schema for multi-space system $\multiSpace=(\mathcal{C})$ then $\goa$ can be instantiated in $\mathcal{C}$.\vskip2ex

Now we can look back at the original instantiation problem proposed in Section~\ref{sec:sequentsAndSchemas}. We can solve using schemas from Examples~\ref{ex:euler-singletons}, \ref{ex:addCurve-bottom-up}, and~\ref{ex:addCurve-top-down}, presented next.

\begin{example}[Singleton instantiation]\label{ex:euler-singletons} For any label, $\eta$, the sequents $\sequent{\ \ }{\goa_1}$, $\sequent{\ \ }{\goa_2}$,  $\sequent{\ \ }{\goa_3}$ and $\sequent{\ \ }{\goa_4}$, pictured below, are schemas for unary multi-space system $\mathcal{D}$.
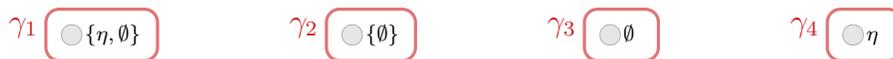
	\begin{center}

	\end{center}
\end{example}

\begin{example}[Bottom-up instantiation of an \texttt{addCurve} configuration]\label{ex:addCurve-bottom-up} Provided that \texttt{addCurve} takes as inputs: a diagram of type $\mathtt{\{x_1,\ldots,x_k,y_1,\ldots,y_m,z_1,\ldots,z_n\}} \leq \mathtt{diagram}$, a label $\eta \leq \mathtt{label}$, an \textit{in} region $\mathtt{\{x_1,\ldots,x_k\}} \leq \mathtt{region}$, and an \textit{out} region $\mathtt{\{z_1,\ldots,z_n\}} \leq \mathtt{region}$, where $\mathtt{\eta}$ does not appear in $\mathtt{\{x_1,\ldots,x_k,y_1,\ldots,y_m,z_1,\ldots,z_n\}}$; that is, if $\mathtt{\eta}$ is not a\pagebreak[4] label in any of the words, \textit{then} we can infer the output. This is captured by the family of schemas $\sequent{\assump}{\goa}$, pictured below, for unary multi-space system $\mathcal{D}$:
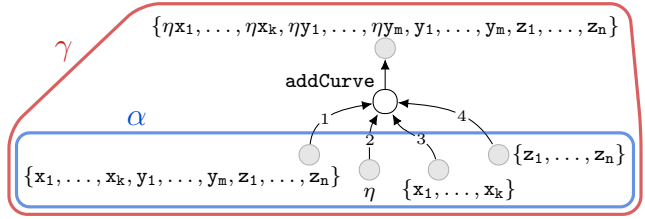
	\begin{center}
		\begin{tikzpicture}[construction]
			\node[typeN = {$\mathtt{\{\eta x_1,\ldots,\eta x_k,\eta y_1,\ldots,\eta y_m,y_1,\ldots,y_m,z_1,\ldots,z_n\}}$}] (t) {};
			\node[constructor = {\texttt{addCurve}}, below = 0.4cm of t] (c) {};
			\node[typepos = {$\mathtt{\{x_1,\ldots,x_k,y_1,\ldots,y_m,z_1,\ldots,z_n\}}$}{-170}{1.57cm},below left = 0.4cm and 0.7cm of c] (t1) {};
			\node[typepos = {$\mathtt{\eta}$}{-90}{0.15cm},below left = 0.6cm and -0.1cm of c] (t2) {};
			\node[typepos = {$\mathtt{\{x_1,\ldots,x_k\}}$}{-45}{0.18cm},below right = 0.6cm and 0.4cm of c] (t3) {};
			\node[typeE = {$\mathtt{\{z_1,\ldots,z_n\}}$},below right = 0.4cm and 1.2cm of c] (t4) {};
			\path[->] (c) edge (t) 
			(t1) edge[out = 85, in = -170] node[index label]{1} (c) 
			(t2) edge[out = 90, in = -120] node[index label]{2} (c) 
			(t3) edge[out = 105, in = -60] node[index label]{3} (c) 
			(t4) edge[out = 120, in = -5] node[index label]{4} (c);
			\coordinate[above left = 0.15cm and 3.75cm of t1] (x1);
			\coordinate[above right = 0.15cm and 1.65cm of t4] (x2);
			\coordinate[below right = 0.55cm and 1.65cm of t4] (x3);
			\coordinate[below left = 0.55cm and 3.75cm of t1] (x4);
			\draw[rounded corners = 5, very thick, darkblue, draw opacity = 0.6] (x1) --node[xshift = -2.5cm, yshift = 0.2cm]{\large$\assump$} (x2) -- (x3) -- (x4) -- cycle;
			\coordinate[above left = 0.45cm and 3.1cm of t] (y1);
			\coordinate[above right = 0.45cm and 3.1cm of t] (y2);
			\coordinate[above right = 0.2cm and 1.75cm of t4] (y3);
			\coordinate[below right = 0.65cm and 1.75cm of t4] (y4);
			\coordinate[below left = 0.65cm and 3.85cm of t1] (y5);
			\coordinate[above left = 0.2cm and 3.85cm of t1] (y6);
			\draw[rounded corners = 7, very thick, darkred, draw opacity = 0.7] (y1) -- (y2) -- (y3) -- (y4) -- (y5) -- (y6) --node[xshift = -0.15cm, yshift = 0.2cm]{\large$\goa$} cycle;
		\end{tikzpicture}
	\end{center}
	Note that in the output, with type $\mathtt{\{\eta x_1,\ldots,\eta x_k,\eta y_1,\ldots,\eta y_m,y_1,\ldots,y_m,z_1,\ldots,z_n\}}$, the curve labelled by $\mathtt{\eta}$:
	\begin{itemize}[itemsep = 0pt, topsep = 4pt]
		\item \textit{contains} every zone, $\mathtt{x_i}$, of the \textit{in} region,
		\item \textit{excludes} every zone, $\mathtt{z_i}$, of the \textit{out} region, and 
		\item \textit{cuts} each zone, $\mathtt{y_i}$, in two: one contained ($\mathtt{\eta y_i}$), and one excluded ($\mathtt{y_i}$).
	\end{itemize} 
	
	If we see the constructor \texttt{addCurve} as a program, this family of schemas represents knowledge about how to  \textit{compute} this program. 
\end{example}

\begin{example}[Top-down instantiation of an \texttt{addCurve} configuration]\label{ex:addCurve-top-down} Provided a diagram of type $\mathtt{\{\eta x_1,\ldots,\eta x_k,\eta y_1,\ldots,\eta y_m,y_1,\ldots,y_m,z_1,\ldots,z_n\}} \leq \mathtt{diagram}$, where $\mathtt{\eta}$ is a label that does not appear in either $\mathtt{x_1,\ldots,x_k,y_1,\ldots,y_m,z_1,\ldots,z_n}$ then we can decompose $\mathtt{\{\eta x_1,\ldots,\eta x_k,}\linebreak[2]\mathtt{\eta y_1,\ldots,\eta y_m,y_1,\ldots,y_m,z_1,\ldots,z_n\}}$ by removing $\mathtt{\eta}$. This is captured by the family of schemas $\sequent{\assump}{\goa}$, pictured below, for unary multi-space system $\mathcal{D}$:
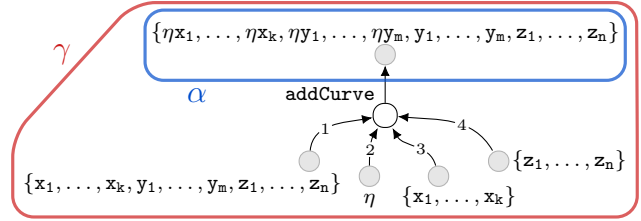
	\begin{center}
		\begin{tikzpicture}[construction]
			\node[typeN = {$\mathtt{\{\eta x_1,\ldots,\eta x_k,\eta y_1,\ldots,\eta y_m,y_1,\ldots,y_m,z_1,\ldots,z_n\}}$}] (t) {};
			\node[constructor = {\texttt{addCurve}}, below = 0.5cm of t] (c) {};
			\node[typepos = {$\mathtt{\{x_1,\ldots,x_k,y_1,\ldots,y_m,z_1,\ldots,z_n\}}$}{-170}{1.57cm},below left = 0.3cm and 0.7cm of c] (t1) {};
			\node[typepos = {$\mathtt{\eta}$}{-90}{0.15cm},below left = 0.5cm and -0.1cm of c] (t2) {};
			\node[typepos = {$\mathtt{\{x_1,\ldots,x_k\}}$}{-45}{0.18cm},below right = 0.5cm and 0.4cm of c] (t3) {};
			\node[typeE = {$\mathtt{\{z_1,\ldots,z_n\}}$},below right = 0.3cm and 1.2cm of c] (t4) {};
			\path[->] (c) edge (t) 
			(t1) edge[out = 85, in = -170] node[index label]{1} (c) 
			(t2) edge[out = 90, in = -120] node[index label]{2} (c) 
			(t3) edge[out = 105, in = -60] node[index label]{3} (c) 
			(t4) edge[out = 120, in = -5] node[index label]{4} (c);
			\coordinate[above left = 0.45cm and 3.05cm of t] (x1);
			\coordinate[above right = 0.45cm and 3.05cm of t] (x2);
			\coordinate[below right = 0.2cm and 3.05cm of t] (x3);
			\coordinate[below left = 0.2cm and 3.05cm of t] (x4);
			\draw[rounded corners = 7, very thick, darkblue, draw opacity = 0.7] (x1) -- (x2) -- (x3) --node[xshift = -2.5cm, yshift = -0.2cm]{\large$\assump$} (x4) -- cycle;
			\coordinate[above left = 0.55cm and 3.2cm of t] (y1);
			\coordinate[above right = 0.55cm and 3.2cm of t] (y2);
			\coordinate[above right = 0.2cm and 1.7cm of t4] (y3);
			\coordinate[below right = 0.6cm and 1.7cm of t4] (y4);
			\coordinate[below left = 0.6cm and 3.8cm of t1] (y5);
			\coordinate[above left = 0.2cm and 3.8cm of t1] (y6);
			\draw[rounded corners = 8, very thick, darkred, draw opacity = 0.7] (y1) -- (y2) -- (y3) -- (y4) -- (y5) -- (y6) --node[xshift = -0.15cm, yshift = 0.2cm]{\large$\goa$} cycle;
		\end{tikzpicture}
	\end{center}
	This family of schemas represents knowledge about how to \textit{parse} an Euler diagram.
\end{example}

\newcommand{\instFirst}{%
	\begin{tikzpicture}[construction, anchor = center]
		\node[typeE = {\texttt{diagram}}] (t1) {};
		\node[constructorNW = {\texttt{addCurve}}, below left = 0.4cm and 1.2cm of t1] (c1) {};
		\node[typepos = {\texttt{diagram}}{165}{0.55cm}, below left = 0.4cm and 0.8cm of c1] (t11) {};
		\node[typeS = {$\mathtt{B}$}, below left = 0.4cm and 0.1cm of c1] (t12) {};
		\node[typeS = {$\mathtt{\emptyset}$}, below right = 0.5cm and 0.1cm of c1] (t13) {};
		\node[typepos = {\texttt{region}}{-60}{0.15cm}, below right = 0.2cm and 0.8cm of c1] (t14) {};
		\node[constructorNW = {\texttt{addCurve}}, below = 0.4cm of t11] (c11) {};
		\node[typepos = {\texttt{diagram}}{-139}{0.25cm}, below left = 0.4cm and 0.7cm of c11] (t111) {};
		\node[typeS = {$\mathtt{A}$}, below left = 0.5cm and 0.1cm of c11] (t112) {};
		\node[typeS = {\texttt{region}}, below right = 0.5cm and 0.1cm of c11] (t113) {};
		\node[typepos = {\texttt{region}}{-60}{0.15cm}, below right = 0.2cm and 0.8cm of c11] (t114) {};
		\node[constructorNE = {\texttt{addCurve}}, below right = 1.1cm and 0.9cm of t1] (c2) {};
		\node[typepos = {\texttt{diagram}}{-145}{0.32cm}, below left = 0.4cm and 0.5cm of c2] (t21) {};
		\node[typeS = {$\mathtt{C}$}, below left = 0.5cm and 0.0cm of c2] (t22) {};
		\node[typeS = {\texttt{region}}, below right = 0.5cm and 0.2cm of c2] (t23) {};
		\node[typepos = {$\mathtt{\emptyset}$}{-60}{0.15cm}, below right = 0.2cm and 1cm of c2] (t24) {};
		
		%\node[typeE = {$\{BD,D,\emptyset\}$}, below right = 0.1cm and 0.7cm of c] (t2) {};
		
		\path[->] 
		(c1) edge[out = 60, in = -160] (t1) 
		(t11) edge[out = 85, in = -170] node[index label]{1} (c1) 
		(t12) edge[out = 95, in = -120] node[index label]{2} (c1) 
		(t13) edge[out = 90, in = -60] node[index label]{3} (c1) 
		(t14) edge[out = 120, in = -10] node[index label]{4} (c1)
		(c2) edge[out = 120, in = -60] (t1) 
		(t21) edge[out = 85, in = -170] node[index label]{1} (c2) 
		(t22) edge[out = 95, in = -120] node[index label]{2} (c2) 
		(t23) edge[out = 90, in = -60] node[index label]{3} (c2) 
		(t24) edge[out = 120, in = -10] node[index label]{4} (c2)
		(c11) edge (t11) 
		(t111) edge[out = 75, in = -170] node[index label]{1} (c11) 
		(t112) edge[out = 90, in = -115] node[index label]{2} (c11) 
		(t113) edge[out = 90, in = -60] node[index label]{3} (c11) 
		(t114) edge[out = 120, in = -10] node[index label]{4} (c11);
		
		\coordinate[above left = 0.15cm and 2.6cm of t1] (a1);
		\coordinate[above right = 0.15cm and 1.8cm of t1] (a2);
		\coordinate[above right = 0.8cm and 0.3cm of t24] (a3);
		\coordinate[below right = 0.5cm and 0.3cm of t24] (a4);
		\coordinate[below right = 0.55cm and 2.5cm of t111] (a5);
		\coordinate[below left = 0.55cm and 0.9cm of t111] (a6);
		\coordinate[above left = 0.8cm and 0.9cm of t111] (a7);
		\draw[rounded corners = 5, very thick, darkred, draw opacity = 0.8] (a1) -- (a2) -- (a3) -- (a4) -- (a5) -- (a6) -- (a7) -- cycle;
\end{tikzpicture}}

\newcommand{\instSecond}{%
	\begin{tikzpicture}[construction, anchor = center]
		\node[typeE = {\texttt{diagram}}] (t1) {};
		\node[constructorNW = {\texttt{addCurve}}, below left = 0.4cm and 1.2cm of t1] (c1) {};
		\node[typepos = {\texttt{diagram}}{165}{0.55cm}, below left = 0.4cm and 0.8cm of c1] (t11) {};
		\node[typeS = {$\mathtt{B}$}, below left = 0.4cm and 0.1cm of c1] (t12) {};
		\node[typeS = {$\mathtt{\emptyset}$}, below right = 0.5cm and 0.1cm of c1] (t13) {};
		\node[typepos = {\texttt{region}}{-60}{0.15cm}, below right = 0.2cm and 0.8cm of c1] (t14) {};
		\node[constructorNW = {\texttt{addCurve}}, below = 0.4cm of t11] (c11) {};
		\node[typepos = {$\mathtt{\{C,\emptyset\}}$}{-120}{0.25cm}, below left = 0.4cm and 0.7cm of c11] (t111) {};
		\node[typeS = {$\mathtt{A}$}, below left = 0.5cm and 0.1cm of c11] (t112) {};
		\node[typeS = {\texttt{region}}, below right = 0.5cm and 0.1cm of c11] (t113) {};
		\node[typepos = {\texttt{region}}{-60}{0.15cm}, below right = 0.2cm and 0.8cm of c11] (t114) {};
		\node[constructorNE = {\texttt{addCurve}}, below right = 1.1cm and 0.9cm of t1] (c2) {};
		\node[typepos = {\texttt{diagram}}{-145}{0.32cm}, below left = 0.4cm and 0.5cm of c2] (t21) {};
		\node[typeS = {$\mathtt{C}$}, below left = 0.5cm and 0.0cm of c2] (t22) {};
		\node[typeS = {\texttt{region}}, below right = 0.5cm and 0.2cm of c2] (t23) {};
		\node[typepos = {$\mathtt{\emptyset}$}{-60}{0.15cm}, below right = 0.2cm and 1cm of c2] (t24) {};
		
		%\node[typeE = {$\{BD,D,\emptyset\}$}, below right = 0.1cm and 0.7cm of c] (t2) {};
		
		\path[->] 
		(c1) edge[out = 60, in = -160] (t1) 
		(t11) edge[out = 85, in = -170] node[index label]{1} (c1) 
		(t12) edge[out = 95, in = -120] node[index label]{2} (c1) 
		(t13) edge[out = 90, in = -60] node[index label]{3} (c1) 
		(t14) edge[out = 120, in = -10] node[index label]{4} (c1)
		(c2) edge[out = 120, in = -60] (t1) 
		(t21) edge[out = 85, in = -170] node[index label]{1} (c2) 
		(t22) edge[out = 95, in = -120] node[index label]{2} (c2) 
		(t23) edge[out = 90, in = -60] node[index label]{3} (c2) 
		(t24) edge[out = 120, in = -10] node[index label]{4} (c2)
		(c11) edge (t11) 
		(t111) edge[out = 75, in = -170] node[index label]{1} (c11) 
		(t112) edge[out = 90, in = -115] node[index label]{2} (c11) 
		(t113) edge[out = 90, in = -60] node[index label]{3} (c11) 
		(t114) edge[out = 120, in = -10] node[index label]{4} (c11);
		
		\coordinate[above left = 0.15cm and 2.6cm of t1] (a1);
		\coordinate[above right = 0.15cm and 1.8cm of t1] (a2);
		\coordinate[above right = 0.8cm and 0.3cm of t24] (a3);
		\coordinate[below right = 0.5cm and 0.3cm of t24] (a4);
		\coordinate[below right = 0.55cm and 2.5cm of t111] (a5);
		\coordinate[below left = 0.55cm and 0.6cm of t111] (a6);
		\coordinate[above left = 1.6cm and 0.6cm of t111] (a7);
		\draw[rounded corners = 5, very thick, darkred, draw opacity = 0.8] (a1) -- (a2) -- (a3) -- (a4) -- (a5) -- (a6) -- (a7) -- cycle;
		
		\coordinate[above left = 0.1cm and 0.5cm of t111] (b1);
		\coordinate[above right = 0.1cm and 0.15cm of t111] (b2);
		\coordinate[below right = 0.45cm and 0.15cm of t111] (b3);
		\coordinate[below left = 0.45cm and 0.5cm of t111] (b4);
		\draw[rounded corners = 4, very thick, darkblue, draw opacity = 0.8] (b1) -- (b2) -- (b3) -- (b4) -- cycle;
\end{tikzpicture}}

\newcommand{\instThird}{%
	\begin{tikzpicture}[construction, anchor = center]
		\node[typeE = {\texttt{diagram}}] (t1) {};
		\node[constructorNW = {\texttt{addCurve}}, below left = 0.4cm and 1.2cm of t1] (c1) {};
		\node[typepos = {\texttt{diagram}}{165}{0.55cm}, below left = 0.4cm and 0.8cm of c1] (t11) {};
		\node[typeS = {$\mathtt{B}$}, below left = 0.4cm and 0.1cm of c1] (t12) {};
		\node[typeS = {$\mathtt{\emptyset}$}, below right = 0.5cm and 0.1cm of c1] (t13) {};
		\node[typepos = {\texttt{region}}{-60}{0.15cm}, below right = 0.2cm and 0.8cm of c1] (t14) {};
		\node[constructorNW = {\texttt{addCurve}}, below = 0.4cm of t11] (c11) {};
		\node[typepos = {$\mathtt{\{C,\emptyset\}}$}{-120}{0.25cm}, below left = 0.4cm and 0.7cm of c11] (t111) {};
		\node[typeS = {$\mathtt{A}$}, below left = 0.5cm and 0.1cm of c11] (t112) {};
		\node[typeS = {\texttt{region}}, below right = 0.5cm and 0.1cm of c11] (t113) {};
		\node[typepos = {\texttt{region}}{-60}{0.15cm}, below right = 0.2cm and 0.8cm of c11] (t114) {};
		\node[constructorNE = {\texttt{addCurve}}, below right = 1.1cm and 0.9cm of t1] (c2) {};
		\node[typepos = {\texttt{diagram}}{-145}{0.32cm}, below left = 0.4cm and 0.5cm of c2] (t21) {};
		\node[typeS = {$\mathtt{C}$}, below left = 0.5cm and 0.0cm of c2] (t22) {};
		\node[typeS = {\texttt{region}}, below right = 0.5cm and 0.2cm of c2] (t23) {};
		\node[typepos = {$\mathtt{\emptyset}$}{-60}{0.15cm}, below right = 0.2cm and 1cm of c2] (t24) {};
		
		%\node[typeE = {$\{BD,D,\emptyset\}$}, below right = 0.1cm and 0.7cm of c] (t2) {};
		
		\path[->] 
		(c1) edge[out = 60, in = -160] (t1) 
		(t11) edge[out = 85, in = -170] node[index label]{1} (c1) 
		(t12) edge[out = 95, in = -120] node[index label]{2} (c1) 
		(t13) edge[out = 90, in = -60] node[index label]{3} (c1) 
		(t14) edge[out = 120, in = -10] node[index label]{4} (c1)
		(c2) edge[out = 120, in = -60] (t1) 
		(t21) edge[out = 85, in = -170] node[index label]{1} (c2) 
		(t22) edge[out = 95, in = -120] node[index label]{2} (c2) 
		(t23) edge[out = 90, in = -60] node[index label]{3} (c2) 
		(t24) edge[out = 120, in = -10] node[index label]{4} (c2)
		(c11) edge (t11) 
		(t111) edge[out = 75, in = -170] node[index label]{1} (c11) 
		(t112) edge[out = 90, in = -115] node[index label]{2} (c11) 
		(t113) edge[out = 90, in = -60] node[index label]{3} (c11) 
		(t114) edge[out = 120, in = -10] node[index label]{4} (c11);
		
		\coordinate[above left = 0.15cm and 2.6cm of t1] (a1);
		\coordinate[above right = 0.15cm and 1.8cm of t1] (a2);
		\coordinate[above right = 0.8cm and 0.3cm of t24] (a3);
		\coordinate[below right = 0.5cm and 0.3cm of t24] (a4);
		\coordinate[below right = 0.55cm and 2.5cm of t111] (a5);
		\coordinate[below left = 0.55cm and 0.6cm of t111] (a6);
		\coordinate[above left = 1.6cm and 0.6cm of t111] (a7);
		\draw[rounded corners = 5, very thick, darkred, draw opacity = 0.8] (a1) -- (a2) -- (a3) -- (a4) -- (a5) -- (a6) -- (a7) -- cycle;
		
		\coordinate[above left = 0.1cm and 0.5cm of t111] (b1);
		\coordinate[above right = 0.1cm and 0.1cm of t112] (b2);
		\coordinate[below right = 0.35cm and 0.1cm of t112] (b3);
		\coordinate[below left = 0.45cm and 0.5cm of t111] (b4);
		\draw[rounded corners = 4, very thick, darkblue, draw opacity = 0.8] (b1) -- (b2) -- (b3) -- (b4) -- cycle;
	\end{tikzpicture}
}

\newcommand{\instFourth}{%
	\begin{tikzpicture}[construction, anchor = center]
		\node[typeE = {\texttt{diagram}}] (t1) {};
		\node[constructorNW = {\texttt{addCurve}}, below left = 0.4cm and 1.2cm of t1] (c1) {};
		\node[typepos = {\texttt{diagram}}{165}{0.55cm}, below left = 0.4cm and 0.8cm of c1] (t11) {};
		\node[typeS = {$\mathtt{B}$}, below left = 0.4cm and 0.1cm of c1] (t12) {};
		\node[typeS = {$\mathtt{\emptyset}$}, below right = 0.5cm and 0.1cm of c1] (t13) {};
		\node[typepos = {\texttt{region}}{-60}{0.15cm}, below right = 0.2cm and 0.8cm of c1] (t14) {};
		\node[constructorNW = {\texttt{addCurve}}, below = 0.4cm of t11] (c11) {};
		\node[typepos = {$\mathtt{\{C,\emptyset\}}$}{-120}{0.25cm}, below left = 0.4cm and 0.7cm of c11] (t111) {};
		\node[typeS = {$\mathtt{A}$}, below left = 0.5cm and 0.1cm of c11] (t112) {};
		\node[typeS = {$\mathtt{\emptyset}$}, below right = 0.5cm and 0.0cm of c11] (t113) {};
		\node[typepos = {$\mathtt{\emptyset}$}{-60}{0.15cm}, below right = 0.3cm and 0.6cm of c11] (t114) {};
		\node[constructorNE = {\texttt{addCurve}}, below right = 1.1cm and 0.9cm of t1] (c2) {};
		\node[typepos = {\texttt{diagram}}{-145}{0.32cm}, below left = 0.4cm and 0.5cm of c2] (t21) {};
		\node[typeS = {$\mathtt{C}$}, below left = 0.5cm and 0.0cm of c2] (t22) {};
		\node[typeS = {\texttt{region}}, below right = 0.5cm and 0.2cm of c2] (t23) {};
		\node[typepos = {$\mathtt{\emptyset}$}{-60}{0.15cm}, below right = 0.2cm and 1cm of c2] (t24) {};
		
		%\node[typeE = {$\{BD,D,\emptyset\}$}, below right = 0.1cm and 0.7cm of c] (t2) {};
		
		\path[->] 
		(c1) edge[out = 60, in = -160] (t1) 
		(t11) edge[out = 85, in = -170] node[index label]{1} (c1) 
		(t12) edge[out = 95, in = -120] node[index label]{2} (c1) 
		(t13) edge[out = 90, in = -60] node[index label]{3} (c1) 
		(t14) edge[out = 120, in = -10] node[index label]{4} (c1)
		(c2) edge[out = 120, in = -60] (t1) 
		(t21) edge[out = 85, in = -170] node[index label]{1} (c2) 
		(t22) edge[out = 95, in = -120] node[index label]{2} (c2) 
		(t23) edge[out = 90, in = -60] node[index label]{3} (c2) 
		(t24) edge[out = 120, in = -10] node[index label]{4} (c2)
		(c11) edge (t11) 
		(t111) edge[out = 75, in = -170] node[index label]{1} (c11) 
		(t112) edge[out = 90, in = -115] node[index label]{2} (c11) 
		(t113) edge[out = 90, in = -60] node[index label]{3} (c11) 
		(t114) edge[out = 120, in = -10] node[index label]{4} (c11);
		
		\coordinate[above left = 0.15cm and 2.6cm of t1] (a1);
		\coordinate[above right = 0.15cm and 1.8cm of t1] (a2);
		\coordinate[above right = 0.8cm and 0.3cm of t24] (a3);
		\coordinate[below right = 0.5cm and 0.3cm of t24] (a4);
		\coordinate[below right = 0.55cm and 2.5cm of t111] (a5);
		\coordinate[below left = 0.55cm and 0.6cm of t111] (a6);
		\coordinate[above left = 1.6cm and 0.6cm of t111] (a7);
		\draw[rounded corners = 5, very thick, darkred, draw opacity = 0.8] (a1) -- (a2) -- (a3) -- (a4) -- (a5) -- (a6) -- (a7) -- cycle;
		
		\coordinate[above left = 0.1cm and 0.5cm of t111] (b1);
		\coordinate[above right = 0.1cm and 0.2cm of t114] (b2);
		\coordinate[below right = 0.55cm and 0.2cm of t114] (b3);
		\coordinate[below left = 0.45cm and 0.5cm of t111] (b4);
		\draw[rounded corners = 4, very thick, darkblue, draw opacity = 0.8] (b1) -- (b2) -- (b3) -- (b4) -- cycle;
\end{tikzpicture}}

\newcommand{\instFifth}{%
	\begin{tikzpicture}[construction, anchor = center]
		\node[typeE = {\texttt{diagram}}] (t1) {};
		\node[constructorNW = {\texttt{addCurve}}, below left = 0.4cm and 1.2cm of t1] (c1) {};
		\node[typepos = {$\mathtt{\{C,A,CA,\emptyset\}}$}{175}{0.75cm}, below left = 0.4cm and 0.8cm of c1] (t11) {};
		\node[typeS = {$\mathtt{B}$}, below left = 0.4cm and 0.1cm of c1] (t12) {};
		\node[typeS = {$\mathtt{\emptyset}$}, below right = 0.5cm and 0.1cm of c1] (t13) {};
		\node[typepos = {\texttt{region}}{-60}{0.15cm}, below right = 0.2cm and 0.8cm of c1] (t14) {};
		\node[constructorNW = {\texttt{addCurve}}, below = 0.4cm of t11] (c11) {};
		\node[typepos = {$\mathtt{\{C,\emptyset\}}$}{-120}{0.25cm}, below left = 0.4cm and 0.7cm of c11] (t111) {};
		\node[typeS = {$\mathtt{A}$}, below left = 0.5cm and 0.1cm of c11] (t112) {};
		\node[typeS = {$\mathtt{\emptyset}$}, below right = 0.5cm and 0.0cm of c11] (t113) {};
		\node[typeS = {$\mathtt{\emptyset}$}, below right = 0.3cm and 0.6cm of c11] (t114) {};
		\node[constructorNE = {\texttt{addCurve}}, below right = 1.1cm and 0.9cm of t1] (c2) {};
		\node[typepos = {\texttt{diagram}}{-145}{0.32cm}, below left = 0.4cm and 0.5cm of c2] (t21) {};
		\node[typeS = {$\mathtt{C}$}, below left = 0.5cm and 0.0cm of c2] (t22) {};
		\node[typeS = {\texttt{region}}, below right = 0.5cm and 0.2cm of c2] (t23) {};
		\node[typepos = {$\mathtt{\emptyset}$}{-60}{0.15cm}, below right = 0.2cm and 1cm of c2] (t24) {};
		
		%\node[typeE = {$\{BD,D,\emptyset\}$}, below right = 0.1cm and 0.7cm of c] (t2) {};
		
		\path[->] 
		(c1) edge[out = 60, in = -160] (t1) 
		(t11) edge[out = 85, in = -170] node[index label]{1} (c1) 
		(t12) edge[out = 95, in = -120] node[index label]{2} (c1) 
		(t13) edge[out = 90, in = -60] node[index label]{3} (c1) 
		(t14) edge[out = 120, in = -10] node[index label]{4} (c1)
		(c2) edge[out = 120, in = -60] (t1) 
		(t21) edge[out = 85, in = -170] node[index label]{1} (c2) 
		(t22) edge[out = 95, in = -120] node[index label]{2} (c2) 
		(t23) edge[out = 90, in = -60] node[index label]{3} (c2) 
		(t24) edge[out = 120, in = -10] node[index label]{4} (c2)
		(c11) edge (t11) 
		(t111) edge[out = 75, in = -170] node[index label]{1} (c11) 
		(t112) edge[out = 90, in = -115] node[index label]{2} (c11) 
		(t113) edge[out = 90, in = -60] node[index label]{3} (c11) 
		(t114) edge[out = 120, in = -10] node[index label]{4} (c11);
		
		\coordinate[above left = 0.15cm and 2.6cm of t1] (a1);
		\coordinate[above right = 0.15cm and 1.8cm of t1] (a2);
		\coordinate[above right = 0.8cm and 0.3cm of t24] (a3);
		\coordinate[below right = 0.5cm and 0.3cm of t24] (a4);
		\coordinate[below right = 0.55cm and 2.5cm of t111] (a5);
		\coordinate[below left = 0.55cm and 0.6cm of t111] (a6);
		\coordinate[above left = 1.6cm and 0.6cm of t111] (a7);
		\draw[rounded corners = 5, very thick, darkred, draw opacity = 0.8] (a1) -- (a2) -- (a3) -- (a4) -- (a5) -- (a6) -- (a7) -- cycle;
		
		\coordinate[above left = 0.15cm and 1.51cm of t11] (b1);
		\coordinate[above right = 0.15cm and -0.05cm of t11] (b2);
		\coordinate[above right = -0.0cm and 0.1cm of t114] (b3);
		\coordinate[below right = 0.55cm and 0.1cm of t114] (b4);
		\coordinate[below left = 0.45cm and 0.5cm of t111] (b5);
		\coordinate[above left = 0.75cm and 0.5cm of t111] (b6);
		\draw[rounded corners = 4, very thick, darkblue, draw opacity = 0.8] (b1) -- (b2) -- (b3) -- (b4) -- (b5) -- (b6) -- cycle;
\end{tikzpicture}}

\newcommand{\instSixth}{%
	\begin{tikzpicture}[construction, anchor = center]
		\node[typeE = {$\mathtt{\{A,C,AC,BC,ABC,\emptyset\}}$}] (t1) {};
		\node[constructorNW = {\texttt{addCurve}}, below left = 0.4cm and 1.0cm of t1] (c1) {};
		\node[typepos = {$\mathtt{\{C,A,CA,\emptyset\}}$}{175}{0.75cm}, below left = 0.3cm and 0.7cm of c1] (t11) {};
		\node[typeS = {$\mathtt{B}$}, below left = 0.4cm and 0.1cm of c1] (t12) {};
		\node[typeS = {$\mathtt{\emptyset}$}, below right = 0.5cm and -0.1cm of c1] (t13) {};
		\node[typeS = {$\mathtt{\{A,\emptyset\}}$}, below right = 0.2cm and 0.6cm of c1] (t14) {};
		\node[constructorNW = {\texttt{addCurve}}, below = 0.4cm of t11] (c11) {};
		\node[typepos = {$\mathtt{\{C,\emptyset\}}$}{-120}{0.25cm}, below left = 0.4cm and 0.7cm of c11] (t111) {};
		\node[typeS = {$\mathtt{A}$}, below left = 0.5cm and 0.1cm of c11] (t112) {};
		\node[typeS = {$\mathtt{\emptyset}$}, below right = 0.5cm and 0.0cm of c11] (t113) {};
		\node[typeS = {$\mathtt{\emptyset}$}, below right = 0.3cm and 0.6cm of c11] (t114) {};
		\node[constructorNE = {\texttt{addCurve}}, below right = 1.2cm and 1.0cm of t1] (c2) {};
		\node[typepos = {\texttt{diagram}}{-145}{0.32cm}, below left = 0.4cm and 0.5cm of c2] (t21) {};
		\node[typeS = {$\mathtt{C}$}, below left = 0.5cm and 0.0cm of c2] (t22) {};
		\node[typeS = {\texttt{region}}, below right = 0.5cm and 0.2cm of c2] (t23) {};
		\node[typepos = {$\mathtt{\emptyset}$}{-60}{0.15cm}, below right = 0.2cm and 1cm of c2] (t24) {};
		
		%\node[typeE = {$\{BD,D,\emptyset\}$}, below right = 0.1cm and 0.7cm of c] (t2) {};
		
		\path[->] 
		(c1) edge[out = 60, in = -160] (t1) 
		(t11) edge[out = 85, in = -170] node[index label]{1} (c1) 
		(t12) edge[out = 95, in = -120] node[index label]{2} (c1) 
		(t13) edge[out = 90, in = -60] node[index label]{3} (c1) 
		(t14) edge[out = 120, in = -10] node[index label]{4} (c1)
		(c2) edge[out = 120, in = -60] (t1) 
		(t21) edge[out = 85, in = -170] node[index label]{1} (c2) 
		(t22) edge[out = 95, in = -120] node[index label]{2} (c2) 
		(t23) edge[out = 90, in = -60] node[index label]{3} (c2) 
		(t24) edge[out = 120, in = -10] node[index label]{4} (c2)
		(c11) edge (t11) 
		(t111) edge[out = 75, in = -170] node[index label]{1} (c11) 
		(t112) edge[out = 90, in = -115] node[index label]{2} (c11) 
		(t113) edge[out = 90, in = -60] node[index label]{3} (c11) 
		(t114) edge[out = 120, in = -10] node[index label]{4} (c11);
		
		\coordinate[above left = 0.2cm and 2.6cm of t1] (a1);
		\coordinate[above right = 0.2cm and 2.65cm of t1] (a2);
		\coordinate[above right = 0.8cm and 0.3cm of t24] (a3);
		\coordinate[below right = 0.58cm and 0.3cm of t24] (a4);
		\coordinate[below right = 0.55cm and 2.5cm of t111] (a5);
		\coordinate[below left = 0.55cm and 0.6cm of t111] (a6);
		\coordinate[above left = 1.6cm and 0.6cm of t111] (a7);
		\draw[rounded corners = 5, very thick, darkred, draw opacity = 0.8] (a1) -- (a2) -- (a3) -- (a4) -- (a5) -- (a6) -- (a7) -- cycle;
		
		\coordinate[above left = 0.1cm and 2.54cm of t1] (b1);
		\coordinate[above right = 0.1cm and 2.55cm of t1] (b2);
		\coordinate[below right = 0.05cm and 2.55cm of t1] (b3);
		\coordinate[above right = 0.4cm and 0.6cm of t114] (b4);
		\coordinate[below right = 0.55cm and 0.1cm of t114] (b5);
		\coordinate[below left = 0.45cm and 0.5cm of t111] (b6);
		\coordinate[above left = 1.54cm and 0.5cm of t111] (b7);
		\draw[rounded corners = 3, very thick, darkblue, draw opacity = 0.8] (b1) -- (b2) -- (b3) -- (b4) -- (b5) -- (b6) -- (b7) -- cycle;
\end{tikzpicture}}

\newcommand{\instSeventh}{%
	\begin{tikzpicture}[construction, anchor = center]
		\node[typeE = {$\mathtt{\{A,C,AC,BC,ABC,\emptyset\}}$}] (t1) {};
		\node[constructorNW = {\texttt{addCurve}}, below left = 0.4cm and 1.0cm of t1] (c1) {};
		\node[typepos = {$\mathtt{\{C,A,CA,\emptyset\}}$}{175}{0.75cm}, below left = 0.3cm and 0.7cm of c1] (t11) {};
		\node[typeS = {$\mathtt{B}$}, below left = 0.4cm and 0.1cm of c1] (t12) {};
		\node[typeS = {$\mathtt{\emptyset}$}, below right = 0.5cm and -0.1cm of c1] (t13) {};
		\node[typeS = {$\mathtt{\{A,\emptyset\}}$}, below right = 0.2cm and 0.6cm of c1] (t14) {};
		\node[constructorNW = {\texttt{addCurve}}, below = 0.4cm of t11] (c11) {};
		\node[typepos = {$\mathtt{\{C,\emptyset\}}$}{-120}{0.25cm}, below left = 0.4cm and 0.7cm of c11] (t111) {};
		\node[typeS = {$\mathtt{A}$}, below left = 0.5cm and 0.1cm of c11] (t112) {};
		\node[typeS = {$\mathtt{\emptyset}$}, below right = 0.5cm and 0.0cm of c11] (t113) {};
		\node[typeS = {$\mathtt{\emptyset}$}, below right = 0.3cm and 0.6cm of c11] (t114) {};
		\node[constructorNE = {\texttt{addCurve}}, below right = 1.2cm and 0.9cm of t1] (c2) {};
		\node[typepos = {$\mathtt{\{A,B,AB,\emptyset\}}$}{-143}{0.41cm}, below left = 0.4cm and 0.5cm of c2] (t21) {};
		\node[typeS = {$\mathtt{C}$}, below left = 0.5cm and -0.1cm of c2] (t22) {};
		\node[typepos = {$\mathtt{\{B,AB\}}$}{-68}{0.18cm}, below right = 0.4cm and 0.2cm of c2] (t23) {};
		\node[typeS = {$\mathtt{\emptyset}$}, below right = 0.2cm and 1cm of c2] (t24) {};
		
		%\node[typeE = {$\{BD,D,\emptyset\}$}, below right = 0.1cm and 0.7cm of c] (t2) {};
		
		\path[->] 
		(c1) edge[out = 60, in = -160] (t1) 
		(t11) edge[out = 85, in = -170] node[index label]{1} (c1) 
		(t12) edge[out = 95, in = -120] node[index label]{2} (c1) 
		(t13) edge[out = 90, in = -60] node[index label]{3} (c1) 
		(t14) edge[out = 120, in = -10] node[index label]{4} (c1)
		(c2) edge[out = 120, in = -60] (t1) 
		(t21) edge[out = 85, in = -170] node[index label]{1} (c2) 
		(t22) edge[out = 95, in = -120] node[index label]{2} (c2) 
		(t23) edge[out = 90, in = -60] node[index label]{3} (c2) 
		(t24) edge[out = 120, in = -10] node[index label]{4} (c2)
		(c11) edge (t11) 
		(t111) edge[out = 75, in = -170] node[index label]{1} (c11) 
		(t112) edge[out = 90, in = -115] node[index label]{2} (c11) 
		(t113) edge[out = 90, in = -60] node[index label]{3} (c11) 
		(t114) edge[out = 120, in = -10] node[index label]{4} (c11);
		
		\coordinate[above left = 0.2cm and 2.6cm of t1] (a1);
		\coordinate[above right = 0.2cm and 2.65cm of t1] (a2);
		\coordinate[above right = 0.8cm and 0.3cm of t24] (a3);
		\coordinate[below right = 0.58cm and 0.3cm of t24] (a4);
		\coordinate[below right = 0.55cm and 2.5cm of t111] (a5);
		\coordinate[below left = 0.55cm and 0.6cm of t111] (a6);
		\coordinate[above left = 1.6cm and 0.6cm of t111] (a7);
		\draw[rounded corners = 5, very thick, darkred, draw opacity = 0.8] (a1) -- (a2) -- (a3) -- (a4) -- (a5) -- (a6) -- (a7) -- cycle;
		
		\coordinate[above left = 0.1cm and 2.54cm of t1] (b1);
		\coordinate[above right = 0.1cm and 2.55cm of t1] (b2);
		\coordinate[above right = 0.7cm and 0.2cm of t24] (b3);
		\coordinate[below right = 0.5cm and 0.2cm of t24] (b4);
		\coordinate[below right = 0.45cm and 2.42cm of t111] (b5);
		\coordinate[below left = 0.45cm and 0.5cm of t111] (b6);
		\coordinate[above left = 1.54cm and 0.5cm of t111] (b7);
		\draw[rounded corners = 4, very thick, darkblue, draw opacity = 0.8] (b1) -- (b2) -- (b3) -- (b4) -- (b5) -- (b6) -- (b7) -- cycle;
\end{tikzpicture}}

\newcommand{\instDiagrams}{%
	\begin{tikzpicture}[construction, anchor = center]
		\node[typeE = {$\mathtt{\{A,C,AC,BC,ABC,\emptyset\}}$}] (t1) {};
		\node[constructorNW = {\texttt{addCurve}}, below left = 0.4cm and 1.0cm of t1] (c1) {};
		\node[typepos = {$\mathtt{\{C,A,CA,\emptyset\}}$}{175}{0.75cm}, below left = 0.3cm and 0.7cm of c1] (t11) {};
		\node[typeS = {$\mathtt{B}$}, below left = 0.4cm and 0.1cm of c1] (t12) {};
		\node[typeS = {$\mathtt{\emptyset}$}, below right = 0.5cm and -0.1cm of c1] (t13) {};
		\node[typeS = {$\mathtt{\{A,\emptyset\}}$}, below right = 0.2cm and 0.6cm of c1] (t14) {};
		\node[constructorNW = {\texttt{addCurve}}, below = 0.4cm of t11] (c11) {};
		\node[typepos = {$\mathtt{\{C,\emptyset\}}$}{-120}{0.25cm}, below left = 0.4cm and 0.7cm of c11] (t111) {};
		\node[typeS = {$\mathtt{A}$}, below left = 0.5cm and 0.1cm of c11] (t112) {};
		\node[typeS = {$\mathtt{\emptyset}$}, below right = 0.5cm and 0.0cm of c11] (t113) {};
		\node[typeS = {$\mathtt{\emptyset}$}, below right = 0.3cm and 0.6cm of c11] (t114) {};
		\node[constructorNE = {\texttt{addCurve}}, below right = 1.2cm and 0.9cm of t1] (c2) {};
		\node[typepos = {$\mathtt{\{A,B,AB,\emptyset\}}$}{-143}{0.41cm}, below left = 0.4cm and 0.5cm of c2] (t21) {};
		\node[typeS = {$\mathtt{C}$}, below left = 0.5cm and -0.1cm of c2] (t22) {};
		\node[typepos = {$\mathtt{\{B,AB\}}$}{-68}{0.18cm}, below right = 0.4cm and 0.2cm of c2] (t23) {};
		\node[typeS = {$\mathtt{\emptyset}$}, below right = 0.2cm and 1cm of c2] (t24) {};
		
		%\node[typeE = {$\{BD,D,\emptyset\}$}, below right = 0.1cm and 0.7cm of c] (t2) {};
		
		\path[->] 
		(c1) edge[out = 60, in = -160] (t1) 
		(t11) edge[out = 85, in = -170] node[index label]{1} (c1) 
		(t12) edge[out = 95, in = -120] node[index label]{2} (c1) 
		(t13) edge[out = 90, in = -60] node[index label]{3} (c1) 
		(t14) edge[out = 120, in = -10] node[index label]{4} (c1)
		(c2) edge[out = 120, in = -60] (t1) 
		(t21) edge[out = 85, in = -170] node[index label]{1} (c2) 
		(t22) edge[out = 95, in = -120] node[index label]{2} (c2) 
		(t23) edge[out = 90, in = -60] node[index label]{3} (c2) 
		(t24) edge[out = 120, in = -10] node[index label]{4} (c2)
		(c11) edge (t11) 
		(t111) edge[out = 75, in = -170] node[index label]{1} (c11) 
		(t112) edge[out = 90, in = -115] node[index label]{2} (c11) 
		(t113) edge[out = 90, in = -60] node[index label]{3} (c11) 
		(t114) edge[out = 120, in = -10] node[index label]{4} (c11);
		
		\coordinate[above left = 0.2cm and 2.6cm of t1] (a1);
		\coordinate[above right = 0.2cm and 2.65cm of t1] (a2);
		\coordinate[above right = 0.8cm and 0.3cm of t24] (a3);
		\coordinate[below right = 0.58cm and 0.3cm of t24] (a4);
		\coordinate[below right = 0.55cm and 2.5cm of t111] (a5);
		\coordinate[below left = 0.55cm and 0.6cm of t111] (a6);
		\coordinate[above left = 1.6cm and 0.6cm of t111] (a7);
		\draw[rounded corners = 5, very thick, darkred, draw opacity = 0.8] (a1) -- (a2) -- (a3) -- (a4) -- (a5) -- (a6) -- (a7) -- cycle;
		
		\coordinate[above left = 0.1cm and 2.54cm of t1] (b1);
		\coordinate[above right = 0.1cm and 2.55cm of t1] (b2);
		\coordinate[above right = 0.7cm and 0.2cm of t24] (b3);
		\coordinate[below right = 0.5cm and 0.2cm of t24] (b4);
		\coordinate[below right = 0.45cm and 2.42cm of t111] (b5);
		\coordinate[below left = 0.45cm and 0.5cm of t111] (b6);
		\coordinate[above left = 1.54cm and 0.5cm of t111] (b7);
		\draw[rounded corners = 4, very thick, darkblue, draw opacity = 0.8] (b1) -- (b2) -- (b3) -- (b4) -- (b5) -- (b6) -- (b7) -- cycle;
\end{tikzpicture}}

\begin{wrapfigure}[3]{0}{0.18\linewidth}\vspace*{-2ex}
	\hfill
\vspace*{-2ex}
\end{wrapfigure}
The schemas above are sufficient for solving the instantiation problem. Figure~\ref{instantiation-steps} shows one way of doing it. The resulting construction graph encodes two ways of building the diagram shown here (right).

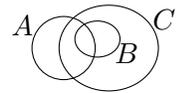
\begin{figure}[hbt]
\adjustbox{scale = 0.9,center}{%
	%
}
\caption{Solving an instantiation problem with forward-applications.}\label{instantiation-steps}
\end{figure}

\medskip
The following theorem is simply a corollary of Theorem~\ref{cor:instantiation-of-inst-seq}.

\begin{theorem}
	Let $\mathbb{S}$ be a set of schemas and let $\sequent{\ }{\goa}$ be a sequent for construction space $\mathcal{C}$ that is valid over $\mathbb{S}$. Then $\goa$ can be instantiated in $\mathcal{C}$.
\end{theorem}

The significance of this theorem is that we can use schemas iteratively to derive new knowledge about which pattern graphs can be instantiated in $\mathcal{C}$. We have seen that schemas can be \textit{applied} to derive such new knowledge. %Now we will turn our attention to inferreing new information about meta-properties of a constructions space, $\mathcal{C}$, that are encoded in a construction system, $(\mathcal{C},\goalSpace)$.

\subsection{Inference: Determining Whether Meta-Properties Hold}\label{sec:inferenceSchema}

%Another important application of schemas is for identifying meta-properties of construction spaces. That is, given a
%construction system, $(\mathcal{C},\goalSpace)$, we can use schemas to derive whether a specific property of
%tokens in $\mathcal{C}$ holds, provided that this property is expressed in $\goalSpace$. Thus, in this context, schemas are like inference rules for the properties and relations expressible in $\goalSpace$.
%
%\paragraph{Problem} Given a pattern graph, $\contextGraph$, for construction space $\mathcal{C}$ and a pair of pattern graphs, $\assump$ and $\goa$, for meta-space $\goalSpace$, if $\contextGraph$ can be instantiated in such a way that the property encoded by $\assump$ holds can it can be instantiated so that the property encoded by $\assump\cup \goa$ holds?
%
%\paragraph{Solution} Determine whether $\sequent{\contextGraph,\assump}{\goa}$ is a schema for the construction system $(\mathcal{C},\goalSpace)$.

%\begin{definition}\label{defn:inferenceSchema}
%Let $(\mathcal{C},\goalSpace)$ be a construction system. An \textbf{inference sequent} for $(\mathcal{C},\goalSpace)$ is a sequent, $\sequent{\contextGraph,\assump}{\goa}$, for $(\mathcal{C},\goalSpace)$. An \textbf{inference schema} for $(\mathcal{C},\goalSpace)$ is a schema, $\sequent{\contextGraph,\assump}{\goa}$, for $(\mathcal{C},\goalSpace)$. We call $\contextGraph$ the \textbf{context}, $\assump$ the \textbf{assumption}, and $\goa$ the \textbf{goal} of $\sequent{\contextGraph,\assump}{\goa}$.
%\end{definition}

Another important application of schemas is for identifying meta-properties across construction spaces in a multi-space system, $\multiSpace$. That is, we can use schemas to derive whether a specific property of
tokens in $\multiSpace$ holds, provided that this property is expressed in the meta-space, $\goalSpace$. Thus, in this context, schemas are like inference rules for the relations expressible in $\goalSpace$.

\paragraph{Problem} Take pattern graphs $\contextGraph_1,...,\contextGraph_n$ for construction spaces $\mathcal{C}_1,\ldots,\mathcal{C}_n$, respectively. Assume they can be simultaneously instantiated to satisfy some constraint $P$. Can they also be instantiated to satisfy some constraint $Q$?

%\paragraph{Problem} Given a pattern graph, $(\contextGraph_1,...,\contextGraph_n,\assump)$, for multi-space system $\multiSpace$, and a pattern graph, $\goa$, for meta-space $\goalSpace$, if $(\contextGraph_1,...,\contextGraph_n,\assump)$ can be instantiated in such a way that the property encoded by $\assump$ holds can it can be instantiated so that the property encoded by $\assump\cup \goa$ holds?

\paragraph{Solution} Take a meta-space, $\goalSpace$, for which $(\mathcal{C}_1,\ldots,\mathcal{C}_n,\goalSpace)$ is a multi-space system, where the constraints $P$ and $Q$ can be encoded as $\assump$ and $\goa$, respectively. Determine whether $\sequent{\contextGraph_1,...,\contextGraph_n,\assump}{\goa}$ is a schema for $(\mathcal{C}_1,\ldots,\mathcal{C}_n,\goalSpace)$.

\vskip1ex
Schemas for $(\mathcal{C}_1,\ldots,\mathcal{C}_n,\goalSpace)$ can model inference rules if the antecedent and consequent, which are graphs in $\goalSpace$, model relations between the tokens of spaces $\mathcal{C}_1,\ldots,\mathcal{C}_n$. For any schema, $\sequent{\contextGraph_1,\ldots,\contextGraph_n,\assump}{\goa}$, where $\sequent{\contextGraph_1,\ldots,\contextGraph_n, \assump}{\goa}$ is monotonic for \textit{every} graph $(\contextGraph_1',\ldots,\contextGraph_n',\assump')$, we have an analogous concept to logical monotonicity, which is often expected in logical calculi. Given a schema, $\sequent{\contextGraph_1,\ldots,\contextGraph_n, \assump}{\goa}$, if we can instantiate
$(\contextGraph_1,\ldots,\linebreak[2]\contextGraph_n,\linebreak[2] \assump)$ we must, by definition, be able to instantiate $(\contextGraph_1,\ldots,\contextGraph_n, \assump \cup \goa)$. Thus, if $\assump$ encodes a relation, $P$, between the tokens of the graphs $\contextGraph_1,\ldots,\contextGraph_n$, and $\goa$ encodes another relation, $Q$, between the tokens of the graphs $\contextGraph_1,\ldots,\contextGraph_n$, then the schema $\sequent{\contextGraph_1,\ldots,\contextGraph_n,\assump}{\goa}$ encodes the fact that for any instantiation of $\contextGraph_1,\ldots,\contextGraph_n$ that satisfies $P$, we can find an instantiation that also satisfies $Q$. Importantly, by definition, schemas do not need to be monotonic. The implications of this are yet to be studied, but in principle this allows us to encode non-monotonic logics for inference in multi-spaces. 

Now we can look back at the specific inference problem described at the start of Section~\ref{sec:sequentsAndSchemas}. The following examples describe schemas that can be used to solve this problem.

\begin{example}[Observe and depict are dual]\label{ex:depict-observe} If a diagram depicts a formula, then the same formula can be observed from the diagram. Thus $\sequent{\contextGraph_1,\contextGraph_2,\assump}{\goa}$, depicted below, is a schema.
	
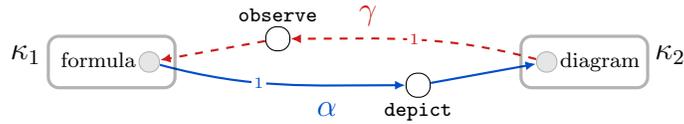
\begin{center}
		\begin{tikzpicture}[construction]
			\node[typeW = {formula}] (t) {};
			
			\node[typeE = {diagram}, right = 5cm of t] (t') {};
			
			\node[constructorN = {\texttt{observe}}, above right = 0.0cm and 1.4cm of t] (cc) {};
			\path[->,dashed,thick,darkred] (t') edge[out = 165, in = -0, pos = 0.5] node[index label, pos = 0.5]{1} (cc) (cc) edge (t);
			
			\node[constructorS = {\texttt{depict}}, below left = 0.0cm and 1.4cm of t'] (cc') {};
			\path[->,thick,darkblue] (t) edge[out=-15,in=180, pos = 0.5] node[index label]{1} (cc') (cc') edge (t');
			\coordinate[above left = 0.2cm and 1.2cm of t] (a1);
			\coordinate[above right = 0.2cm and 0.2cm of t] (a2);
			\coordinate[below right = 0.2cm and 0.2cm of t] (a3);
			\coordinate[below left = 0.2cm and 1.2cm of t] (a4);
			\draw[rounded corners = 4, draw opacity = 0.3, very thick] (a1) -- (a2) -- (a3) -- (a4) -- node[xshift = -0.3cm, yshift = 0.1cm]{\large$\contextGraph_1$} cycle;
			\coordinate[above left = 0.2cm and 0.2cm of t'] (a1');
			\coordinate[above right = 0.2cm and 1.2cm of t'] (a2');
			\coordinate[below right = 0.2cm and 1.2cm of t'] (a3');
			\coordinate[below left = 0.2cm and 0.2cm of t'] (a4');
			\draw[rounded corners = 4, draw opacity = 0.3, very thick] (a1') -- (a2') -- node[xshift = 0.3cm, yshift = 0.1cm]{\large$\contextGraph_2$} (a3') -- (a4') -- cycle;
			\node[below left = -0.05cm and 0.8cm of cc'] () {\large\textcolor{darkblue}{$\assump$}};
			\node[above right = -0.1cm and 0.8cm of cc] () {\large\textcolor{darkred}{$\goa$}};
		\end{tikzpicture}
	\end{center}
\end{example}

\begin{example}[Depicting disjoint sets]\label{ex:disjoint-depict} To depict an expression of the form $x \cap \eta = \emptyset$, where $x$ is a set expression and $\eta$ is a variable, it suffices to take a diagram, $d$, and add $\eta$ in such a way that the region corresponding to $x$ in $d$ is part of the \textit{out} regions of the curve labelled by $\eta$ (based on~\citeA{stapleton2010inductively}). Thus, the following figure describes a family of schemas, one for each $\eta$, subtype of $\texttt{var}$ in \textsc{Set Algebra}, and subtype of $\texttt{label}$ in \textsc{Euler Diagrams}.%\footnote{Hereafter we will draw sequents and schemas of the form $\sequent{\contextGraph_1,\contextGraph_2,\assump}{\goa}$ with $\contextGraph_1$ on the left, $\contextGraph_2$ on the right, $\assump$ with thick blue arrows, and $\goa$ with dashed red arrows.}:
	\begin{center}

	\end{center}
\end{example}
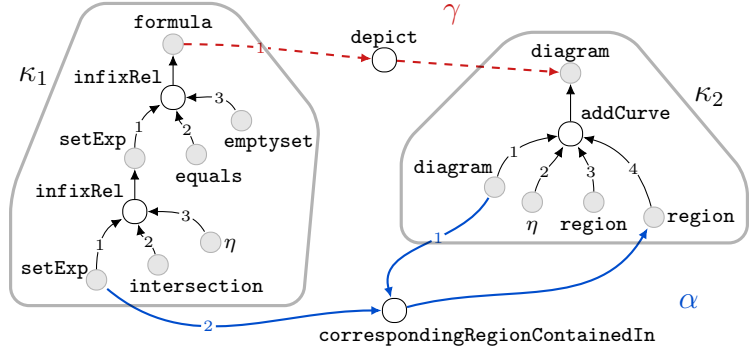
\begin{example}[Set expressions and their corresponding regions]\label{ex:corr-var-region}
		This schema tells you how to propagate from the input to the output of the \texttt{correspondingRegionContainedIn} constructor. For every $\mathtt{x_1,\ldots,x_k,y_1,\ldots,y_m,z_1,\ldots,z_n}$ and $\mathtt{\eta}$ the following is a schema:
		\begin{center}
			%
		\end{center}
\end{example}
It is easy to see that a sequence of three applications of the schemas from Examples~\ref{ex:depict-observe}, \ref{ex:disjoint-depict} and~\ref{ex:corr-var-region} is sufficient to solve the inference problem (redrawn below). This simply means that the statement $A \cap C = \emptyset$ can be observed from the diagram $\AsubBdisjCinline$:
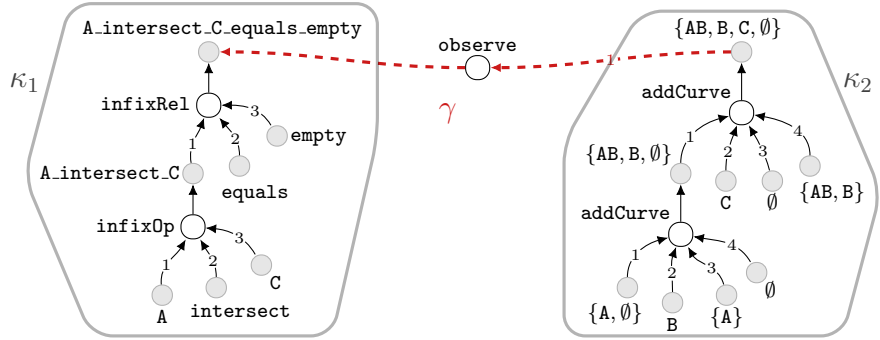
\begin{center}
	%
\end{center}

\subsection{Transfer Schemas: Abducting Structure}\label{sec:transfer}
As we have seen, schemas are used to determine whether a relation holds between some tokens in a multi-space, \textit{given} a context structure graph where such tokens live. For example, we can use a set of schemas to determine whether a sequent, $\sequent{\contextGraph_1,\contextGraph_2,\assump}{\goa}$, is itself a schema. In this case $\contextGraph_1$ and $\contextGraph_2$ are the given context for the tokens. Now consider the case where we know $\contextGraph_1$ and we want to \textit{find} some yet-unknown $\contextGraph_2$ such that $\sequent{\contextGraph_1,\contextGraph_2,\assump}{\goa}$ holds. For instance, we may be given a set-theoretic formula with structure $\contextGraph_1$, and we may want to find an Euler diagram with structure $\contextGraph_2$ that depicts it, as drawn below\footnote{For illustration purposes we are not showing the structure of $B \cap C = \emptyset$ here.}:
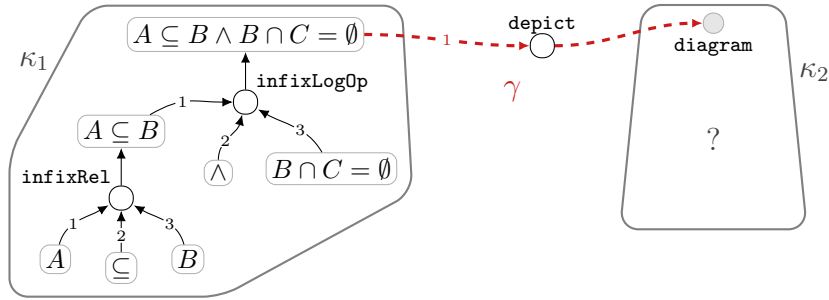
\begin{center}

\end{center}
Thus, the problem in this case is stated roughly as follows: we are given a source graph, $\contextGraph_1$, an assumption graph, $\assump$ (empty in this case), a goal graph, $\goa$, and an \textit{initial} target graph, $\contextGraph_2'$ (with a single vertex) which we want to \textit{reify}\footnote{Recall that reifying means adding context and specialising types.} into some graph $\contextGraph_2$ in such a way that $\sequent{\contextGraph_1,\contextGraph_2,\assump}{\goa}$ is provably a schema. 
As we will see, the \textit{application} of \textit{transfer schemas} will differ from the application of schemas in that transfer schema applications reify a target graph. 

%\gnote{need motivating example here}

%At this point, it is prudent to explain, in a little more detail, the problem that this section is solving. Suppose we have pattern graphs, $\contextGraph_1$ and $\contextGraph_2$, for construction spaces $\mathcal{C}_1$ and $\mathcal{C}_2$ and that we desire to know whether a particular relation, $\goa$, encoded by a pattern graph for meta-space, $\goalSpace$, holds for instantiations of $\contextGraph_1$ and $\contextGraph_2$. We can use information about assumed relations, encoded by another pattern graph, $\assump$, for $\goalSpace$ to deduce whether $\goa$ holds. Now, when we are seeking tokens that instantiation $\contextGraph_2$, it may be that we can reify $\contextGraph_2$ -- adding more context, whilst making the types more specialised -- to give a new pattern graph, $\contextGraph_2'$, that allows us to infer our relation encoded by $\goa$ holds. Without such alterations to $\contextGraph_2$, we may not have sufficient knowledge to deduce whether $\goa$ holds. We can state our problem more generally:

%\gnote{update the problem statement for the general case}

%\dnote{and more about the problem statement: it's technically incorrect if you consider that $gamma_n$ shares tokens with the rest, which is why the isomorphism is label-preserving up to the tokens shared with $gamma_n$.}
The next problem statement formalises this in a more general way, where an arbitrary number of the given context graphs is allowed reification.
\pagebreak[4]

\paragraph{Problem} Take pattern graphs $\contextGraph_1,...,\contextGraph_n$ for construction spaces $\mathcal{C}_1,\ldots,\mathcal{C}_n$, respectively. Take a set $\sigma \subseteq \{1,\ldots,n\}$ that indicates the target. Is it possible to reify the target graphs, $\contextGraph_i$ with $i\in \sigma$, in such a way that any instantiation of them satisfying $P$ ensures an instantiation of them satisfying $Q$?

%\paragraph{Problem} Given a pattern graph $(\contextGraph_1,\ldots,\contextGraph_n, \assump)$, for multi-space system, $\multiSpace$, a pattern graph $\goa$, for meta-space $\goalSpace$, and a set $\sigma \subseteq \{1,\ldots,n\}$, is it possible to reify the graphs $\contextGraph_i$, for $i \in \sigma$, in such a way that if the property encoded by $\assump$ holds for the reified graphs then the property encoded by $\assump \cup \goa$ also holds for them?

%, via a reification function $\reify \colon (\contextGraph_1,\ldots,\contextGraph_n) \to (\contextGraph_1',\ldots,\contextGraph_n')$ where $\reify[\contextGraph_j] = \contextGraph_j$ for every $j \notin I$ in such a way that if the property encoded by $\assump$ holds for $(\reify[\contextGraph_1],\ldots,\reify[\contextGraph_n])$ then also the property encoded by $\assump \cup \goa$ holds?

\paragraph{Solution} Encode $P$ and $Q$ as pattern graphs $\assump$ and $\goa$ respectively, for some multi-space system $(\mathcal{C}_1,\ldots,\mathcal{C}_n,\goalSpace)$. Determine whether there exists a sequent, $\sequent{\contextGraph_1',...,\contextGraph_n',\assump'}{\goa'}$, and a reification function $\reify \colon (\contextGraph_1,\ldots,\contextGraph_n,\assump \cup \goa) \to (\contextGraph_1',\ldots,\contextGraph_n',\assump' \cup \goa')$ such that
\begin{enumerate}
	\item the restrictions $\reify|_{\assump} \colon \assump \to \assump'$, $\reify|_{\goa} \colon \goa \to \goa'$ and every $\reify|_{\contextGraph_j} \colon \contextGraph_j \to \contextGraph_j'$ for $j \notin \sigma$ are label-preserving isomorphisms (up to the tokens shared with $\bigcup_{i \in \sigma}\contextGraph_i$), and
	\item $\sequent{\contextGraph_1',...,\contextGraph_n',\assump'}{\goa'}$ is a schema.
\end{enumerate}

%a schema, $\sequent{\reify[\contextGraph_1],...,\reify[\contextGraph_n],\assump}{\goa}$ and a refinement, $\refine\colon (\contextGraph_1,...,\contextGraph_n,\assump,\goa)\to (\contextGraph_1',...,\contextGraph_n',\assump',\goa')$, such that the restriction of $\refine$ to $\contextGraph_1,...,\contextGraph_{n-1},\assump,\goa$ is an isomporphism with co-domain $(\contextGraph_1',...,\contextGraph_{n-1}',\assump',\goa')$.

As we saw, schemas can capture invariants across multi-space systems. Now we define $\sigma$-transfer schema as a schema paired with a set, $\sigma$. This set simply provides information about which dimensions will be targetted for transferring structure to them.

\begin{definition}
	A \textbf{$\sigma$-transfer schema} for $\multiSpace$ is a pair consisting of a schema, $\sequent{\contextGraph_1,\ldots,\contextGraph_n,\assump}{\goa}$, for $\multiSpace$, and a set, $\sigma \subseteq \{1,\ldots,n\}$.
\end{definition}

\textit{Applying} a $\sigma$-transfer schema, $\sequent{\contextGraph_1,\ldots,\contextGraph_n,\assump}{\goa}$, to a sequent, $\sequent{\contextGraph_1',\ldots,\contextGraph_n',\assump'}{\goa'}$, will be similar to applying it as a schema, except that we are allowed to reify the graphs indexed by $\sigma$. 

Let us begin by illustrating with some examples of schemas where $n = 2$, that is, schemas of the form $\sequent{\contextGraph_1,\contextGraph_2,\assump}{\goa}$. Then, we will see that choosing a $\sigma$ means selecting the target to produce structure while applying them.

\begin{example}[Depicting subsets]\label{ex:subset-depict} To depict an expression of the form $x \subseteq \eta$, where $\eta$ is a variable, it suffices to take a diagram, $d$, and make sure that  $\eta$  is drawn so that the region corresponding to $x$ is part of the \textit{in} regions of the curve labelled by $\eta$ (based on~\citeA{stapleton2010inductively}). This means that \textbf{for every $\eta$} the following is a schema\footnote{Hereafter we will draw sequents and schemas of the form $\sequent{\contextGraph_1,\contextGraph_2,\assump}{\goa}$ with $\contextGraph_1$ on the left, $\contextGraph_2$ on the right, $\assump$ with thick blue arrows, and $\goa$ with dashed red arrows.}:
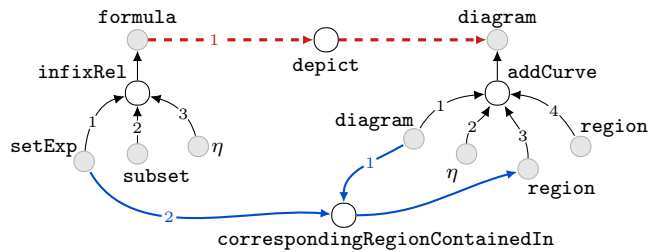
		\begin{center}

		\end{center}
	Using this schema as a $\{2\}$-transfer schema would produce a diagram from a given formula\footnote{Moreover, using this schema as a $\{1\}$-transfer schema can be used for producing a formula for a given diagram. Using it as $\{1,2\}$-transfer schema can be used to produce both the given assumptions and goals. Using it as a $\emptyset$-transfer schema is no different to using it as a schema.}.
\end{example}

\pagebreak[2]

\begin{example}[Matrices and linear maps] As we hinted in Example~\ref{ex:matrices-semantics}, the homomorphism between matrices and linear
	maps can be represented as a schema.
	\begin{center}

	\end{center}
\end{example}

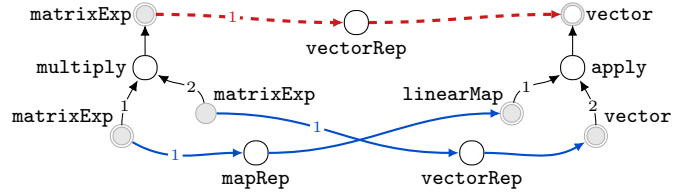
\begin{example}[Cardinality]\label{ex:cardinality} The relation that links every finite set to its cardinality is respected along some
	\textit{corresponding} operations. We can encode these properties as schemas, using constructors \texttt{card} and \texttt{disj} in
	an meta-space between some set-theory space and some number-theory space. For example, the diagram below depicts a schema
	which states that the union of disjoint sets is analogous to addition. The blue graph encodes the antecedent: that $A$ and $B$ are disjoint
	and that $|A|=n$ and $|B|=m$, and the red graph encodes the consequent: that $|A \cup B| = n + m$.
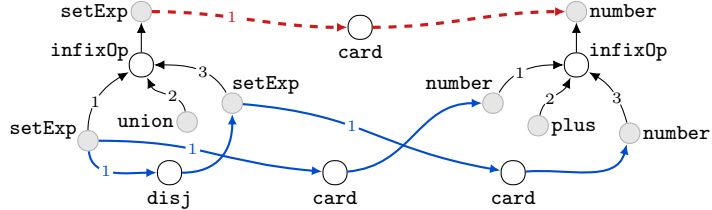
	\begin{center}
		\adjustbox{valign = c}{%
}
	\end{center}
\end{example}

These examples capture homomorphisms which, as we have seen in Section~\ref{sec:inferenceSchema}, can be used for proving whether some specified relation holds. This is ideal to use them as $\sigma$-transfer schemas to produce a desirable structure in the target graphs specified by $\sigma$. As we will see, \textit{structure transfer} is a calculus that uses transfer schemas to \textit{abduct}\footnote{In logic, abduction is an operation by which, given a set of statements (e.g., a set of observations), a sufficient set of axioms (or a model) is produced that entails the set of statements.} the structure of some target graphs in order to ensure that the specified relations hold. Our abductive approach is based on reification of the selected target graphs. Given some sequent, $\sequent{\contextGraph_1,...,\contextGraph_n,\assump}{\goa}$, and a set $\sigma\subseteq \{1,...,n\}$, we will be allowed to reify those graphs indexed by $\sigma$ to ensure that the sequent is valid. The operation that does precisely this is called a \textit{$\sigma$-reification}.

\begin{definition}
	Let $\sequent{\contextGraph_1,...,\contextGraph_n,\assump}{\goa}$ and $\sequent{\contextGraph_1',...,\contextGraph_n',\assump'}{\goa'}$ be sequents for multi-space system $\multiSpace$ and let $\sigma\subseteq \{1,...,n\}$. Then $\sequent{\contextGraph_1',...,\contextGraph_n',\assump'}{\goa'}$ is a \textbf{$\sigma$-reification} of $\sequent{\contextGraph_1,...,\contextGraph_n,\assump}{\goa}$ provided there exists a reification function $\reify\colon (\contextGraph_1,...,\contextGraph_n,\assump \cup \goa) \to (\contextGraph_1',...,\contextGraph_n',\assump' \cup \goa')$ such that $\reify|_{\assump} \colon \assump \to \assump'$, $\reify|_{\goa} \colon \goa \to \goa'$ and every $\reify|_{\contextGraph_j} \colon \contextGraph_j \to \contextGraph_j'$ where $j \notin \sigma$, are a label-preserving isomorphisms up to the tokens of $\bigcup_{i \in \sigma} \contextGraph_i$.
	In this context we call $\{\contextGraph_j : j \notin \sigma\}$ the \textbf{source} and $\{\contextGraph_i : i \in \sigma\}$ the \textbf{target} of the $\sigma$-reification.
\end{definition}

%\gnote{whilst later defns use $f_\beta$ I am not sure why it is needed. defn 4.20 -- of valid mod $f_beta$ -- is just creating a load of
%	isomorphic states. Couldnt the notion of valid be written without any reference to the function at all? By that I mean remove 'given
%	$f_{\beta}$' from the defn above, then present valid without reference to $f_{\beta}$, just let $\beta_0$ be a reification of $\beta$; it
%	feels like the choice of function should not be the driver, but the choice of reification (for which there can be many suitable reification
%	functions). Looking even furher  ahead, the name of algotithm 5 is suggestive of the reification being the focus, not the function that
%	identifies it as such. See my next attempt at the defn of permmited altered state with that in mind. Indeed, I think defn 4.2.2 -- applying
%	transfer schema -- only makes sense with the defn below. Without knowing the consequences, given these three versions of permitted state
%	alteration defns, my preferred is the third one, i.e. the one below. On a notational point, it would be easier to read if the subscript 0 were
%	replaced by dashes for a consistent style i.e. the first state in the defn does not use subscripts.}

A $\sigma$-reification can only modify the structure and types that appear in the target specified by $\sigma$. Everything else must be mapped through a label-preserving isomorphism, which means neither the structure nor the types can be changed. A $\sigma$-reification is not a deductive operation as the context is specialised and enlarged\footnote{This is analogous to adding constraints and assumptions to a conjecture.}. However, it is constrained to ensure that the speculative structure introduced in the target graphs does not go beyond them into the source graphs or the antecedent.

We can define the notion of application of $\sigma$-transfer schemas, which allows us to reify the target graphs. 
\begin{definition}
	A  \textbf{$\sigma$-transfer schema application} of $\sequent{\contextGraph_1,...,\contextGraph_n,\assump}{\goa}$ to
	$\sequent{\contextGraph_1',...,\contextGraph_n',\assump'}{\goa'}$ is any application of $\sequent{\contextGraph_1,...,\contextGraph_n,\assump}{\goa}$ to a $\sigma$-reification of $\sequent{\contextGraph_1',...,\contextGraph_n',\assump'}{\goa'}$.
\end{definition}
%We accompany this definition with a practical approach (Algorithm~\ref{alg:transfer-schema}) for producing the $\sigma$-reification in such a way that, given a transfer schema, the reification is executed so that just enough structure is added to ensure that the schema is applicable. 

\begin{example}\label{ex:transfer-application} We saw in Example~\ref{ex:schema-application} how applying a schema to a sequent yields a sequent, here renamed $\sequent{\contextGraph_1',\contextGraph_2',\assump'}{\goa'}$, with empty $\assump'$, as visualised below. 
	\begin{center}

	\end{center}
	Note that this sequent is \textit{not} a schema, given that not every instantiation of $\contextGraph_2'$ will depict $A \subseteq B$ and $B \cap C = \emptyset$. However, what we want to know is whether there \textit{exists} an instantiation of $\contextGraph_2'$ that satisfies the conditions encoded by $\goa'$. Transfer schema applications can help us here.
	The schema from Example~\ref{ex:subset-depict} tells us that an Euler diagram depicts a subset expression when it is built in a particular way, so we can use this information to build the diagram\footnote{Note that, given also the graph that constructs $B \cap C = \emptyset$ (which we omit for simplicity), we have a choice whether to use the schema for $A \subseteq B$ or to use the schema from Example~\ref{ex:disjoint-depict}. The choice here is only for demonstration purposes.}. Note that, with the given $\contextGraph_2'$ we would not be able to apply the schema. However, we can apply the schema as a $\{2\}$-transfer schema, allowing us to reify $\contextGraph_2'$. This $\{2\}$-reification \textit{unblocks} the sequent so that we can apply the schema. The figure below shows the
	result of reifying $\contextGraph_2'$ into some graph $\contextGraph_2''$ and applying the schema -- that is, the result of applying it (backwards) as a $\{2\}$-transfer schema. Intuitively, the application of the $\{2\}$-transfer schema gives us sufficient structure to match the context, and reduces the goal of depicting $A \subseteq B$ to the goal of ensuring that the a region corresponding to $A$ is contained in the \textit{in} region of $B$.
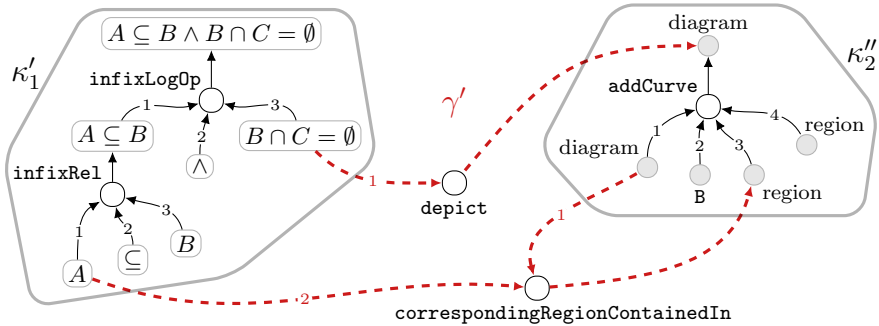
	\begin{center}

	\end{center}
\end{example}

Below we define \textit{validity modulo $\sigma$} -- that is, when a sequent is one $\sigma$-reification away from becoming valid. Crucially, we can prove that a sequence of $\sigma$-transfer schema applications can be used to derive whether a sequent is valid modulo $\sigma$.

\begin{definition}
	A sequent, $\sequent{\contextGraph_1,...,\contextGraph_n,\assump}{\goa}$, is \textbf{valid modulo $\sigma$} provided there exists a valid sequent, $\sequent{\contextGraph_1',...,\contextGraph_n',\assump'}{\goa'}$, that is a $\sigma$-reification of $\sequent{\contextGraph_1,...,\contextGraph_n,\assump}{\goa}$.
\end{definition}

The next theorem establishes that if we $\sigma$-apply transfer schemas sequentially to a sequent and reach a valid sequent this means that the original sequent is valid modulo $\sigma$.
\begin{theorem}\label{thm:transfer-sequence-valid}
	Let $\sequent{\contextGraph_1,...,\contextGraph_n,\assump}{\goa}$ be a sequent for $\multiSpace$, let $\sigma \subseteq \{1,\ldots,n\}$, and let $\mathbb{T}$ be a set of $\sigma$-transfer schemas. Assume we apply these schemas sequentially, starting with $\sequent{\contextGraph_1,...,\contextGraph_n,\assump}{\goa}$ and ending with $\sequent{\contextGraph_1',...,\contextGraph_n',\assump'}{\goa'}$. If all schemas in $\mathbb{T}$ are monotonic for $(\contextGraph_1',...,\contextGraph_n',\assump')$ and $\sequent{\contextGraph_1',...,\contextGraph_n',\assump'}{\goa'}$ is a valid sequent, then $\sequent{\contextGraph_1,...,\contextGraph_n,\assump}{\goa}$ is valid modulo $\sigma$.
\end{theorem}

The significance of this theorem is the basis for \textit{structure transfer}: given a sequent, $\sequent{\contextGraph_1,...,\contextGraph_n,\assump}{\goa}$, and $\sigma \subseteq \{1,\ldots,n\}$, we can apply $\sigma$-transfer schemas iteratively in such a way that, in each application, the target is reified enough for the schema to be applicable. If after such a sequence of applications we manage to obtain a valid sequent, $\sequent{\contextGraph_1',...,\contextGraph_n',\assump'}{\goa'}$, then the target graphs (those indexed by $\sigma$) are a transformation of the source graphs (those not indexed by $\contextGraph_i$).

We have seen that applying a transfer schema is just like applying a schema, but with the additional power to reify the target graphs while everything else is only mapped through label-preserving isomorphisms. Such modifications of a sequent make a difference when the target pattern graph of a transfer schema could not be matched before the $\sigma$-reification but can be matched after it.

\section{Structure Transfer: an algorithmic approach}\label{sec:StructureTransfer}
Structure transfer is a calculus by which we obtain a $\sigma$-reification of a given sequent that makes it valid for a multi-space $\mathcal{M}$. All the ingredients for this have already been presented. Specifically, Theorem~\ref{thm:transfer-sequence-valid} asserts that we can find such $\sigma$-reification by $\sigma$-applying transfer schemas sequentially. So let us review the conceptual tools we presented before, now under an algorithmic lens. 

A fundamental property of the algorithms presented here is that they are non-determi\-nistic. Specifically, in our pseudo-code, we include the use of an operator \textbf{find} which may get zero, one, or more results, in which case we are dealing respectively with either a dead end, a deterministic path or a branching of the search space. Thus, ultimately, structure transfer is a calculus, meaning that it establishes a set of rules and procedures that can be applied in a variety of ways. To develop specific methods and strategies for applying the rules of structure transfer appropriate heuristics must be used. In this paper we do not deal with such strategies and heuristics.

\subsection{Procedures for Schema Application}
Algorithms~\ref{alg:fwd-application} and \ref{alg:bwd-application} produce forward and backward applications of a schema $\sequent{\contextGraph_1,\ldots,\contextGraph_n,\break\assump}{\consequent}$ to a sequent  $\sequent{\contextGraph_1',\ldots,\contextGraph_n',\assump'}{\goa'}$ for a multi-space $\mathcal{M}$. 
Recall that the definition of a schema application involves finding a refinement of the schema to resemble the sequent in question, with a monotonicity constraint. In this paper we do not deal with the general problem of determining monotonicity, and most examples we have presented are trivially monotonic. Thus, the algorithms assume that knowledge about monotonicity is given, through some set $\mathbb{Q}$ whose elements are pairs of the form $(\texttt{sc},(\contextGraph_1,\ldots,\contextGraph_n,\assump))$ where $\texttt{sc}$ is a schema and $(\contextGraph_1,\ldots,\contextGraph_n,\assump)$ is pattern graph for which the schema is monotonic. If a schema, $\texttt{sc}$, is monotonic in general then $\mathbb{Q}$ has all pairs of the form $(\texttt{sc},(\contextGraph_1,\ldots,\contextGraph_n,\assump))$. 

The approach of Algorithm~\ref{alg:fwd-application} to forward-apply $\sequent{\contextGraph_1,\ldots,\contextGraph_n,\assump}{\consequent}$ to $\sequent{\contextGraph_1',\ldots,\contextGraph_n',\assump'}{\goa'}$ involves first finding whether $(\contextGraph_1,\ldots,\contextGraph_n,\assump)$ can be reified into $(\contextGraph_1',\ldots,\contextGraph_n',\assump')$, and if so, we can loosen the consequent, $\goa$, which yields a $\theDelta$, which can be added as a new assumption to obtain the result of the application (provided the monotonicity condition is met).
\begin{algorithm}[h!tb]\small
	\caption{Forward application of schema to sequent}\label{alg:fwd-application}
	\begin{algorithmic}[1]
		\Procedure{applySchemaFWD}{$\sequent{\contextGraph_1,\ldots,\contextGraph_n,\assump}{\consequent},\sequent{\contextGraph_1',\ldots,\contextGraph_n',\assump'}{\goa'},\mathbb{Q}$}
		\State \textbf{find} reification $\reify \colon (\contextGraph_1,\ldots,\contextGraph_n,\assump) \to (\contextGraph_1',\ldots,\contextGraph_n',\assump')$, if none found \textbf{fail}
		\State \textbf{find} $\delta$ and partial function $\loosen \colon \consequent \to \delta$ \ s.t:
		\State\quad\textbullet\ $\loosen|_{\consequent} \colon \consequent \to \loosen[\consequent]$ is a loosening map up to the tokens of $(\contextGraph_1 \cup \cdots \cup \contextGraph_n \cup \assump) \cap \consequent$
		\State\quad\textbullet\ $\loosen$ is compatible with $\reify$, and
		\State\quad\textbullet\ $\delta \subseteq \assump' \cup \loosen[\goa]$ 
		%\State\quad if none found \textbf{fail}
		\If{$(\sequent{\reify[\contextGraph_1],\ldots,\reify[\contextGraph_n],\reify[\assump]}{\loosen[\consequent]}, (\contextGraph_1',\ldots,\contextGraph_n',\assump')) \in \mathbb{Q}$}
		\State\Return $\sequent{\contextGraph_1',\ldots,\contextGraph_n', \assump' \cup \delta}{\goa'}$
		\EndIf
		\EndProcedure
	\end{algorithmic}
\end{algorithm}

Conversely, the approach of Algorithm~\ref{alg:bwd-application} to backward-apply $\sequent{\contextGraph_1,\ldots,\contextGraph_n,\assump}{\consequent}$ to $\sequent{\contextGraph_1',\ldots,\contextGraph_n',\assump'}{\goa'}$ involves first finding whether $\goa$ can be loosened into $\goa'$ (up to some tokens) and then finding whether $(\contextGraph_1,\ldots,\contextGraph_n,\assump)$ can be reified into some $(\contextGraph_1',\ldots,\contextGraph_n',\assump' \cup \theDelta)$. Here we note that the condition in line 7 means $\theDelta$ must contain every part of $\goa'$ that was not mapped by the loosening of $\consequent$ or is already in $\assump'$.
\begin{algorithm}[h!tb]\small
	\caption{Backward application of schema to sequent}\label{alg:bwd-application}
	\begin{algorithmic}[1]
		\Procedure{applySchemaBWD}{$\sequent{\contextGraph_1,\ldots,\contextGraph_n,\assump}{\consequent},\sequent{\contextGraph_1',\ldots,\contextGraph_n',\assump'}{\goa'},\mathbb{Q}$}
		\State \textbf{find} partial function $\loosen \colon \consequent \to \goa'$\ s.t.
		\State\quad\textbullet\ $\loosen|_{\consequent} \colon \consequent \to \loosen[\consequent]$ is a loosening map up to the tokens of $(\contextGraph_1 \cup \cdots \cup \contextGraph_n \cup \assump) \cap \consequent$
		\State\quad if none found \textbf{fail}
		\State \textbf{find} $\delta$ and reification $\reify \colon (\contextGraph_1,\ldots,\contextGraph_n,\assump) \to (\contextGraph_1',\ldots,\contextGraph_n', \assump' \cup \delta)$ \ s.t. 
		\State\quad\textbullet\ $\reify$ is compatible with $\loosen$, and
		\State\quad\textbullet\ $\goa' \subseteq \assump' \cup \delta \cup \loosen[\goa]$
		\State\quad if none found \textbf{fail}
		%\State $\refine \gets (\reify_1,\ldots,\reify_n,\loosen)$
		\If{$(\sequent{\reify[\contextGraph_1],\ldots,\reify[\contextGraph_n],\reify[\assump]}{\loosen[\consequent]}, (\contextGraph_1',\ldots,\contextGraph_n',\assump' \cup \delta)) \in \mathbb{Q}$}
		\State\Return $\sequent{\contextGraph_1',\ldots,\contextGraph_n',\assump'}{\delta}$
		\EndIf
		\EndProcedure
	\end{algorithmic}
\end{algorithm}

Algorithm~\ref{alg:application}, for general schema involves choosing between forward or backward applications. Again note the use of the \textbf{find} operator which identifies potential branching places in the search space.
\begin{algorithm}[h!tb]\small
	\caption{Application of schema to sequent}\label{alg:application}
	\begin{algorithmic}[1]
		\Procedure{applySchema}{$\texttt{sc},\texttt{sq},\mathbb{Q}$}
		\State \textbf{find} $\texttt{sq'}$ s.t.
		\State\quad $\texttt{sq'} = \textsc{applySchemaFWD}(\texttt{sc},\texttt{sq},\mathbb{Q})$ \ \textbf{or} \ $\texttt{sq'} = \textsc{applySchemaBWD}(\texttt{sc},\texttt{sq},\mathbb{Q})$
		\State\Return $\texttt{sq'}$
		\EndProcedure
	\end{algorithmic}
\end{algorithm}

Now, given procedures for applying schemas, we can check for validity with Algorithm~\ref{alg:valid}.
\begin{algorithm}[h!t]\small
	\caption{Validity of sequent}\label{alg:valid}
	\begin{algorithmic}[1]
		\Procedure{validSequent}{$\texttt{sq}, \mathbb{S},\mathbb{Q}$}
		\If{$\texttt{sq} \in \mathbb{S}$}
		\State\Return \texttt{true}
		\Else 
		\State{\textbf{find} $\texttt{sc} \in \mathbb{S}$\, and \ $\texttt{sq'}$ s.t.\
		  $\texttt{sq'} = \textsc{applySchema}(\texttt{sc}, \texttt{sq}, \mathbb{Q})$
		\State\quad \Return $\textsc{validSequent}(\texttt{sq'}, \mathbb{S}, \mathbb{Q})$}
		\State\quad if none found \textbf{fail}
		\EndIf
		\EndProcedure
	\end{algorithmic}
\end{algorithm}

\subsection{Procedures for Transfer Schema applications}
To apply a $\sigma$-transfer schema to a sequent means to produce a $\sigma$-reification of the sequent so that the schema is applicable. Thus, Algorithm~\ref{alg:transfer-application} starts by reifying every target graph, $\contextGraph_{i}'$, with $i \in \sigma$, to some graph $\contextGraph_{i}''$ in such a way that $\contextGraph_{i}''$ is also a reification of $\contextGraph_{i}$ (lines 2 to 4). This adds the context to every $\contextGraph_{i}'$ necessary for the schema to be applicable. Next, we find the label-preserving isomorphism of the rest of the sequent (lines 5 to 9), and return the schema application (line 10).
\begin{algorithm}[h!t]\small
	\caption{Application of transfer schema to sequent, for $\sigma = \{i_1,\ldots,i_k\}$}\label{alg:transfer-application}
	\begin{algorithmic}[1]
		\Procedure{applyTransferSchema}{$\{i_1,\ldots,i_k\},\sequent{\schemaContext_1,\ldots,\schemaContext_n,\assump}{\consequent},\sequent{\contextGraph_1',\ldots,\contextGraph_n',\assump'}{\goa'},\mathbb{Q}$}
		\State \textbf{find} $\contextGraph_{i_1}'',\ldots,\contextGraph_{i_k}''$ and homomorphism $\reify \colon (\contextGraph_{i_1}',\ldots,\contextGraph_{i_k}') \to (\contextGraph_{i_1}'',\ldots,\contextGraph_{i_k}'')$\ s.t. for all $1\leq l \leq k$,
		\State\quad\textbullet\ $f|_{\contextGraph_{i_l}'} \colon \contextGraph_{i_l}' \to \contextGraph_{i_l}''$ is a reification function, and
		\State\quad\textbullet\  $\contextGraph_{i_l}''$ is a reification of $\schemaContext_{i_l}$  
		\State\textbf{let} $\{j_1,\ldots,j_{n-k}\} = \{i,\ldots,n\} \setminus \{i_1,\ldots,i_k\}$
		\State \textbf{find} $\contextGraph_{j_1}'',\ldots,\contextGraph_{j_{n-k}}'',\assump'',\goa''$ and isomorphism\\\hspace{1.23cm} $h \colon (\contextGraph_{j_1}',\ldots,\contextGraph_{j_{n-k}}',\assump',\goa') \to (\contextGraph_{j_1}'',\ldots,\contextGraph_{j_{n-k}}'',\assump'',\goa'')$ \ s.t.
		\State\quad\textbullet\ $h$ is label-preserving up to the tokens of $\contextGraph_{i_1}' \cup \cdots \cup \contextGraph_{i_k}'$ and
		\State\quad\textbullet\ $h$ is compatible with $\reify$.
		\State\Return$\textsc{applySchema}(\sequent{\schemaContext_1,\ldots,\schemaContext_n,\assump}{\consequent},\sequent{\contextGraph_1'',\ldots,\contextGraph_n'',\assump''}{\goa''},\mathbb{Q})$
		\EndProcedure
	\end{algorithmic}
\end{algorithm}

Finally, Algorithm~\ref{alg:structure-transfer}, for structure transfer, relies on the recursive, non-deterministic, application of $\sigma$-transfer schemas until a valid sequent is found, which, as we know from Theorem~\ref{thm:transfer-sequence-valid}, means the original sequent was valid modulo $\sigma$. If such a valid sequent is found, this is the transformation of the original we want, so we call this a \texttt{full} transformation. Otherwise we return a \texttt{partial}. With a \texttt{full} sequent we can ensure any instantiation of the assumptions implies that there is an instantiation of the goal, and with a \texttt{partial} one we cannot. Of course, given the non-deterministic nature of the procedure, \texttt{partial} transformations can be useful and have different heuristic values, not discussed in this paper. 
\begin{algorithm}[h!bt]\small
	\caption{Structure transfer with target specified by $\sigma$}\label{alg:structure-transfer}
	\begin{algorithmic}[1]
		\Procedure{structureTransfer}{$\sigma, \texttt{sq}, \mathbb{S},\mathbb{Q}$}
		\If{$\textsc{validSequent}(\texttt{sq}, \mathbb{S},\mathbb{Q})$}
		\State\Return $(\texttt{sq}, \texttt{full})$
		\Else 
		\State \textbf{find} $\texttt{sc} \in \mathbb{S}$\ and\ $\texttt{sq'}$\ s.t.\
		 $\texttt{sq'} = \textsc{applyTransferSchema}(\sigma,\texttt{sc}, \texttt{sq}, \mathbb{Q})$
		\State\qquad \Return $\textsc{structureTransfer}(\sigma, \texttt{sq'}, \mathbb{S}, \mathbb{Q})$
		\State\qquad if none found \Return $(\texttt{sq}, \texttt{partial})$
		\EndIf
		\EndProcedure
	\end{algorithmic}
\end{algorithm}

\subsection{Structure transfer for the depict-and-observe process}
We started this paper with an informal presentation of the depict-and-observe process, wherein we take a formula in \textsc{Set Algebra}, then we \textit{depict} it in \textsc{Euler Diagrams}, and then we produce an \textit{observation} from the diagram, which is itself a formula in \textsc{Set Algebra}. Now we shall see that this problem can be easily encoded and then solved with structure transfer.

\paragraph{Problem encoding} The multi-space we will work on is $(\mathcal{C},\mathcal{D},\mathcal{C},\mathcal{G})$ where $\mathcal{C}$ encodes \textsc{Set Algebra}, $\mathcal{D}$ encodes \textsc{Euler Diagrams} and $\mathcal{G}$ encodes relations across and within the spaces. Given expression $A \subseteq B \wedge B \cap C = \emptyset$, we want to find a diagram that depicts it and a set expression that can be observed from the diagram. In other words, given the sequent $\sequent{\contextGraph_1,\contextGraph_2,\contextGraph_3,\assump}{\goa}$, visualised below, we want to find a $\{2,3\}$-reification of it that makes it valid. Note that we can even encode the relation of \textit{not being observable from} to prevent a trivial observation.
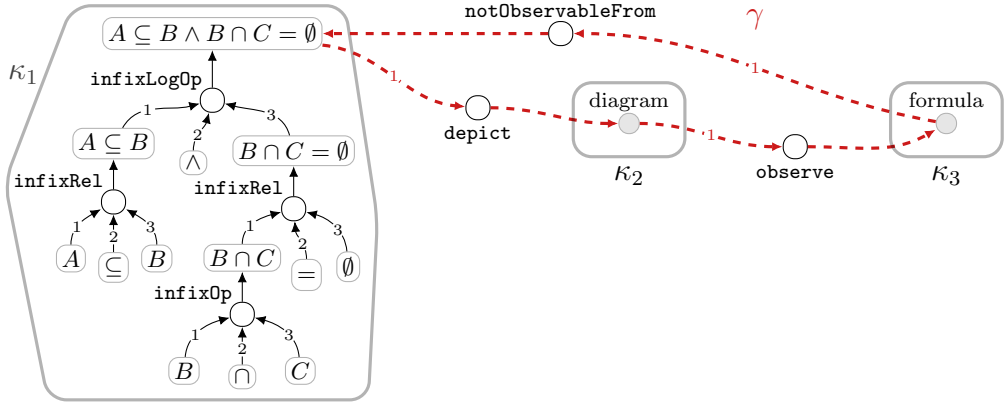
\begin{center}

\end{center}

\paragraph{Solution approach} 
The first step is to apply a sequence of $\{2,3\}$-transfer schemas\footnote{Note that for schemas to be applicable in this setting they need to be schemas for the space $(\mathcal{C},\mathcal{D},\mathcal{C},\mathcal{G})$. We previously presented some relevant schemas for space $(\mathcal{C},\mathcal{D},\mathcal{G})$ which technically need to be \textit{lifted} to $(\mathcal{C},\mathcal{D},\mathcal{C},\mathcal{G})$. From a theoretical perspective, lifting schemas is a trivial operation, not discussed in this paper.}, similar to those of examples~\ref{ex:schema-application} and~\ref{ex:transfer-application} in order to reify $\contextGraph_2$ to a pattern graph, $\contextGraph_2'$. As we will see, this results in a construction of a diagram with type $\mathtt{\{AB,B,C,\emptyset\}}$ (recall this is the type of $\AsubBdisjCinline$). Step-by-step, this involves first breaking down the goal of depicting $A \subseteq B \wedge B \cap C = \emptyset$ into two goals, one for each of the conjuncts, using the schema from Example~\ref{ex:depict-conj}. Then, we apply the schema from example~\ref{ex:disjoint-depict} (as a $\{2,3\}$-transfer schema) to reify $\contextGraph_2$ into a graph (unlabelled in the figure below) constrained by a goal which specifies disjointness\footnote{The graphs $\contextGraph_1$ and $\contextGraph_3$ remain label-isomorphic in this step.}:
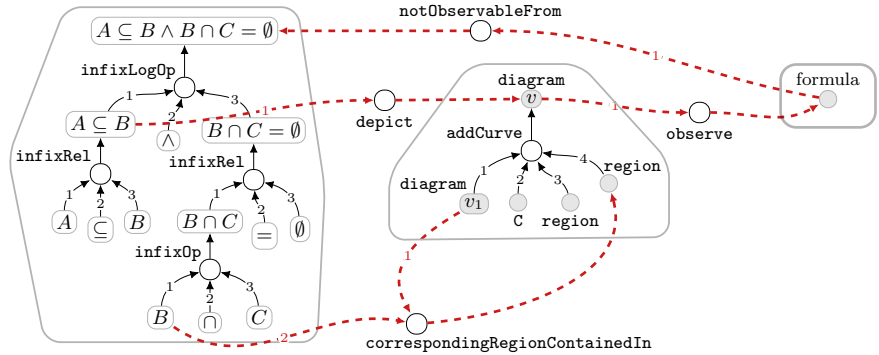
\begin{center}
	\adjustbox{scale = 0.85}{
}
\end{center}
Next, we can reduce the goal stating that $v$ must depict $A \subseteq B$ to one that states that $v_1$ depicts $A \subseteq B$, using a simple schema not presented here (included in Appendix~\ref{sec:appendixTransfer}). This results in the following:
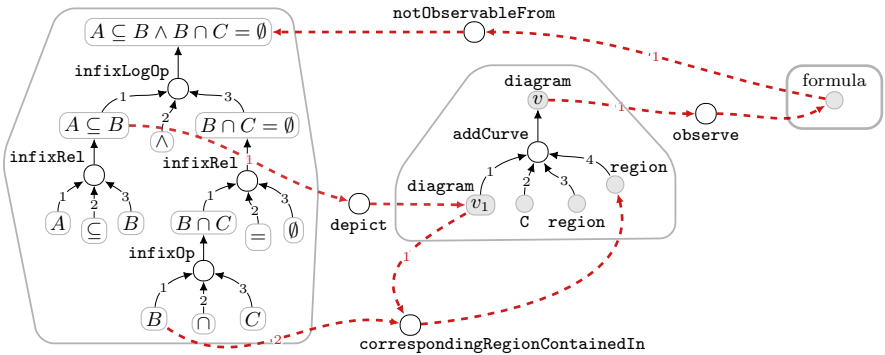
\begin{center}
	\adjustbox{scale = 0.85}{%
}
\end{center}
Then, an application of the schema from Example~\ref{ex:subset-depict} adds constraints on the structure of $v_1$. Then, the goals for constructor \texttt{correspondingRegionContainedIn} can be satisfied by applications of the schema from Example~\ref{ex:corr-var-region}, which effectively specialises the types of the Euler diagram pattern until $v$ has type $\mathtt{\{AB,B,C,\emptyset\}}$, which we know is the type of $\AsubBdisjCinline$. At this stage we have:
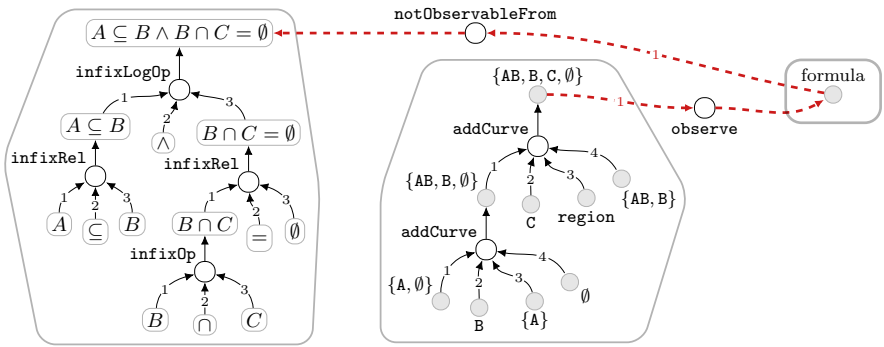
\begin{center}
	\adjustbox{scale = 0.85}{%
}
\end{center}

The next step is to perform more applications to reify $\contextGraph_3$ in order to satisfy both remaining goals. There are many ways of doing this, but we can do it using schemas~\ref{ex:depict-observe} and then~\ref{ex:disjoint-depict} to obtain a pattern graph which constructs $A \cap B = \emptyset$. The goal \texttt{notObservableFrom} can be discharged with more applications of schemas included in Appendix~\ref{sec:appendixTransfer}, noting that this can only be done if $\contextGraph_3'$ is distinct from $A \subseteq B$ nor $B \cap C = \emptyset$. The result is a sequent, $\sequent{\contextGraph_1',\contextGraph_2',\contextGraph_3',\assump'}{\goa'}$, with empty $\assump'$ and $\goa'$, which is clearly valid.
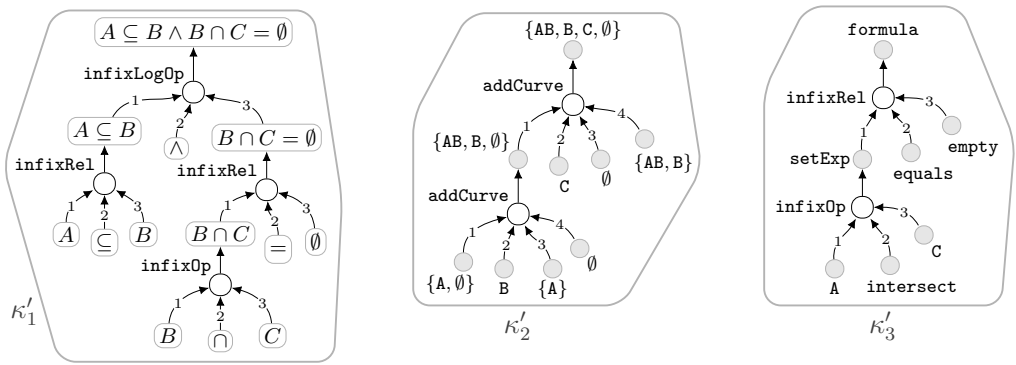
\begin{center}
	\adjustbox{scale = 0.9}{
}
\end{center}

The pattern graphs $\contextGraph_2'$ and $\contextGraph_3'$ are precisely patterns whose instantiations in $\mathcal{D}$ and $\mathcal{C}$ (the construction spaces for \textsc{Euler Diagrams} and \textsc{Set Algebra}) construct $\AsubBdisjCinline$ and, respectively, $A \cap C = \emptyset$. The pattern $\contextGraph_1'$ is label-isomorphic to our original $\contextGraph_1$, as we only performed $\{2,3\}$-transfer schema applications.

\subsection{Conclusion on structure transfer}\label{sec:structure-transfer-conclusion}
Structure transfer allows us to transform a given representation by exploiting transfer schema applications. It starts with a sequent $\sequent{\contextGraph_1,\ldots,\contextGraph_n,\assump}{\goa}$ where $(\contextGraph_1,\ldots,\contextGraph_n)$ sets a structural context, $\assump$ encodes some assumptions and $\goa$ encodes some goals. A target is specified by $\sigma \subseteq \{1,\ldots,n\}$, and an application of structure transfer given $\sigma$ returns a sequent $\sequent{\contextGraph_1',\ldots,\contextGraph_n',\assump'}{\goa'}$ where the target specified by $\sigma$ is our sought-after transformation of the source. In the \texttt{full} case, the resulting sequent is known to be valid, and in the \texttt{partial} case it may or may not be valid, so heuristics are needed to determine whether the transformation is useful for the given task. If the resulting sequent is valid it ensures the instantiatability of $(\contextGraph_1',\ldots,\contextGraph_n',\assump' \cup \goa')$ given an instantiation of $(\contextGraph_1',\ldots,\contextGraph_n',\assump')$, which may not be known, but which we may try to solve as the instantiation problem presented in Section~\ref{sec:instSchema} simply by applying schemas to, say, the graph $\contextGraph_1'\cup\cdots\cup\contextGraph_n'\cup\assump'$. Although it may not seem so, the general instantiation problem has its own challenges that go beyond the methods in this paper. Specifically, a monotonicity condition needs to be satisfied for every schema application. While checking for monotonicity is trivial for any meta-space $\goalSpace$ that encodes a typical monotonic logic, if the space where we are trying to determine instantiatability is a grammar like \textsc{Set Algebra}, where even simple instantiation schemas are not monotonic (see Example~\ref{ex:non-monotonic}), then monotonicity needs to be determined case-by-case. This is a research and implementation challenge for the future. In this paper we will not address directly the general case for checking the instantiability of arbitrary pattern graphs in multi-spaces.

We have an implementation of the core notions of RST and structure transfer, called Oruga~\cite{raggi2022oruga,raggirep2rep}. The implementation of Oruga assumes that any meta-space encoded is monotonic, and it does not deal with the instantiation problem that comes after structure transfer. The search space is navigated with heuristics and parameters that can be changed according to the task, but there is still much work left to enable multiple strategies. The main challenges for this and any implementations are:
\begin{enumerate}[itemsep=0pt, topsep = 4pt]
	\item strategies for navigating the structure transfer search space,
	\item strategies for inference about monotonicity, and more generally for determining instantiatability,
	\item a language for expressing and declaring schemas. For instance, if we want to express families of schemas, such as those presented in Examples~\ref{ex:fuzzy-schema}, \ref{ex:euler-singletons}, \ref{ex:addCurve-bottom-up}, \ref{ex:addCurve-top-down}, \ref{ex:disjoint-depict} and~\ref{ex:subset-depict} we need a powerful (meta-)type theory.
\end{enumerate}

\section{Related concepts and applications}\label{sec:properties}
Representational Systems Theory provides new foundations for thinking about structures that are familiar to logicians, computer scientists, linguists and cognitive scientists. In particular, we claim that constructions in construction spaces generalise syntax trees, allowing for some atypical (but useful) structures. Now, one of the first things that we must do when providing a new and exciting structure that \textit{generalises} over something familiar, is to show that the good-old familiar techniques are still available in the new structures.

Sections~\ref{sec:theory} and~\ref{sec:transfer} provide evidence that RST can be used for one of its main motivating problems, which is that of producing transformations across representational systems. Now, it happens to be that structure transfer, when restricted, looks like some known techniques in formal methods.

\subsection{Formal methods: rewriting, abstraction and refinement}
One technique generalised by structure transfer is \textit{term rewriting}. %, which we will call simply \textit{rewriting}. 
In formal theories with equality, rewriting is defined as the process by which, given $t_1=t_2$, we can replace $t_1$ for $t_2$ within a term. 
%The principle associated with it is that $t_1 = t_2$ implies $f(t_1) = f(t_2)$ for any $f$, and when we apply this principle to predicates we may use the associated inference rule:
%\begin{prooftree}\footnotesize
%	\AxiomC{$t_1 = t_2$}
%	\AxiomC{$P[t_1/t_2]$}
%	\BinaryInfC{$P$}
%\end{prooftree}
%Now, in Section~\ref{sec:transfer} we demonstrated that structure transfer can be used for rewriting given the source and target construction spaces are the same. 
For this, it is not difficult to define the notion of \textit{reflexive} schemas. Given some types $\tau,\tau_1,\ldots,\tau_n$ and notions of equivalence, $\equiv,\equiv_1,\ldots,\equiv_n$, for each of them, captured in a meta-space, we (can) introduce schemas of the form:
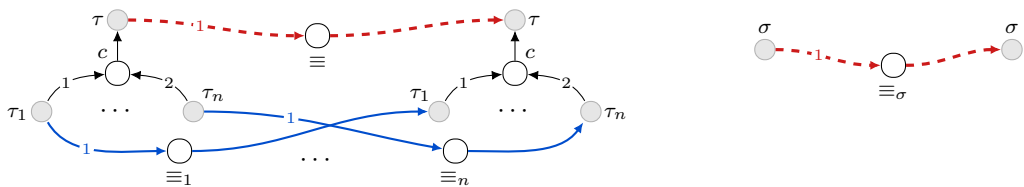
\begin{center}
	\adjustbox{valign=t}{
}
\end{center}
for every constructor, $c$, and every \textit{minimal} type, $\sigma$. Of course, this requires the type system to have such minimal types. Moreover, such a set of schemas can be enriched with knowledge about the behaviour of the equivalence relations, such as:
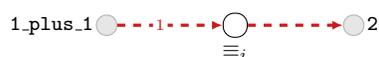
\begin{center}
	%
\end{center}
A set of reflexive schemas with some additional ones like this one above results in a \textit{rewrite system}. Given a sequent of the form $\sequent{\contextGraph_1,\contextGraph_2,\assump}{\goa}$ where $\contextGraph_1$ encodes, say $2(1+1) \leq x$, and\pagebreak[4] $\contextGraph_2$ contains a single token $t_2$, with empty $\assump$ and $\goa$ encoding the equivalence of $2(1+1) \leq x$ and $t_2$, structure transfer can give us a $\{2\}$-reification yielding $\contextGraph_2'$ which encodes $2(2) \leq x$.

\subsubsection{Abstraction and refinement}
Structure transfer can handle reasoning by substitution with ad-hoc equalities, but it can handle what \citeA{coen2004semi} calls \textit{sub-equational rewriting}, which refers to rewriting where terms are replaced for non equal ones but where a transitive relation holds (e.g., entailment $\vdash$, or the order of integers $\leq$).
\begin{center}

\end{center}
The result of sub-equational rewriting is not an equivalent term, but one that stands in a desired relation with the source term.

And finally, it should be clear that the relations encoded in $\mathcal{I}$ need not be transitive, so structure transfer is able to satisfy arbitrary relations, which is similar to the \textit{transfer} tactic introduced by \citeA{huffman2013lifting}, in Isabelle/HOL. Their \textit{transfer rules} are analogous to transfer schemas, but of course, the former is limited to terms and relations in HOL, with some additional limitations on the structure of the source and target terms. The mechanism of the transfer tactic, and similar tools (e.g.,~\citeA{cohen2013refinements}) are used for data \textit{abstraction} and \textit{refinement}~\cite{tabareau2018equivalences,delaware2015fiat,abrial2010rodin}. Abstract data-types are useful for human-level specification, reasoning and verification, and more refined data-types are necessary at the implementation-level for computation. Formal links between various levels of abstraction are necessary and widely used in formal verification. For example, sets are useful for specifying programs, but their specific implementation can vary. For instance, a very simple refinement of a set is a list, and a translation of various expressions between these data-types is possible. For example, it is not hard to see that with an appropriate set of transfer schemas, we can derive that $(\{1\}\cup (\texttt{set\_of L}))\cup (\{4\}\cup \emptyset)$ is the set of list $\texttt{append}\ (\texttt{insert}\ 1\ (\texttt{rev}\ \texttt{L}))\ (\texttt{insert}\ 4\ [])$.

The concepts of abstraction and refinement have applications beyond formal verification. For example, within RST, when trying to develop or specify a representational system, we have modelling choices, such as whether to include parentheses as tokens, or whether to model some operation directly as a constructor or to model it as a token. Decisions concerning how much to abstract or refine a model usually depend on the purpose of the model. Each model may have its benefits and drawbacks depending on its use-case. For example, if we needed to assess a representation based on how a novice may interact with it, or if we simply need to know how much ink is needed to print it, we may need to capture the symbols at a high granularity. However, if we want a compact way of capturing the semantics, a low granularity may be more desirable. If this modelling variability is expected, it is imperative that tools for translating between them can readily be created. Suffice it to say, we can use structure transfer for this. Below, the proposition $(A \wedge B) \vee B$ is modelled in four ways, ordered in decreasing level of abstraction. The left-most model does not even distinguish between two distinct tokens of the same type, while the right-most model captures often-ignored tokens such as parentheses.
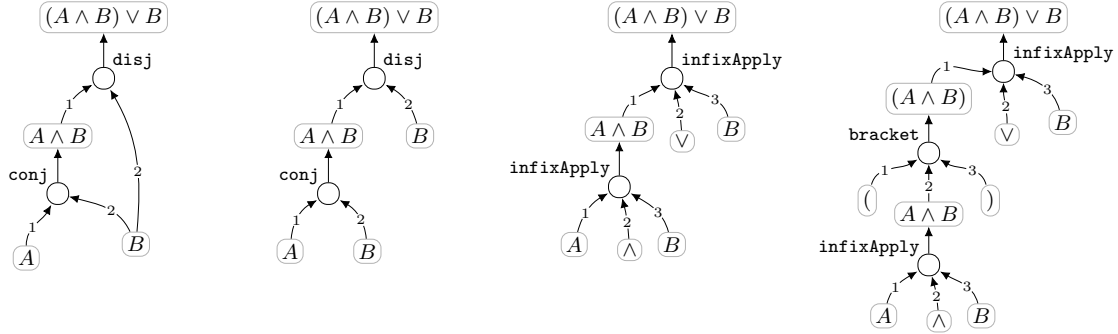
\begin{center}
	\adjustbox{valign = c, scale = 0.9}{%
}
\end{center}
It is not difficult to see that we can define transfer schemas between three such construction spaces, so that we can exploit transformations between them.

\subsection{Analogy: structure-mapping and anti-unification}\label{sec:structure-mapping}

% And finally, ANALOGY \citeA{gentner1983structure}.
Finally, we turn our attention towards mechanisms more philosophically similar, than formally similar. Structure transfer, as outlined here, has clear parallels to analogy, specifically \emph{structure-mapping} by~\citeA{gentner1983structure,falkenhainer1989structure} and related approaches, such as \textit{anti-unification} by~\citeA{krumnack2007restricted,schmidt2014heuristic}. The most important distinction between structure-mapping and structure transfer is that the latter is a procedure for \textit{applying} a known mapping rather than \textit{discovering} it. Here we explore briefly how analogy can be modelled in terms of schemas, and further we will suggest how our tools could be used for discovering analogies.

The core idea of structure-mapping is that an analogy between two domains is characterised by similarity at the level of relations between objects rather than at the level of the objects themselves or their attributes (i.e., unary properties). Thus, it is about how a seemingly arbitrary mapping at the object level (e.g., electron $\mapsto$ planet, atom's nucleus $\mapsto$ sun) is meaningful because it preserves some relations between the objects (e.g., the atom's nucleus and the electron attract each other like the sun and a planet attract each other). Moreover, \citeauthor{gentner1983structure}'s \emph{systematicity principle} favours finding coherent \emph{systems} of relationships, not just one-off relationships. 

It is easy to see that invariants across domains, as those that characterise an analogy according to Gentner can be encoded with some simple schemas, such as the one drawn below, which simply captures the rule: \textit{two statements are analogous modulo \texttt{R} (denoted $\simeq_{\texttt{R}}$) if their arguments are related by \texttt{R}}. If \texttt{R} maps \texttt{nucleus} to \texttt{sun} and \texttt{electron} to \texttt{planet} then the statements $\texttt{attracts}(\texttt{nucleus},\texttt{electron})$ and $\texttt{attracts}(\texttt{sun},\texttt{planet})$ are analogous modulo  \texttt{R}. The generalisation of this rule can be captured as a schema, drawn below:  %A first-order version, where the predicates must map directly, can be captured by transfer schemas such as the following:
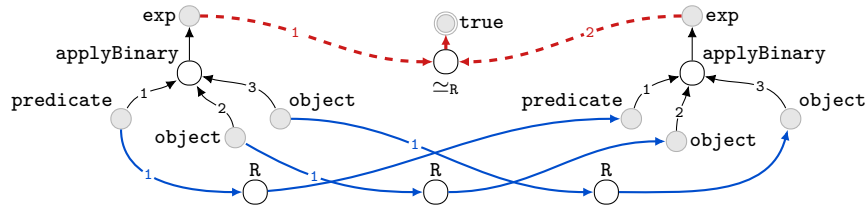
\begin{center}

\end{center}
This is not dissimilar to the idea that two statements are analogous if they can be anti-unified. Structure-mapping~\cite{gentner1983structure} requires that knowledge for both domains be encoded with identical relations, such as \textit{attracts}, while higher order versions~\cite{krumnack2007restricted,schmidt2014heuristic} can map the predicates, as the schema above does.

To capture that two knowledge bases are analogous we might use schemas such as the one below (left), stating that two knowledge bases are analogous if their individual statements are analogous, or a \textit{fuzzy} version (right) wherein the {truth value} (\textit{strength}) of the analogy is a function of the values of the arguments:
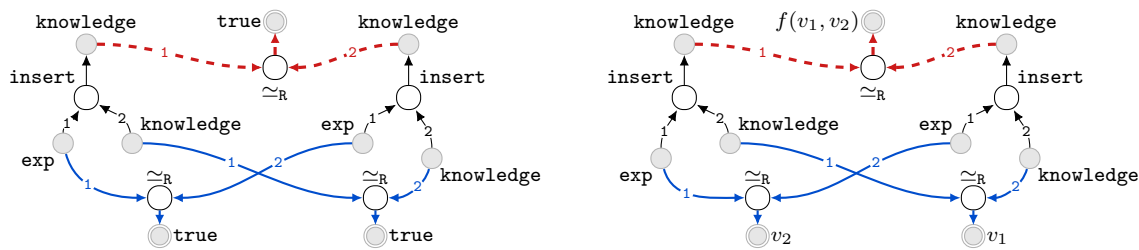
\begin{center}

\end{center}
Other rules, such as `$t_1 = t_2$ implies $t_1 \simeq_{\texttt{R}} t_2$' may also be encoded.

In this context, the question that structure-mapping aims to solve, given a pair of knowledge bases $t_1$ and $t_2$, is whether we can find a mapping $\texttt{R}$ that ensures $t_1 \simeq_{\texttt{R}} t_2$ (or has maximal value in the fuzzy case). Our proposal for approaching this problem using the notions presented in this paper is to use schemas abductively: that is, encode knowledge bases $\contextGraph_1$ and $\contextGraph_2$, noting that each knowledge base can be constructed in many ways. Try to show that $\sequent{\contextGraph_1,\contextGraph_2,\assump}{\goa}$ is valid using Algorithm~\ref{alg:valid}, for some $\goa$ that encodes the relation $\simeq_{\texttt{R}}$ between the pair of knowledge bases, for some variable $\texttt{R}$. Without knowing $\texttt{R}$ we will not be able to prove validity, but the remaining goals, $\goa$, may be used to inform how a good $\texttt{R}$ may look like. In the fuzzy case we may try to find an $\texttt{R}$ that maximises the truth value of the analogy.

A consequence of using this approach for finding an analogy is that once we have discovered a mapping \texttt{R} we can use structure transfer to transform arbitrary statements from the language of one knowledge base to the other, regardless of whether the statement or its resulting transformation, lives in the knowledge base. Thus, structure transfer can be used as a mechanism for the actual \textit{transfer} of knowledge that happens once a mapping is grasped~\cite{gick1983schema}.

\subsection{Last remarks on the applications of structure transfer}
We have shown how structure transfer generalises various known procedures. Of course, all our claims of generality are in terms of computability and not complexity or efficiency. Our notion of structure transfer is defined by a search space, which may be infinite, and we do not provide explicit strategies or heuristics for traversing the space.

The fact that structure transfer is, in principle, a really powerful tool depending on the setting in which it is used calls for the \textit{tactification} of it, wherein depending on the task, we can use a specific version of it. For example, for certain scenarios we might consider sets of transfer schemas where all elements are reflexive except for one (for standard rewriting), for others we might consider quasi-reflexive ones (for analogy). Or we may conceive of decision procedures where structure transfer is applied iteratively, modifying the set of transfer or inference schemas.

\section{Conclusion}\label{sec:conclusion}
\citeA{raggi2022rst} introduced Representational Systems Theory motivated by the prospect of understanding the structure of representations in order to facilitate their analysis, use and transformations. One of the main innovations of RST was the notion of a construction space, where the structure of representations is captured through a graph-theoretic generalisation of syntax trees, along with an unassuming type structure.

All the content of this paper builds on the foundation provided by the concept of a construction space. From the abstract nature of this concept we get representational generality. We extended the theory to include the notions of schemas. As we showed, schemas can be applied like logical rules to determine the validity of sequents, and in their transfer version, to generate structure in some target spaces, to satisfy some specified constraints. This is the key of structure transfer. Crucially, this means that we can use knowledge about the preservation of information across construction spaces to produce re-representations through \textit{any} relation about which we have some knowledge.

We showed that structure transfer has a wide range of applications for producing transformations across systems -- which is valuable given the heterogeneous nature of our symbolic systems. In particular, we showed how it can be used to produce diagrams from sentences, and sentential observations from diagrams, it can be used for data abstraction and refinement, and for modelling and enacting analogies. Moreover, the procedures presented in this paper are based on an extensible knowledge base, as opposed to hard-coded transformations, and the knowledge can be encoded in any variety of logics expressible as multi-spaces with the use of pattern graphs.

Our vision is that the ideas in this paper are used to build tools that not only \textit{manage} heterogeneity but \textit{exploit} it.
Exploiting the intrinsic heterogeneity of our symbolic systems is a huge and important challenge for science and communication, and we have provided some foundations and methods to do this.
\acks{We thank Atsushi Shimojima and Jean Flower for their helpful comments. The research is funded by the UK EPSRC, grant numbers EP/R030642/1, EP/T019603/1, EP/T019034/1 and EP/R030650/1.}

\newpage
\appendix
\section{Representational Systems Theory}\label{secApp:RST}

The next two examples  -- on geometric constructions and proofs -- complement that on matrix algebra in Section~\ref{sec:theory}, demonstrating what can be modelled with construction spaces.

\begin{example}[\textsc{Geometric Constructions}] Constructors can encode geometric operations, with the tokens encoding points (e.g. $(7.8,2.6)$), magnitudes (e.g. $2.6$), and geometric figures (e.g. circles). The graph below shows two ways of constructing \adjustbox{raise=-0.04cm}{\scalebox{0.4}{\geometricA}}, with two rotations resulting in a cycle. Note that the tokens that represent points (e.g. $(7.8,2.6)$), magnitudes (e.g. $2.6$), and angles (e.g., $135^{\circ}$), are not themselves graphical tokens but necessary pieces of information to build the constructors' outputs.
	\begin{center}

	\end{center}
\end{example}
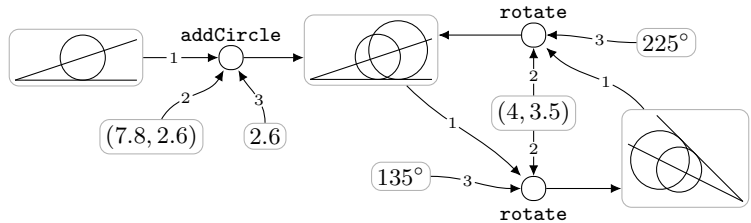
\begin{example}[\textsc{Proofs/Arguments}] We can use constructors to encode inference rules, such as modus ponens, conjunction introduction, or universal specification. Given a constructor vertex, its output is the conclusion of the input premises given the inference rule:
	\begin{center}
	%
	\end{center}
\end{example}

The next example focuses on using a construction system to encode low-level properties of tokens using a meta-space.

%\gnote{Comment to add two contours: one around Gc and one around Gi, with labels to indicate the space from which the graph arises.}

\begin{example}[Modelling low-level properties of symbols] Consider the construction space, $\mathcal{C}$, for \textsc{Set Algebra}. An meta-space, $\mathcal{I}$, for $\mathcal{C}$ can be defined by adding two types of meta-tokens: a boolean one (tokens $\top$ and $\bot$ of types \texttt{true} and \texttt{false}) for determining truth, and a numerical one for determining quantities. The meta-constructor \texttt{isLeftOf} captures whether one token is to the left of the other, and \texttt{inkUsed} captures the amount of ink, in cubic micrometers, used to print the token. Here the graph with black arrows belongs to $\mathcal{C}$ and the graph with thick blue arrows belongs to $\mathcal{I}$.
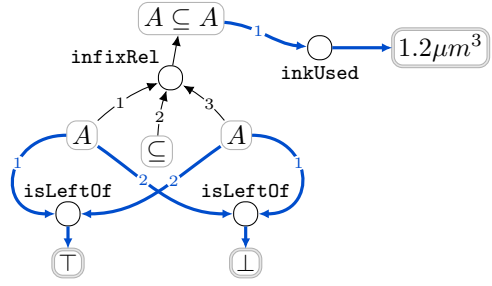
\begin{center}
	\begin{tikzpicture}[construction]\small
		\node[termrep](t1){$A \subseteq A$};
		\node[constructor = {\texttt{infixRel}}, below left = 0.4cm and -0.6cm of t1](u1){};
		\node[termrep, below left = 0.4cm and 0.8cm of u1](t11){$A$};
		\node[termrep, below left = 0.6cm and -0.2cm of u1](t12){$\subseteq$};
		\node[termrep, below right = 0.4cm and 0.5cm of u1](t13){$A$};
\node[constructorIS = \texttt{inkUsed}, below right = 0.0cm and 1.1cm of t1] (ccc) {};
\node[termIrep, right = 0.8cm of ccc] (ttt) {$1.2\mu m^3$};
		\node[constructorN = \texttt{isLeftOf}, below left = 0.45cm and 0.8cm of t12] (cc) {};
		\node[termIrep, below = 0.3cm of cc] (tt) {\footnotesize$\top$};
		\node[constructorN = \texttt{isLeftOf}, below right = 0.45cm and 0.8cm of t12] (cc') {};
		\node[termIrep, below = 0.3cm of cc'] (tt') {\footnotesize$\bot$};
		\path[->, darkblue, very thick]
		(ccc) edge (ttt)
		(t1) edge[out = -5, in = 175] node[index label]{1} (ccc);
		\path[->, darkblue, very thick]
		(cc') edge (tt')
		(t11) edge[out = -35, in = 180, looseness = 1.0] node[index label,pos =0.35]{2} (cc')
		(t13) edge[out = -0, in = 0, looseness = 2] node[index label]{1} (cc');
		\path[->, darkblue, very thick]
		(cc) edge (tt)
		(t11) edge[out = -180, in = 180, looseness = 2] node[index label]{1} (cc)
		(t13) edge[out = -145, in = 0, looseness = 1.0] node[index label,pos =0.35]{2} (cc);
		\path[->]
		(u1) edge (t1)
		(t11) edge[bend left = 10] node[index label]{1} (u1)
		(t12) edge[bend left = 7] node[index label]{2} (u1)
		(t13) edge[bend right = 10] node[index label]{3} (u1);
	\end{tikzpicture}
\end{center}
This example highlights that the meta-space may encode information about tokens not available at the type-system level. For example, both occurrences of $A$ in $A \subseteq A$ might be indistinguishable at the type level, yet we can identify different properties of them.
\end{example}

\section{Sequents and Schemas}\label{secApp:IIATS}
We restate and prove lemma~\ref{lemma:weakening-image}.

\begin{lemma}\label{alemma:weakening-image}
Let $\sequent{\contextGraph_1,\ldots,\contextGraph_n,\assump}{\goa}$ be a sequent for multi-space system $\multiSpace$ and let $\sequent{\contextGraph_1',\ldots,\contextGraph_n',\assump'}{\goa'}$ be a weakening of $\sequent{\contextGraph_1,\ldots,\contextGraph_n,\assump}{\goa}$ with map $\refine$. If $\sequent{\contextGraph_1,\ldots,\contextGraph_n,\assump}{\goa}$ is a schema for $\multiSpace$ then so is $\sequent{r[\contextGraph_1],\ldots,r[\contextGraph_n],\refine[\assump]}{r[\goa]}$.
\end{lemma}

\begin{proof}
Assume that $\sequent{\contextGraph_1,...,\contextGraph_n,\assump}{\goa}$ is a schema for $\multiSpace$ and let $\specialise\colon (\refine[\contextGraph_1],..., \refine[\contextGraph_n],\refine[\assump])\to (\contextGraph_1'',...,\contextGraph_n'',\assump'')$ be an instantiatable specialisation function. We show that we can extend $\specialise$ to map $\refine[\goa]$ to some instantiatable specialisation in the meta-space, $\goalSpace$. To simplify notation, we start by defining $\refine_a$ and $\refine_g$ to be the functions $\refine|_{(\contextGraph_1,...,\contextGraph_n,\assump)}\colon (\contextGraph_1,...,\contextGraph_n,\assump) \to (\refine[\contextGraph_1],..., \refine[\contextGraph_n],\refine[\assump])$ and, resp., $\refine|_{\goa}\colon \goa\to \refine[\goa]$. We have it that $\refine_a$ is a specialisation of $(\contextGraph_1,...,\contextGraph_n,\assump)$ in $\multiSpace$. Therefore, it must be that $\specialise\circ \refine_a\colon (\contextGraph_1,...,\contextGraph_n,\assump) \to (\contextGraph_1'',...,\contextGraph_n'',\assump'')$ is an instantiatable specialisation function. Given that $\sequent{\contextGraph_1,...,\contextGraph_n,\assump}{\goa}$ is a schema, there exists an instantiatable specialisation function $\specialise'\colon (\contextGraph_1,...,\contextGraph_{n},\assump\cup \goa) \to (\contextGraph_1'',...,\contextGraph_{n}'',\assump''\cup\goa'')$, in $\multiSpace$, that extends $\specialise\circ \refine_a$.
  Setting $\specialise''$ to be the restriction of $\specialise'$ to domain $\refine_{g}^{-1}[\goa']$, we have $\specialise''\colon \refine_{g}^{-1}[\goa']\to \goa'''$, where $\goa'''$ is the subgraph of $\goa''$ that ensures $\specialise''$ is surjective.  We show that %
  $$\specialise\cup (\specialise''\circ \refine_g^{-1})\colon (\refine[\contextGraph_1],...,\refine[\contextGraph_{n}],\refine[\assump_n]\cup \refine[\goa'])\to (\contextGraph_1'',...,\contextGraph_{n}'',\assump''\cup \goa''')$$
  is an instantiatable specialisation function that extends $\specialise$. Firstly, by construction, $\specialise\cup (\specialise''\circ \refine_g^{-1})$  is a function. Clearly the instantiatablilty of $(\contextGraph_1'',...,\contextGraph_{n-1}'',\assump''\cup \goa'')$ in $\multiSpace$
implies the instantiatablility of $(\contextGraph_1'',...,\contextGraph_{n}'',\assump''\cup \goa''')$ in $\multiSpace$. We already have it, by assumption, that $\specialise$ is a specialisation in $\multiSpace$.
What remains is to show that $\specialise''\circ \refine_\goa^{-1}\colon \refine[\goa]\to \goa'''$ is a specialisation in $\goalSpace$. Well, the function $\refine_g$ is a generalisation in $\goalSpace$, so its inverse, $\refine_g^{-1}$, is a specialisation in $\goalSpace$. In addition, $\specialise''$ is a restriction of a specialisation in $\goalSpace$ and, thus, $\specialise''$ is itself a specialisation in $\goalSpace$. Trivially, then, $\specialise''\circ \refine_g^{-1}\colon \refine[\goa]\to \goa'''$ is a specialisation in $\goalSpace$, as required. Therefore, $\specialise\cup (\specialise''\circ \refine_g^{-1})$ is an instantiatable specialisation that extends $\specialise$. Hence  $\sequent{\refine[\contextGraph_1],...,\refine[\contextGraph_n],\refine[\assump]}{r[\goa]}$ is a schema.
\end{proof}

We restate and prove theorem~\ref{thm:weakening-schema}

\begin{theorem}\label{athm:weakening-schema}
Let $\sequent{\contextGraph_1,...,\contextGraph_n,\assump}{\goa}$ be a schema for multi-space system $\multiSpace$
with refinement $\sequent{\contextGraph_1',...,\contextGraph_n'\assump'}{\goa'}$. Then  $\sequent{\contextGraph_1',...,\contextGraph_n'\assump'}{\goa'}$ is a schema for $\multiSpace$.
\end{theorem}

\begin{proof}
Given that $\sequent{\contextGraph_1',...,\contextGraph_n',\assump'}{\goa'}$ is a refinement of $\sequent{\contextGraph_1,...,\contextGraph_n,\assump}{\goa}$, we know that there exists a refinement map, $\refine\colon (\contextGraph_1,...,\contextGraph_{n},\assump\cup \goa) \to (\contextGraph_1',...,\contextGraph_{n},\assump'\cup \goa')$. By lemma~\ref{lemma:weakening-image}, we know that $\sequent{\refine[\contextGraph_1],...,\refine[\contextGraph_n],\refine[\assump]}{\refine[\consequent]}$ is a schema. By definition~\ref{defn:refinement}, $\sequent{\refine[\contextGraph_1],...,\refine[\contextGraph_n],\refine[\assump]}{\refine[\consequent]}$ is monotonic for $(\contextGraph_1',...,\contextGraph_n',\assump')$.  Therefore $\sequent{\contextGraph_1',...,\contextGraph_n',\assump'}{\refine[\consequent]}$ is a schema.  This implies that for any instantiatable specialisation, $\specialise\colon (\contextGraph_1',...,\contextGraph_n',\assump')\to (\contextGraph_1'',...,\contextGraph_n'',\assump'')$, there exists an instantiatable specialisation, say $\specialise'\colon (\contextGraph_1',...,\contextGraph_{n}',\assump'\cup \refine[\consequent])\to (\contextGraph_1'',...,\contextGraph_{n}'',\assump''\cup \consequent'')$, that extends $\specialise$. By definition~\ref{defn:refinement}, $\goa'\subseteq  \assump' \cup \refine[\consequent]$, so  $\specialise'$ induces an instantiatable specialisation of $(\contextGraph_1',...,\contextGraph_{n}',\assump'\cup \goa')$. Hence $\sequent{\contextGraph_1',...,\contextGraph_n',\assump'}{\goa'}$ is a schema for $\multiSpace$.
\end{proof}

We prove two lemmas that immediately entail theorem~\ref{thm:schema-application-schema}.

\begin{lemma}\label{lem:fwd-sequents-inst-schema}
Let $\sequent{\contextGraph_1,...,\contextGraph_n,\assump}{\goa}$ be a schema  and  $\sequent{\contextGraph_1',...,\contextGraph_n',\assump'}{\goa'}$ be a sequent for multi-space system $\multiSpace$. If $\sequent{\contextGraph_1',...,\contextGraph_{n}',\assump'\cup \delta}{\goa'}$ is a $\theDelta$-forward application of $\sequent{\contextGraph_1,...,\contextGraph_n,\assump}{\goa}$ to $\sequent{\contextGraph_1',...,\contextGraph_n',\assump'}{\goa'}$ and a schema for $\multiSpace$ then $\sequent{\contextGraph_1',...,\contextGraph_n',\assump'}{\goa'}$ is also a schema for $\multiSpace$.
\end{lemma}
\begin{proof}
Suppose that $\specialise\colon (\contextGraph_1',...,\contextGraph_{n}',\assump') \to (\contextGraph_1'',...,\contextGraph_{n}'',\assump'')$ is an instantiatable specialisation function. We must show that there exists an instantiatable specialisation function, $\specialise'\colon (\contextGraph_1',...,\contextGraph_{n}',\assump'\cup \goa') \to (\contextGraph_1'',...,\contextGraph_{n}'',\assump''\cup \goa'')$, that extends $\specialise$. By theorem~\ref{thm:weakening-schema}, we know that $\sequent{\contextGraph_1',...,\contextGraph_n',\assump'}{\theDelta}$ is a schema because it is a refinement of $\sequent{\contextGraph_1,...,\contextGraph_n,\assump}{\goa}$. Therefore, there exists an instantiatable specialisation function, $\specialise''\colon (\contextGraph_1',...,\contextGraph_{n}',\assump'\cup \theDelta) \to (\contextGraph_1'',...,\contextGraph_{n}'',\assump''\cup \theDelta')$, that extends $\specialise$. It is given that $\sequent{\contextGraph_1',...,\contextGraph_{n}',\assump'\cup \delta}{\goa'}$ is also a schema, so there must also exist an extension of $\specialise''$ to some instantiatable specialisation function, say  $\specialise'''\colon (\contextGraph_1',...,\contextGraph_{n}',\assump'\cup \theDelta\cup \goa') \to (\contextGraph_1'',...,\contextGraph_{n}'',\assump''\cup \theDelta''\cup \goa''')$. Restricting the domain of $\specialise'''$ to $(\contextGraph_1',...,\contextGraph_{n}',\assump'\cup \goa')$ yields an instantiatable specialisation function that extends $\specialise$, as required. Hence $\sequent{\contextGraph_1',...,\contextGraph_n',\assump'}{\goa'}$ is a schema.
\end{proof}

\begin{lemma}\label{lem:bwd-sequents-inst-schema}
Let $\sequent{\contextGraph_1,...,\contextGraph_n,\assump}{\goa}$ be a schema  and  $\sequent{\contextGraph_1',...,\contextGraph_n',\assump'}{\goa'}$ be a sequent for multi-space system $\multiSpace$. If  $\sequent{\contextGraph_1',...,\contextGraph_n',\assump'}{\theDelta}$ is a $\theDelta$-backward application of $\sequent{\contextGraph_1,...,\contextGraph_n,\assump'}{\goa}$ to $\sequent{\contextGraph_1',...,\contextGraph_n',\assump'}{\goa'}$ and a schema for $\multiSpace$ then $\sequent{\contextGraph_1',...,\contextGraph_n',\assump'}{\goa'}$ is also a schema for $\multiSpace$.
\end{lemma}
\begin{proof}
Suppose that $\specialise\colon (\contextGraph_1',...,\contextGraph_{n}',\assump') \to (\contextGraph_1'',...,\contextGraph_{n}'',\assump'')$ is an instantiatable specialisation function. We must show that there exists an instantiatable specialisation function, $\specialise'\colon (\contextGraph_1',...,\contextGraph_{n}',\assump'\cup \goa') \to (\contextGraph_1'',...,\contextGraph_{n}'',\assump''\cup \goa')$, that extends $\specialise$. It is given that $\sequent{\contextGraph_1',...,\contextGraph_n',\assump'}{\theDelta}$ is a schema, so there exists an extension of $\specialise$ to some instantiatable specialisation function, $\specialise''\colon (\contextGraph_1',...,\contextGraph_{n}',\assump'\cup \theDelta) \to (\contextGraph_1'',...,\contextGraph_{n}'',\assump''\cup \theDelta')$. By theorem~\ref{thm:weakening-schema}, we know that $\sequent{\contextGraph_1',...,\contextGraph_{n},\assump'\cup \theDelta}{\goa'}$ is a schema because it is a refinement of $\sequent{\contextGraph_1,...,\contextGraph_n,\assump}{\goa}$. Therefore, we can extend $\specialise''$ to an instantiatable specialisation function, $\specialise'''\colon (\contextGraph_1',...,\contextGraph_{n}',\assump'\cup \theDelta\cup \goa') \to (\contextGraph_1'',...,\contextGraph_{n}'',\assump''\cup \theDelta'\cup \goa'')$. Notably, since the domain, $(\contextGraph_1',...,\contextGraph_{n}',\assump'\cup \theDelta\cup \goa')$ of $\specialise'''$ can be restricted to $(\contextGraph_1',...,\contextGraph_{n}',\assump'\cup \goa')$, we have the existence of an extension of $\specialise$ to an instantiatable specialisation, $\specialise'\colon (\contextGraph_1',...,\contextGraph_{n}',\assump'\cup \goa') \to (\contextGraph_1'',...,\contextGraph_{n}'',\assump''\cup \goa''')$, for some $\goa'''$. Hence $\sequent{\contextGraph_1',...,\contextGraph_n',\assump'}{\goa'}$ is a schema.
\end{proof}

Theorem~\ref{thm:schema-application-schema}, restated below, is a corollary of the above two lemmas.
\begin{theorem}
	Let $\sequent{\contextGraph_1,\ldots,\contextGraph_n,\assump}{\goa}$  be a schema and $\sequent{\contextGraph_1',\ldots,\contextGraph_n',\assump'}{\goa'}$   be a sequent for multi-space system $\multiSpace$.  If $\sequent{\contextGraph_1',\ldots,\contextGraph_n',\assump''}{\goa''}$ is an application of $\sequent{\contextGraph_1,\ldots,\contextGraph_n,\assump}{\goa}$  to $\sequent{\contextGraph_1',\ldots,\contextGraph_n',\assump'}{\goa'}$ and a schema for $\multiSpace$ then $\sequent{\contextGraph_1',\ldots,\contextGraph_n',\assump'}{\goa'}$ is also a schema for $\multiSpace$.
\end{theorem}
\begin{proof}
	A schema application is either a forward or a backward application, respectively covered by lemmas~\ref{lem:fwd-sequents-inst-schema} and~\ref{lem:bwd-sequents-inst-schema}.
\end{proof}

We now restate and prove theorem~\ref{thm:valid-sequents}.

\begin{theorem}
Let $\mathbb{S}$ be a set of schemas and let $\sequent{\contextGraph_1',...,\contextGraph_n',\assump'}{\goa'}$ be a sequent for multi-space system $\multiSpace$. If $\sequent{\contextGraph_1',...,\contextGraph_n',\assump'}{\goa'}$ is valid over $\mathbb{S}$ then it is also a  schema for $\multiSpace$.
\end{theorem}

\begin{proof}
The proof is by induction, over the depth of the recursion used to establish the validity of $\sequent{\contextGraph_1',...,\contextGraph_n',\assump'}{\goa'}$. The base case, where $\sequent{\contextGraph_1',...,\contextGraph_n',\assump'}{\goa'}\in \mathbb{S}$, is trivial since $\mathbb{S}$ is a set of schemas. Assume, for any sequent, $\sequent{\contextGraph_1',...,\contextGraph_n',\assump'}{\goa'}$, valid over $\mathbb{S}$ at depth $k$ that $\sequent{\contextGraph_1',...,\contextGraph_n',\assump'}{\goa'}$ is a schema. Let $\sequent{\contextGraph_1',...,\contextGraph_n',\assump'}{\goa'}$ be a valid sequent at depth $k+1$. We must show that $\sequent{\contextGraph_1',...,\contextGraph_n',\assump'}{\goa'}$ is a schema. Since $\sequent{\contextGraph_1',...,\contextGraph_n',\assump'}{\goa'}$ is valid, there exists $\sequent{\contextGraph_1,...,\contextGraph_n,\assump}{\goa} \in \mathbb{S}$ and a sequent, $\sequent{\contextGraph_1',...,\contextGraph_n',\assump''}{\goa''}$, for $\multiSpace$ such that
    \begin{enumerate}
    \item $\sequent{\contextGraph_1',...,\contextGraph_n',\assump''}{\goa''}$ is valid at depth $k$, and
    \item $\sequent{\contextGraph_1',...,\contextGraph_n',\assump''}{\goa''}$ is an application of $\sequent{\contextGraph_1,...,\contextGraph_n,\assump}{\goa}$ to $\sequent{\contextGraph_1',...,\contextGraph_n',\assump'}{\goa'}$.
    \end{enumerate}
By the inductive assumption, $\sequent{\contextGraph_1',...,\contextGraph_n',\assump''}{\goa''}$ is a schema. Thus, by theorem~\ref{thm:schema-application-schema}, we deduce that $\sequent{\contextGraph_1',...,\contextGraph_n',\assump'}{\goa'}$ is a schema for $\multiSpace$, as required.
\end{proof}

\begin{example}\label{appendixSec4:example} In Figure~\ref{fig:inst-applications}, on the left, we show 3 schemas, $\pi$, $\rho$, and $\sigma$, for \textsc{Set Algebra}. Schema $\pi$ encodes the fact that any type we \textit{input} into constructor \texttt{infixOp} leads to a potential instantiation of the output in a manner that is consistent with the types of the inputs. Schema $\rho$, with empty antecedent, encodes the fact that type $\texttt{B}$ can be instantiated -- that is, that there exist a token, $B$, of type $\texttt{B}$. Finally, schema $\sigma$ encodes the fact that the \texttt{union} type can be instantiated (with a token of the form $\cup$).
	
	Suppose we want to know whether a pattern graph, $\goa$, as shown in the pattern graph (centre top), can be instantiated. We represent this task with a sequent $\sequent{\assump}{\goa}$ where $\assump$ is the empty graph and $\goa$ is the pattern graph we in question. Using schemas A, B and C we can show that $\goa$ is instantiatble by applying schemas in a backward manner, in forward manner, or a combination of the two. Below we show one particular sequence of schema applications that works. The first step involves applying A in a backward manner. remove the graph of $\goa$ corresponding to $\consequent$ and add back the part corresponding to $\antecedent$. Notice that, in the end, we reach a sequent where $\assump$ contains $\goa$, showing that the starting sequent is a valid instantiation sequent. As we will show, this means $\goa$ is instantiatable.
	\begin{figure}[ht]
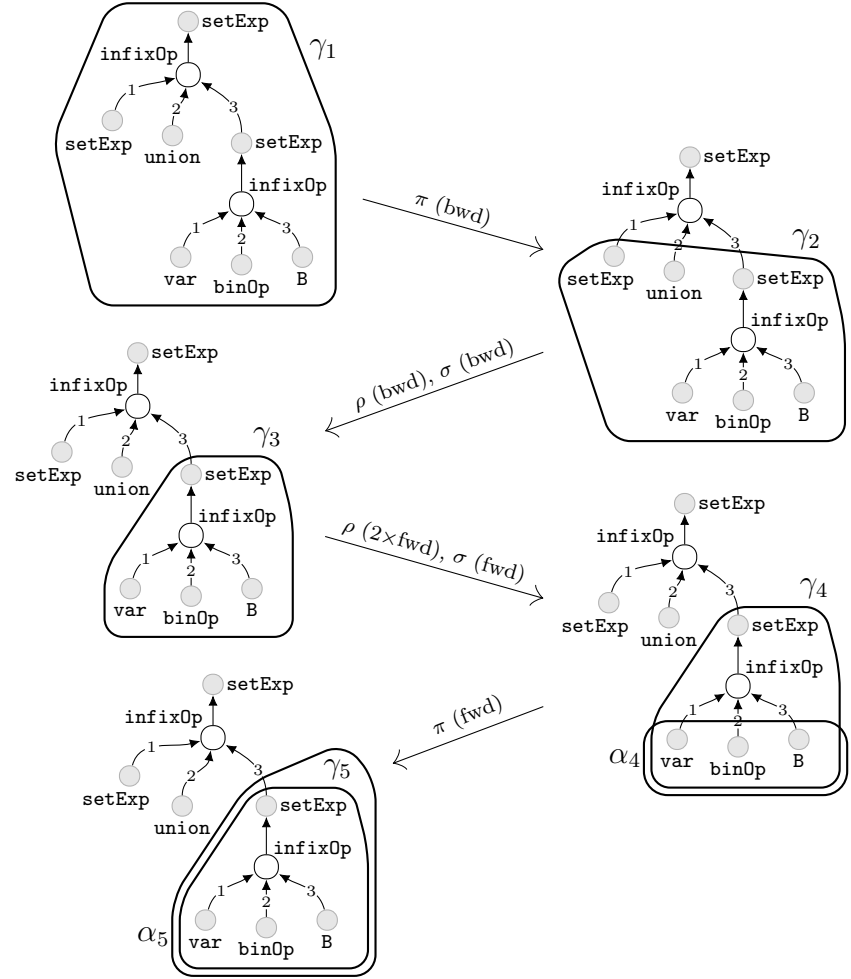

}};
	\draw[-{Classical TikZ Rightarrow[length=1.5mm]}] (p1) -- node[anchor=center, yshift = 0.2cm]{\scriptsize \rotatebox{-17}{$\pi$ (bwd)}} (p2);
	\draw[-{Classical TikZ Rightarrow[length=1.5mm]}] (p2) -- node[anchor=center, yshift = 0.22cm]{\scriptsize \rotatebox{20}{$\rho$ (bwd), $\sigma$ (bwd)}} (p3);
	\draw[-{Classical TikZ Rightarrow[length=1.5mm]}] (p3) -- node[anchor=center, yshift = 0.22cm]{\scriptsize \rotatebox{-16}{$\rho$ ($2\times$fwd), $\sigma$ (fwd)}} (p4);
	\draw[-{Classical TikZ Rightarrow[length=1.5mm]}] (p4) -- node[anchor=center, yshift = 0.25cm]{\scriptsize \rotatebox{22}{$\pi$ (fwd)}} (p5);
	\draw (1.5,3.02) -- (1.5,-10.75);
	\node[anchor = north west, text width = 4cm] () at (-2.1,3.1) {\textbf{Schemas}};
	\node[anchor = north west] () at (1.7,3.1) {\textbf{Sequence of applications}};
\end{tikzpicture}\caption{A sequence of schema applications}\label{fig:inst-applications}
\end{figure}
\end{example}

%%%%%%% Gem altered %%%%%

\begin{example} In Figure~\ref{fig:inst-applications}, on the left, we show 3 schemas, $\pi$, $\rho$ and $\sigma$, for \textsc{Set Algebra}, two of which -- $\rho$ and $\sigma$ --  have an empty antecedent graph. Schema $\pi$, $\sequent{\antecedentN{\pi}}{\consequentN{\pi}}$,  encodes the fact that any tokens, drawn from \textsc{Set Algebra}, of the three \textit{input} types of the constructor \texttt{infixOp} leads to an instantiation of the output. Schema $\rho$, which is $\sequent{\ }{\consequentN{\rho}}$, encodes the fact that type $\texttt{B}$ can be instantiated -- that is, that there exist a token, $B$, of type $\texttt{B}$. Finally, schema $\sigma$, which is $\sequent{\ }{\consequentN{\sigma}}$, encodes the fact that the \texttt{union} type can be instantiated (with a token of the form $\cup$).
	
Suppose we want to know whether a pattern graph, such as $\goa_1$ shown centre top, can be instantiated. We represent this task with a sequent $\sequent{\contextGraph_1}{\goa_1}$ where $\contextGraph_1$ is the empty graph and $\goa_1$ is the pattern graph we in question. Using schemas $\pi$, $\rho$ and $\sigma$ we can show that $\goa_1$ is instantiatable by applying them in a backward manner, in forward manner, or a combination of the two.

Below we show one particular sequence of applications. The first step involves applying $\pi$ in a backward manner. We start with the weakening of $\sequent{\antecedentN{\pi}}{\consequentN{\pi}}$ to $\sequent{\assump^{\pi}}{\goa_1}$, where some map, $f$, maps $\consequentN{\pi}$ to the top part of $\goa_1$; this gives $\goa_1\backslash f[\consequentN{\pi}]$ as the $\goa_1$-extender, highlighted in the figure. We then have a backward application of $\sequent{\antecedentN{\pi}}{\consequentN{\pi}}$ to $\sequent{\ }{\goa_1}$ being $\sequent{\ }{(\goa_1\backslash f[\consequentN{\pi}])\cup \assump^{\pi}}$, which is the next sequent in the diagram, namely $\sequent{\ }{\goa_2}$. Essentially, this backwards application removed -- from $\goa_1$ -- the subgraph of $\goa_1$ it that was mapped to by $\consequentN{\pi}$ and added back the part, $\assump^{\pi}$, that was mapped to by $\antecedentN{\pi}$. Notice that, after all seven applications, we reach the sequent $\sequent{\contextGraph_5}{\goa_5}$, where the assumption, $\contextGraph_5$ is a subgraph of $\goa_5$ (in fact, they are equal).  Trivially, this sequent is a schema. Theorem~\ref{thm:valid-sequents} allows us to deduce that the original sequent, $\sequent{\ }{\goa_1}$, is also a schema. Therefore, it must be that $\goa_1$ can be instantiated, since the empty graph is itself instantiatable.
\end{example}

We restate and prove Theorem~\ref{thm:transfer-sequence-valid}.
\begin{theorem}
	Let $\sequent{\contextGraph_1,...,\contextGraph_n,\assump}{\goa}$ be a sequent for $\multiSpace$, let $\sigma \subseteq \{1,\ldots,n\}$, and let $\mathbb{T}$ be a set of $\sigma$-transfer schemas. Assume we apply these schemas sequentially, starting with $\sequent{\contextGraph_1,...,\contextGraph_n,\assump}{\goa}$ and ending with $\sequent{\contextGraph_1',...,\contextGraph_n',\assump'}{\goa'}$. If all schemas in $\mathbb{T}$ are monotonic for $(\contextGraph_1',...,\contextGraph_n',\assump')$ and $\sequent{\contextGraph_1',...,\contextGraph_n',\assump'}{\goa'}$ is a valid sequent, then $\sequent{\contextGraph_1,...,\contextGraph_n,\assump}{\goa}$ is valid modulo $\sigma$.
\end{theorem}
\begin{proof}
	Suppose we $\sigma$-apply transfer schemas $m$ times to $\sequent{\contextGraph_1,...,\contextGraph_n,\assump}{\goa}$ to end with a valid sequent, $\sequent{\contextGraph_1',...,\contextGraph_n',\assump'}{\goa'}$. This gives rise to $m$ $\sigma$-reifications, that is $m$ reifications, $\reify_1, \ldots, \reify_m$. The composition of reifications is a reification. Then, it is easy to see that we can apply $\reify_m \circ \cdots \circ \reify_1$ to obtain some sequent $\sequent{\contextGraph_1'',...,\contextGraph_n'',\assump''}{\goa''}$ and then we can apply the schemas to obtain a sequent which is label-isomorphic to $\sequent{\contextGraph_1',...,\contextGraph_n',\assump'}{\goa'}$. This is because if a monotonic schema is applicable to a sequent then it is applicable to any reification of the context. Therefore,  $\sequent{\contextGraph_1'',...,\contextGraph_n'',\assump''}{\goa''}$ is valid and thus $\sequent{\contextGraph_1,...,\contextGraph_n,\assump}{\goa}$ is valid modulo $\sigma$.
\end{proof}

\section{Appendix: Structure transfer worked example}\label{sec:appendixTransfer}
\newtheorem{sch}{Schema}[]
	
The following transfer schemas characterise the notion of depiction.
\begin{sch}From example~\ref{ex:depict-conj}. Let $\multiSpace = (\mathcal{C},\mathcal{D},\mathcal{G})$ be a multi-space where $\mathcal{C}$ encodes Set Algebra, $\mathcal{D}$ encodes Euler Diagrams and $\mathcal{G}$ encodes relations across and within the spaces. The sequent $\sequent{\contextGraph_1,\contextGraph_2,\assump}{\goa}$ for $\multiSpace$, depicted below, is as schema. It captures the intuition that a diagram depicts a conjunction if it depicts both conjuncts.
	\begin{center}

	\end{center}
\end{sch}

\begin{sch} If a diagram depicts a formula then adding another curve preserves the depiction. %We may need that the curve label being added is not already mentioned in the formula, but I think that comes for free with the right context:
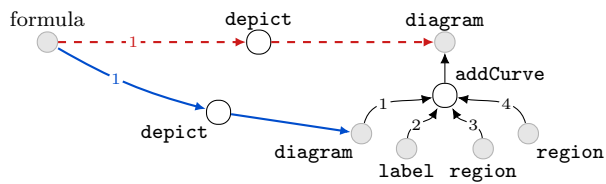
\begin{center}
%
%\end{center}
\begin{sch}[Depicting subsets] From example~\ref{ex:subset-depict}. To depict an expression of the form $x \subseteq \eta$, where $\eta$ is a variable, it suffices to take a diagram, $d$, and make sure that  $\eta$  is drawn so that the region corresponding to $x$ is part of the \textit{in} regions of the curve labelled by $\eta$ (based on~\cite{stapleton2010inductively}). This means that \textbf{for every $\eta$} the following is a schema\footnote{Hereafter we will draw sequents and schemas of the form $\sequent{\contextGraph_1,\contextGraph_2,\assump}{\goa}$ with $\contextGraph_1$ on the left, $\contextGraph_2$ on the right, $\assump$ with thick blue arrows, and $\goa$ with dashed red arrows.}:
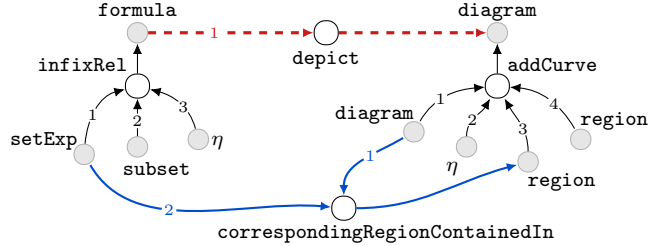
	\begin{center}

	\end{center}
	Using this schema as a $\{2\}$-transfer schema would produce a diagram from a given formula\footnote{Moreover, using this schema as a $\{1\}$-transfer schema can be used for producing a formula for a given diagram. Using it as $\{1,2\}$-transfer schema can be used to produce both the given assumptions and goals. Using it as a $\emptyset$-transfer schema is no different to using it as a schema.}.
\end{sch}
\begin{sch}[Depicting disjoint sets] From example~\ref{ex:disjoint-depict}. To depict an expression of the form $x \cap \eta = \emptyset$, where $x$ is a set expression and $\eta$ is a variable, it suffices to take a diagram, $d$, and add $\eta$ in such a way that the region corresponding to $x$ in $d$ is part of the \textit{out} regions of the curve labelled by $\eta$ (based on~\cite{stapleton2010inductively}). Thus, the following figure describes a family of schemas, one for each $\eta$, subtype of $\texttt{var}$ in \textsc{Set Algebra}, and subtype of $\texttt{\texttt{label}}$ in \textsc{Euler Diagrams}.%\footnote{Hereafter we will draw sequents and schemas of the form $\sequent{\contextGraph_1,\contextGraph_2,\assump}{\goa}$ with $\contextGraph_1$ on the left, $\contextGraph_2$ on the right, $\assump$ with thick blue arrows, and $\goa$ with dashed red arrows.}:
	\begin{center}

	\end{center}
\end{sch}
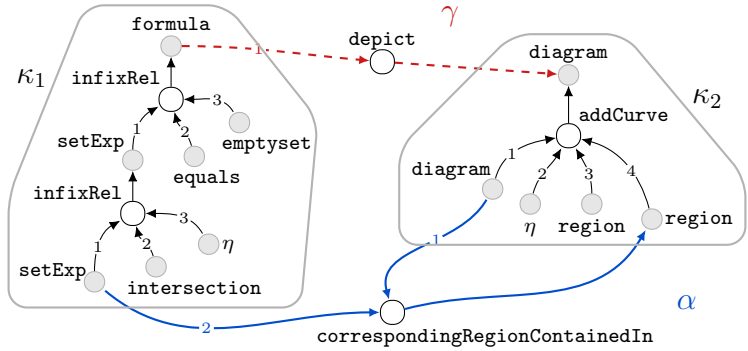
\begin{sch}[Top-down instantiation of an \texttt{addCurve} configuration]From example~\ref{ex:addCurve-top-down}. Provided a diagram of type $\mathtt{\{\eta x_1,\ldots,\eta x_l,\eta y_1,\ldots,\eta y_m,y_1,\ldots,y_m,z_1,\ldots,z_n\}} \leq \mathtt{diagram}$, where $\mathtt{\eta}$ is a label that does not appear in either $\mathtt{x_1,\ldots,x_l,y_1,\ldots,y_m,z_1,\ldots,z_n}$ then we can decompose $\mathtt{\{\eta x_1,\ldots,\eta x_l,}\linebreak[2]\mathtt{\eta y_1,\ldots,\eta y_m,y_1,\ldots,y_m,z_1,\ldots,z_n\}}$ by removing $\mathtt{\eta}$. This is captured by the family of schemas $\sequent{\assump}{\goa}$, pictured below, for unary multi-space system $\mathcal{D}$:
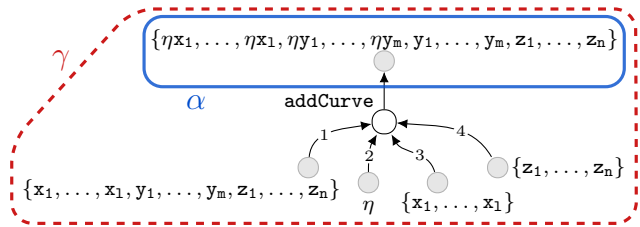
	\begin{center}
		\begin{tikzpicture}[construction]
			\node[typeN = {$\mathtt{\{\eta x_1,\ldots,\eta x_l,\eta y_1,\ldots,\eta y_m,y_1,\ldots,y_m,z_1,\ldots,z_n\}}$}] (t) {};
			\node[constructor = {\texttt{addCurve}}, below = 0.5cm of t] (c) {};
			\node[typepos = {$\mathtt{\{x_1,\ldots,x_l,y_1,\ldots,y_m,z_1,\ldots,z_n\}}$}{-170}{1.57cm},below left = 0.3cm and 0.7cm of c] (t1) {};
			\node[typepos = {$\mathtt{\eta}$}{-90}{0.15cm},below left = 0.5cm and -0.1cm of c] (t2) {};
			\node[typepos = {$\mathtt{\{x_1,\ldots,x_l\}}$}{-45}{0.18cm},below right = 0.5cm and 0.4cm of c] (t3) {};
			\node[typeE = {$\mathtt{\{z_1,\ldots,z_n\}}$},below right = 0.3cm and 1.2cm of c] (t4) {};
			\path[->] (c) edge (t) 
			(t1) edge[out = 85, in = -170] node[index label]{1} (c) 
			(t2) edge[out = 90, in = -120] node[index label]{2} (c) 
			(t3) edge[out = 105, in = -60] node[index label]{3} (c) 
			(t4) edge[out = 120, in = -5] node[index label]{4} (c);
			\coordinate[above left = 0.45cm and 3.05cm of t] (x1);
			\coordinate[above right = 0.45cm and 3.05cm of t] (x2);
			\coordinate[below right = 0.2cm and 3.05cm of t] (x3);
			\coordinate[below left = 0.2cm and 3.05cm of t] (x4);
			\draw[rounded corners = 7, very thick, darkblue, draw opacity = 0.8] (x1) -- (x2) -- (x3) --node[xshift = -2.5cm, yshift = -0.2cm]{\large$\assump$} (x4) -- cycle;
			\coordinate[above left = 0.55cm and 3.2cm of t] (y1);
			\coordinate[above right = 0.55cm and 3.2cm of t] (y2);
			\coordinate[above right = 0.2cm and 1.7cm of t4] (y3);
			\coordinate[below right = 0.6cm and 1.7cm of t4] (y4);
			\coordinate[below left = 0.6cm and 3.8cm of t1] (y5);
			\coordinate[above left = 0.2cm and 3.8cm of t1] (y6);
			\draw[rounded corners = 8, very thick, darkred, dashed] (y1) -- (y2) -- (y3) -- (y4) -- (y5) -- (y6) --node[xshift = -0.15cm, yshift = 0.2cm]{\large$\goa$} cycle;
		\end{tikzpicture}
	\end{center}
	This family of schemas represents knowledge about how to \textit{parse} an Euler diagram.
\end{sch}

\begin{sch}[Bottom-up instantiation of an \texttt{addCurve} configuration]From example~\ref{ex:addCurve-bottom-up}. Provided that \texttt{addCurve} takes as inputs: a diagram of type $\mathtt{\{x_1,\ldots,x_l,y_1,\ldots,y_m,z_1,\ldots,z_n\}} \leq \mathtt{diagram}$, a label $\eta \leq \mathtt{label}$, an \textit{in} region $\mathtt{\{x_1,\ldots,x_l\}} \leq \mathtt{region}$, and an \textit{out} region $\mathtt{\{z_1,\ldots,z_n\}} \leq \mathtt{region}$, where $\mathtt{\eta}$ does not appear in $\mathtt{\{x_1,\ldots,x_l,y_1,\ldots,y_m,z_1,\ldots,z_n\}}$, i.e. it is not a label in any of the words,
	\textit{then} we can infer the output. This is captured by the family of schemas $\sequent{\assump}{\goa}$, pictured below, for unary multi-space system $\mathcal{D}$:
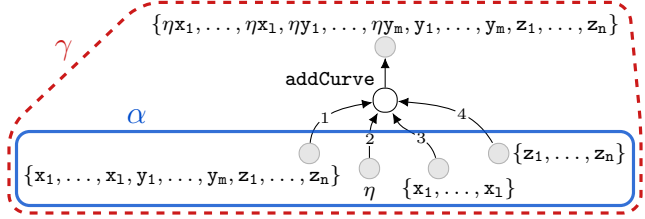
	\begin{center}
		\begin{tikzpicture}[construction]
			\node[typeN = {$\mathtt{\{\eta x_1,\ldots,\eta x_l,\eta y_1,\ldots,\eta y_m,y_1,\ldots,y_m,z_1,\ldots,z_n\}}$}] (t) {};
			\node[constructor = {\texttt{addCurve}}, below = 0.4cm of t] (c) {};
			\node[typepos = {$\mathtt{\{x_1,\ldots,x_l,y_1,\ldots,y_m,z_1,\ldots,z_n\}}$}{-170}{1.57cm},below left = 0.4cm and 0.7cm of c] (t1) {};
			\node[typepos = {$\mathtt{\eta}$}{-90}{0.15cm},below left = 0.6cm and -0.1cm of c] (t2) {};
			\node[typepos = {$\mathtt{\{x_1,\ldots,x_l\}}$}{-45}{0.18cm},below right = 0.6cm and 0.4cm of c] (t3) {};
			\node[typeE = {$\mathtt{\{z_1,\ldots,z_n\}}$},below right = 0.4cm and 1.2cm of c] (t4) {};
			\path[->] (c) edge (t) 
			(t1) edge[out = 85, in = -170] node[index label]{1} (c) 
			(t2) edge[out = 90, in = -120] node[index label]{2} (c) 
			(t3) edge[out = 105, in = -60] node[index label]{3} (c) 
			(t4) edge[out = 120, in = -5] node[index label]{4} (c);
			\coordinate[above left = 0.15cm and 3.75cm of t1] (x1);
			\coordinate[above right = 0.15cm and 1.65cm of t4] (x2);
			\coordinate[below right = 0.55cm and 1.65cm of t4] (x3);
			\coordinate[below left = 0.55cm and 3.75cm of t1] (x4);
			\draw[rounded corners = 5, very thick, darkblue, draw opacity = 0.8] (x1) --node[xshift = -2.5cm, yshift = 0.2cm]{\large$\assump$} (x2) -- (x3) -- (x4) -- cycle;
			\coordinate[above left = 0.45cm and 3.1cm of t] (y1);
			\coordinate[above right = 0.45cm and 3.1cm of t] (y2);
			\coordinate[above right = 0.2cm and 1.75cm of t4] (y3);
			\coordinate[below right = 0.65cm and 1.75cm of t4] (y4);
			\coordinate[below left = 0.65cm and 3.85cm of t1] (y5);
			\coordinate[above left = 0.2cm and 3.85cm of t1] (y6);
			\draw[rounded corners = 7, very thick, darkred, dashed] (y1) -- (y2) -- (y3) -- (y4) -- (y5) -- (y6) --node[xshift = -0.15cm, yshift = 0.2cm]{\large$\goa$} cycle;
		\end{tikzpicture}
	\end{center}
	Note that in the output, $\mathtt{\{\eta x_1,\ldots,\eta x_l,\eta y_1,\ldots,\eta y_m,y_1,\ldots,y_m,z_1,\ldots,z_n\}}$, the zones of the \textit{in} region, $\mathtt{\{x_1,\ldots,x_l\}}$, are immersed in $\mathtt{\eta}$, the zones of the \textit{out} region, $\mathtt{\{z_1,\ldots,z_n\}}$, are disjoint of $\mathtt{\eta}$, and the remaining ones, $\mathtt{\{y_1,\ldots,y_m\}}$, are each split in two: one inside, and one outside of $\mathtt{\eta}$.
	
	If we see the constructor \texttt{addCurve} as a program, this family of schemas represents knowledge about how to  \textit{compute} this program. 
\end{sch}

\begin{sch}[Set expressions and their corresponding regions] From Example~\ref{ex:corr-var-region}. 
	This schema tells you how to propagate from the input to the output of the \texttt{corresponding\-RegionContainedIn} constructor. For every $\mathtt{x_1,\ldots,x_l,y_1,\ldots,y_m,z_1,\ldots,z_n}$ and $\mathtt{\eta}$ the following is a schema:
	\begin{center}

	\end{center}
\end{sch}

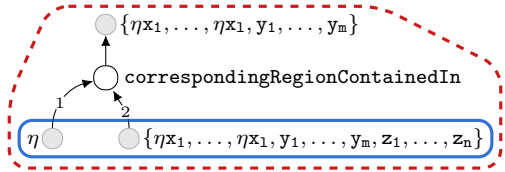
\begin{sch}[Singleton instantiation] From Example~\ref{ex:euler-singletons}. For any label, $\eta$, the sequents $\sequent{\ \ }{\goa_1}$, $\sequent{\ \ }{\goa_2}$,  $\sequent{\ \ }{\goa_3}$ and $\sequent{\ \ }{\goa_4}$, pictured below, are schemas for unary multi-space system $\mathcal{D}$.
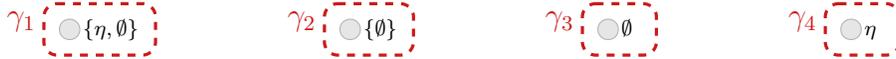
	\begin{center}
		%
	\end{center}
\end{sch}

\begin{sch}[Observe and depict are dual] From example~\ref{ex:depict-observe}. If a diagram depicts a formula then the same formula can be observed from the diagram. Thus $\sequent{\contextGraph_1,\contextGraph_2,\assump}{\goa}$, depicted below, is a schema.
	\begin{center}
		%
	\end{center}
\end{sch}
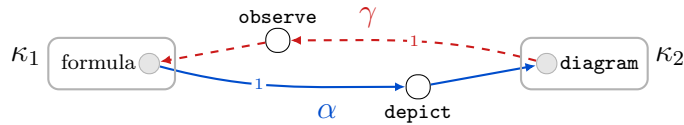
\begin{sch} A formula is not observable from a conjunct of formulas if it is not observable from either conjunct.
	\begin{center}
		%
	\end{center}
\end{sch}
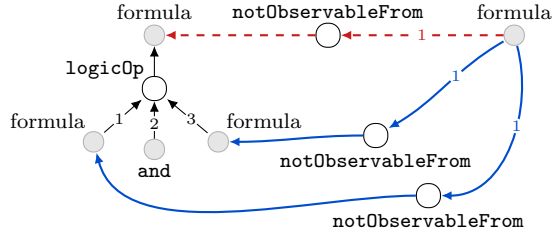
\begin{sch} A formula is not observable from another if they are not \textit{syntactically} equal.
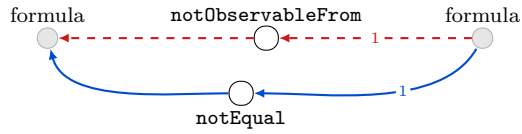
\begin{center}
	%
\end{center}
\end{sch}
\begin{sch} A pair of atomic formulas are not equal if any part of them is not equal. Thus we have three schemas:
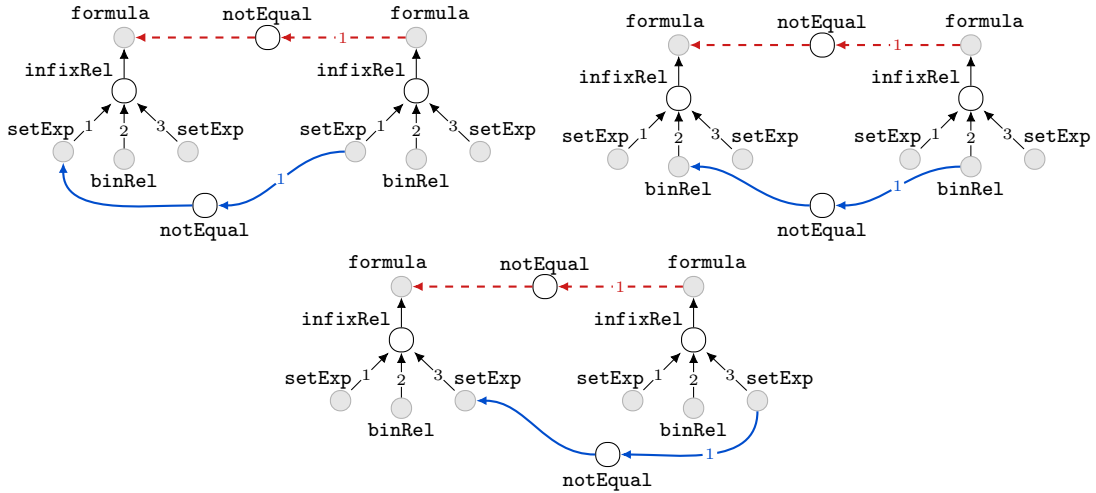
\begin{center}
	%
\end{center}
\end{sch}
\begin{sch} A pair of set expressions are not equal if any part of them is not equal. Thus we have three schemas:
	\begin{center}
		%
	\end{center}
\end{sch}

\paragraph{Problem encoding} The multi-space we will work on is $(\mathcal{C},\mathcal{D},\mathcal{C},\mathcal{G})$ where $\mathcal{C}$ encodes \textsc{Set Algebra}, $\mathcal{D}$ encodes \textsc{Euler Diagrams} and $\mathcal{G}$ encodes relations across and within the spaces. Given expression $A \subseteq B \wedge B \cap C = \emptyset$, we want to find a diagram that depicts it and a set expression that can be observed from the diagram. In other words, given the sequent $\sequent{\contextGraph_1,\contextGraph_2,\contextGraph_3,\assump}{\goa}$, visualised below, we want to find a $\{2,3\}$-reification of it that witnesses its validity modulo $\{2,3\}$. Note that we encode the relation of \textit{not being observable from} to prevent a trivial observation.
\begin{center}

\end{center}
Apply schema 4 (depict disjoint relation):
\begin{center}
	%
\end{center}
Apply Schema 2 (pass depict to the first argument of \texttt{addCurve}):
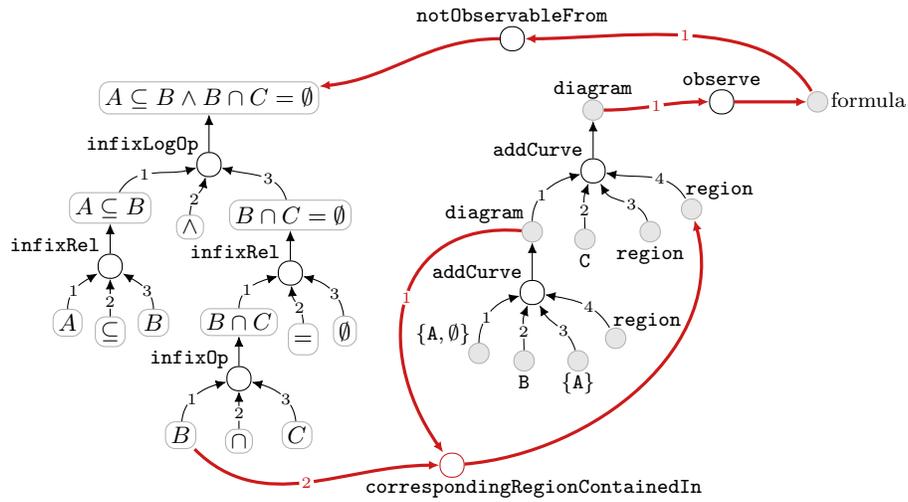
\begin{center}
	%
\end{center}
Apply Schema 3 (depict subset relation):
\begin{center}
	%
\end{center}
Apply Schema 8 (specialise region to $\emptyset$):
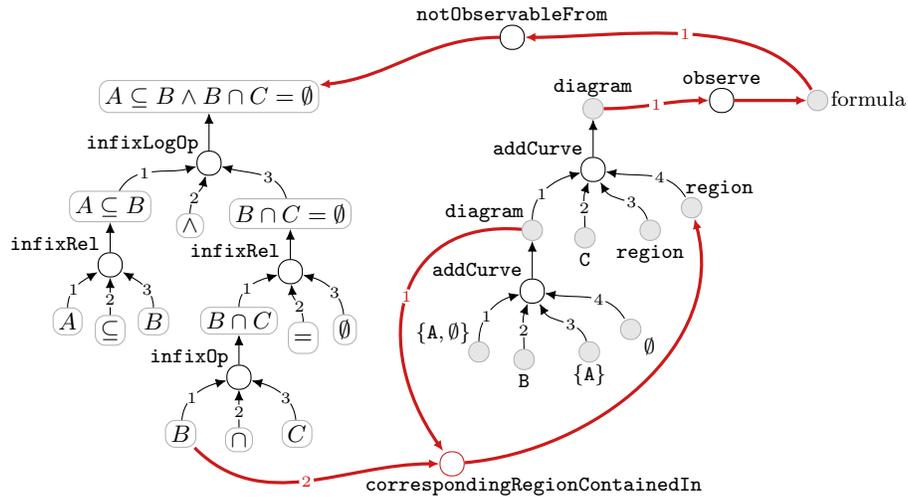
\begin{center}
	%
\end{center}
Apply Schema 6 (specialise type of diagram given inputs of \texttt{addCurve}):
\begin{center}
	%
\end{center}
Apply Schema 8 (specialise region to $\emptyset$) and then Schema 6 (specialise type of diagram):
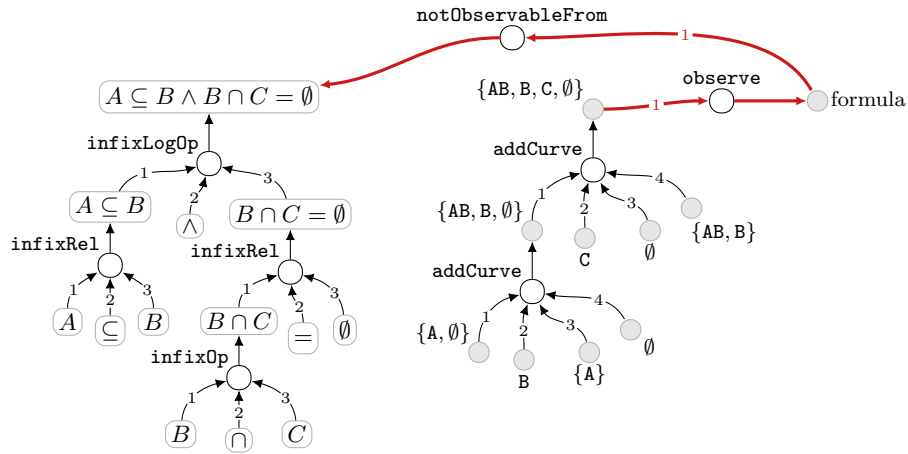
\begin{center}
	%
\end{center}
Apply Schema 9 (flip depict and observe):
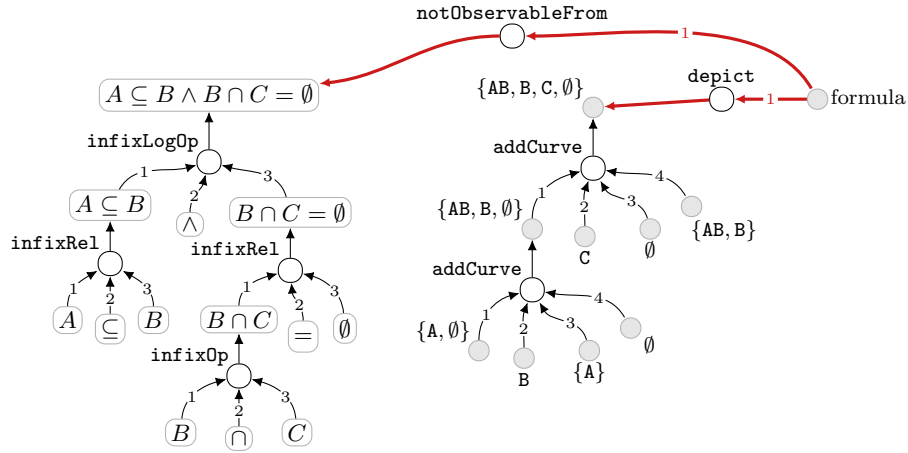
\begin{center}
	%
\end{center}

We could observe something trivial, but we can also allow the schema application to add structure to the set expression. Thus, we apply Schema 4:
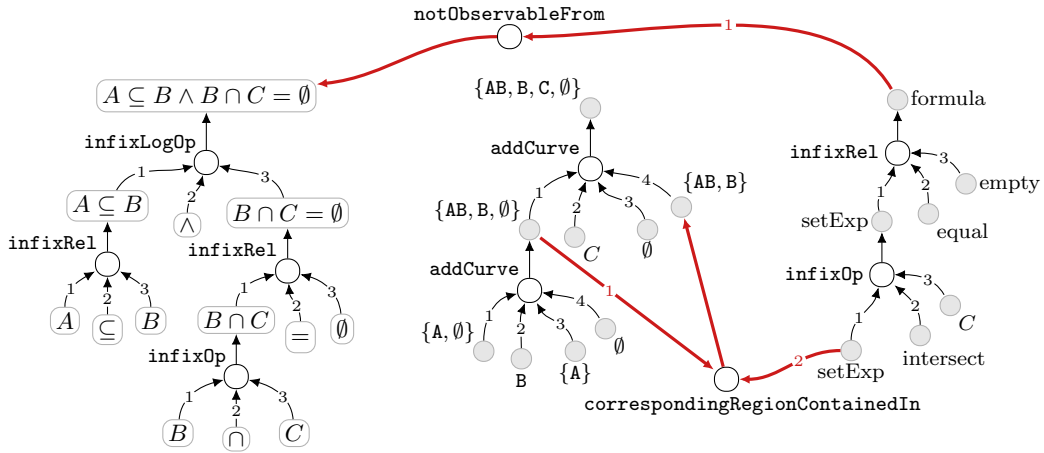
\begin{center}
	%
\end{center}

And now we can apply Schema 7 (compute \texttt{correspondingRegionContainedIn} noting that during the application we could specialise setExp with either $A$ or $B$):
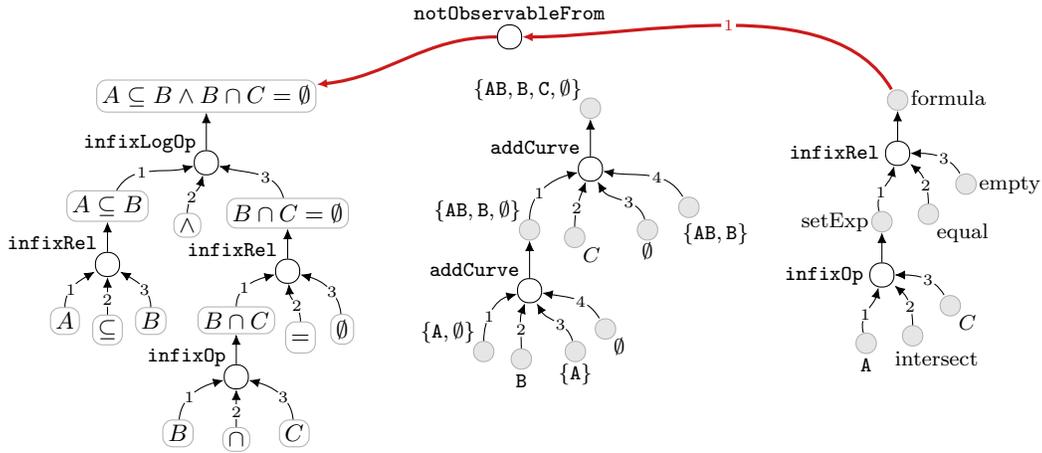
\begin{center}
	%
\end{center}
And finally we can apply Schemas 10, 11 and 12 to discharge the \texttt{notObservableFrom} goal. Note that this would not have been possible if $B$ had been chosen. 
\begin{center}
	%
\end{center}

\vskip 0.2in
\bibliography{library}
\bibliographystyle{theapa}

%\appendix
%\section*{Appendix A.}

%%% bibliography commands left as they were in the sample tex file

%\vskip 0.2in
%\bibliography{sample}
%\bibliographystyle{theapa}

\end{document}